\documentclass[twoside,11pt]{article}

\usepackage{blindtext}

\usepackage[preprint]{jmlr2e}

\usepackage{amsmath,amssymb, amsfonts}
\usepackage{mathrsfs}
\usepackage{dsfont}
\usepackage{xcolor}
\usepackage{stmaryrd}
\usepackage{enumitem}
\usepackage{float}
\usepackage{thm-restate}
\usepackage{subcaption}
\usepackage{changepage}
\usepackage{needspace}

\newcommand{\R}{\mathbb{R}}
\newcommand{\bmB}{\mathcal{B}}

\newcommand{\N}{\mathbb{N}}

\newcommand{\E}{\mathbb{E}}

\newcommand{\bmH}{\mathcal{H}}
\newcommand{\bmG}{\mathcal{G}}
\newcommand{\bmF}{\mathcal{F}}

\newcommand{\bmN}{\mathcal{N}}

\newcommand{\bmX}{\mathcal{X}}
\newcommand{\bmO}{\mathcal{O}}
\newcommand{\bmS}{\mathcal{S}}
\newcommand{\bmY}{\mathcal{Y}}
\newcommand{\bmZ}{\mathcal{Z}}
\newcommand{\bmL}{\mathcal{L}}
\newcommand{\1}{\mathds{1}}

\DeclareMathOperator{\Tr}{Tr}

\DeclareMathOperator{\hs}{HS}

\DeclareMathOperator*{\argmin}{arg\,min\,}

\newtheorem{assumption}{Assumption}

\newenvironment{sproof}{%
  \begin{proof}[Sketch of proof]%
}{%
  \end{proof}%
}

\usepackage{lastpage}
\jmlrheading{~}{~}{~}{~}{~}{~}{Luc Brogat-Motte, Riccardo Bonalli, and Alessandro Rudi}
\ShortHeadings{Learning Controlled SDEs}{Brogat-Motte, Bonalli, and Rudi}
\firstpageno{1}

\begin{document}

\title{Learning Controlled Stochastic Differential Equations}

\author{\name Luc Brogat-Motte (corresponding author)\email luc.brogatmotte@iit.it \\
  \addr Istituto Italiano di Tecnologia \\
        Via Morego 30, 16163 Genova, Italy
  \AND
  \name Riccardo Bonalli \email riccardo.bonalli@cnrs.fr \\
  \addr Laboratoire des Signaux et Systèmes, CNRS, CentraleSupélec \\
        Université Paris-Saclay \\
        Gif-sur-Yvette, France
  \AND
  \name Alessandro Rudi \email alessandro.rudi@sdabocconi.it \\
  \addr SDA Bocconi School of Management \\
        Bocconi University \\
        Milan, Italy
}
       
\editor{My editor}

\maketitle

\begin{abstract}
We study the problem of learning controlled stochastic differential equations (SDEs)
\[
dX_t = b(t,X_t,u_t)\,dt + \sigma(t,X_t,u_t)\,dW_t,
\]
whose drift and diffusion depend nonlinearly on time, state, and control values. From trajectory data, we aim to estimate coefficients whose induced density flows reproduce those of the observed dynamics. The data consist of several controls sampled from a finite-dimensional family and, for each control, multiple independent trajectories observed at discrete times over a fixed horizon. The controls are observed inputs, not learner-selected decisions.
We propose a kernel method for multidimensional nonlinear controlled SDEs. The method estimates the density flow for each sampled control, then fits the drift \(b\) and diffusion matrix \(a=\sigma\sigma^\top\) by matching the estimated flows through the Fokker--Planck equation.
Under Sobolev regularity assumptions, we prove finite-sample bounds on
the error between the density flows of the learned and true SDEs, measured
in \(L^2\) over controls, time, and state. The bounds quantify how the error decreases with the number \(K\) of sampled controls, with rates determined by the state and control-parametrization dimensions and Sobolev regularity. We further derive uniform-in-control guarantees and CVaR-type bounds for tail-sensitive quantities. Numerical experiments illustrate the method, and an open-source Python implementation is provided.
\end{abstract}


\section{Introduction}\label{sec:introduction}

Models of dynamical systems describe how the state of a system evolves over time. They are central in engineering, robotics, physics, biology, and economics, where they support analysis, prediction, simulation, control, optimization, and fault detection \citep{isermann2011identification, nelles2020nonlinear}. In some settings, these models can be derived from physical, mechanical, electrical, chemical, biological, or economic principles. In many others, first-principles equations are unavailable or incomplete. This motivates data-driven system identification: learning dynamical models from observations \citep{aastrom1971system}.

We focus on continuous-time systems subject to both random fluctuations and control inputs. Specifically, we consider controlled stochastic differential equations of the form 
% \begin{align}\label{eq:sde0} dX(t) = b(t, X(t), u(t)) \, dt + \sigma(t, X(t), u(t)) \, dW(t), \qquad X(0) \sim p_0, \qquad u \in \bmH, \end{align}
\begin{equation}\label{eq:sde0}
\begin{aligned}
dX(t) &= b(t,X(t),u(t))\,dt
      + \sigma(t,X(t),u(t))\,dW(t),\\
X(0) &\sim p_0, \qquad u \in \bmH,
\end{aligned}
\end{equation}
where \(X(t)\) is an \(n\)-dimensional stochastic process, \(W(t)\) is a standard Brownian motion, and \(p_0\) is a probability density on \(\mathbb R^n\). The set \(\bmH \subset \mathcal F([0,T],\mathbb R^d)\) denotes a class of deterministic open-loop controls, and
\[
b,\sigma : [0,T]\times\mathbb R^n\times\mathbb R^d
\to \mathbb R^n \times \mathbb R^{n\times n}
\]
denote the drift and diffusion coefficients.
% and
% \[
% \begin{aligned}
% b &: [0,T] \times \mathbb R^n \times \mathbb R^d \to \mathbb R^n,\\
% \sigma &: [0,T] \times \mathbb R^n \times \mathbb R^d \to \mathbb R^{n\times n},
% \end{aligned}
% \]
% are the drift and diffusion coefficients. The set \(\bmH \subset \mathcal F([0,T], \mathbb R^d)\), denotes a class of deterministic open-loop controls.

% where \(X(t)\) is an \(n\)-dimensional stochastic process, \(W(t)\) is a standard Brownian motion, \(p_0\) is a probability density on \(\mathbb R^n\), and \[ b : [0,T] \times \mathbb R^n \times \mathbb R^d \to \mathbb R^n, \qquad \sigma : [0,T] \times \mathbb R^n \times \mathbb R^d \to \mathbb R^{n\times n} \] are the drift and diffusion coefficients. The set \(\bmH \subset \mathcal F([0,T], \mathbb R^d)\) denotes a class of deterministic open-loop controls.
We study a passive learning setting for estimating \(b\) and \(a=\sigma\sigma^\top\). The learner observes trajectories generated under \(K\) controls \(u_1,\ldots,u_K\), but does not choose these controls as part of the estimation procedure. Instead, the controls are modeled as independent draws from a fixed, possibly unknown, distribution on \(\bmH\). This setting arises naturally when inputs come from logged operation, experimental protocols, environmental forcing, or other data-collection processes in which active probing is costly, unsafe, or unavailable.
For each control \(u_k\), we observe \(Q\) independent trajectories of the SDE driven by \(u_k\), observed only at times \(t_1,\ldots,t_M\). Thus the data consist of repeated trajectory observations under finitely many controls. The formal sampling model is stated in Section~\ref{sec:problem_setting}.

Our goal is to learn controlled stochastic dynamics from these passive observations while quantifying generalization across controls. The data are collected under finitely many controls \(u_1,\ldots,u_K\), whereas the learned model is intended to reproduce the system's density flow under new controls from the same family. We study how the learning error depends on the number of sampled controls, the state dimension, the dimension and smoothness of the control parametrization, and whether performance is measured on average or uniformly across controls.

\paragraph{Contributions.}
We propose a kernel-based estimator for multidimensional controlled SDEs whose drift and
diffusion may depend nonlinearly on time, state, and control. The method extends the
uncontrolled estimator of \citet{bonalli2023non}: it first estimates the density flows
associated with the sampled controls, and then learns drift and diffusion coefficients by
matching these flows through the Fokker--Planck equation.

We prove finite-sample bounds for the density flows induced by the learned coefficients. Our main results give \(L^2\) learning rates for the densities of the true and estimated SDEs, averaged over controls, times, and states. The bounds quantify how the error decreases with the number \(K\) of sampled controls, and how the rates depend on the state dimension, the dimension of the control parametrization, and the Sobolev regularity assumptions.

We also establish \(L^\infty\) learning rates, uniform over the control space, together with CVaR bounds for quantities evaluated under the learned SDE. These stronger guarantees are motivated by downstream applications where uniform or tail-sensitive reliability may be needed, including prediction, simulation, and model-based control. Finally, we provide numerical experiments and open-source Python implementations for both uncontrolled and controlled SDEs, available at \url{lmotte/sde-learn} and \url{lmotte/controlled-sde-learn}.

\subsection{Related Work}\label{subsec:related_work}

\paragraph{System identification.} System identification has a long history, with foundations going back to the 1960s \citep{aastrom1965numerical, ho1966effective} and mature treatments developed in the following decades \citep{ljung1999system, soderstrom}. A wide variety of methods for identifying dynamical systems exists \citep{ljung1999system, katayama2005subspace}. We refer to \citet{isermann2011identification, nelles2020nonlinear} for comprehensive treatments.

A system identification method is usually specified by a model class, an estimation criterion, and a data-collection protocol. Common model classes include state-space models, ARX and ARMAX models, linear and polynomial models, lookup tables, fuzzy models, neural networks, and Koopman-based representations \citep{watter2015embed, andersson2019deep, ljung2020deep, klus2020data, nuske2023finite, zhang2023quantitative}. Estimation criteria include least squares, maximum likelihood, and Bayesian methods. The choice of inputs is also central. Classical experiment design uses conditions such as persistence of excitation or criteria based on Fisher information \citep{ljung1999system}. More recent approaches use active learning or reinforcement learning to choose informative inputs during data collection \citep{abraham2019active, wagenmaker2020active}. Online and adaptive identification methods update the model while the system operates, so that the data distribution can depend on previous model estimates and control actions \citep{hewing2020learning, lew2022safe, mania2022active}.

In this paper, the controls are observed inputs sampled from a fixed distribution. They are not chosen adaptively, and there is no feedback from the estimator to the data-collection process. We study how the learned controlled dynamics generalize from the finitely many controls observed in the dataset to new controls drawn from the same distribution.

Finite-sample guarantees are well developed for linear dynamical systems \citep{simchowitz2018learning, sarkar2019near, dean2020sample, tu2022learning}. 
Existing nonlinear guarantees often rely on structured model classes for the dynamics \citep{oymak2019stochastic, chen2019sample, bahmani2020convex, foster2020learning, mania2022active, sattar2022non}. For example, several works study systems of the form \(x_{t+1}=\phi(Ax_t+Bu_t)\), where the nonlinearity \(\phi\) is known and fixed \citep{oymak2019stochastic, bahmani2020convex, foster2020learning, sattar2022non}. \citet{mania2022active} consider models of the form \(x_{t+1}=A\phi(x_t,u_t)+w_t\), under assumptions including warm starts, computational oracles, and controllability. These settings differ from ours, where the drift and diffusion of a continuous-time stochastic system are nonparametric functions of time, state, and control.

\paragraph{Uncontrolled SDE estimation.}
Most statistical work on SDE coefficient estimation considers autonomous systems without control inputs \citep{nielsen2000parameter}. Classical results often use one long trajectory and obtain asymptotic guarantees under ergodicity assumptions \citep{hoffmann1999adaptive, kessler1999estimating, comte2007penalized, abraham2019nonparametric, nickl2020nonparametric, kutoyants2013statistical, dalalyan2005sharp, genon1992non, florens1993estimating}. This observation model differs from ours because there are no controls, and because the information comes from long-time observation of one process rather than repeated short-horizon trajectories under several inputs.

A more recent line of work studies SDEs from independent sample paths observed over a fixed horizon \citep{comte2020nonparametric, bonalli2023non, nuske2023finite, zhang2023quantitative}. In particular, \citet{bonalli2023non} estimate multidimensional SDEs by first estimating the density flow and then fitting coefficients through the Fokker--Planck equation. We use the same density-flow and Fokker--Planck matching principle, but with coefficients that also depend on a control variable.

\paragraph{Controlled SDE estimation.}
For controlled SDEs, the closest finite-sample analysis we are aware of is \citet{nuske2023finite}. They consider nonlinear control-affine dynamics with drift of the form \(b(t,x,u_{\theta}(t)) = b_0(x) + B_1(x)\theta\), where the control is parametrized by \(\theta\in\mathbb R^m\). In their setting, the diffusion coefficient does not depend on time or control. The model therefore excludes drift terms with general nonlinear dependence on \(t\), \(x\), and \(u\), and excludes diffusion coefficients \(a(t,x,u)\) depending on the control. Their method also does not aim at returning explicit estimates for a general drift-diffusion pair \(b(t,x,u)\), \(a(t,x,u)\). The present paper studies a complementary setting, in which both \(b(t,x,u)\) and \(a(t,x,u)\) may depend on time, state, and control, under regularity, ellipticity, and localization assumptions. The guarantees are also stated at the density-flow level, rather than in a coefficient norm.

\subsection{Paper Organization}

The paper is organized as follows. Section~\ref{sec:problem_setting} states the learning problem. Section~\ref{sec:proposed_method} presents the estimation method. Section~\ref{sec:learning_guarantees} gives finite-sample error bounds. Section~\ref{sec:expe} reports the numerical experiments.

% \subsection{Notation}

% For any function
% \(f : \bmH \times A \to C\), we use the shorthand
% \(
% f(\theta) \triangleq f(u_\theta,\cdot),
% \)
% whenever no ambiguity arises.
% We write \(a\) as a shorthand for the diffusion matrix
% \(\sigma \sigma^\top\). We use the notation for Sobolev spaces \(W^q(A,B)\), tensor
% products \(u \otimes v\), and the Loewner partial order \(A \preceq B\).
% Detailed notation and norms are defined in
% Appendix~\ref{app:notation}.

\section{Problem Setting}\label{sec:problem_setting}

We now introduce the learning problem.

\paragraph{Parametrized controls.} We consider deterministic open-loop controls \(u:[0,T]\to\mathbb R^d\) with a finite-dimensional parametrization: \[ \bmH=\{u_\theta:\theta\in\Theta\}, \qquad \Theta\subset\mathbb R^m, \] where \(\Theta\) is a bounded Lipschitz domain.

\paragraph{Controlled stochastic dynamics.} We consider an \(n\)-dimensional controlled stochastic differential equation \begin{align}\label{eq:sde} dX(t) = b(t,X(t),u(t))\,dt + \sigma(t,X(t),u(t))\,dW(t), \qquad X(0) \sim p_0, \end{align} on a fixed time horizon \(T>0\). Here \(W\) is a standard Brownian motion, \(p_0\) is a probability density on \(\mathbb R^n\), and \[ b : [0,T]\times\mathbb R^n\times\mathbb R^d \to \mathbb R^n, \qquad \sigma : [0,T]\times\mathbb R^n\times\mathbb R^d \to \mathbb R^{n\times n} \] are the drift and diffusion coefficients. For \(\theta\in\Theta\), we write \(X^\theta\) for the solution of \eqref{eq:sde} driven by \(u_\theta\), whenever a unique strong solution exists.

\paragraph{Density flows.} Let \(a=\sigma\sigma^\top\) be the diffusion matrix. Since the law of the SDE depends on \(\sigma\) only through \(a\), we index density flows by \((b,a)\), not by a particular square root \(\sigma\). When \eqref{eq:sde} with coefficients \((b,a)\) admits densities, we denote the associated density flow by \[ p_{b,a} : \Theta \times [0,T] \times \mathbb R^n \to \mathbb R_+. \] Thus, for each \(\theta\in\Theta\) and \(t\in[0,T]\), \(p_{b,a}(\theta,t,\cdot)\) is the density of \(X^\theta(t)\) with respect to Lebesgue measure on \(\mathbb R^n\).

\paragraph{Dataset.} We work in an offline episodic setting. Control parameters \((\theta_k)_{k=1}^K\) are sampled independently from a fixed probability distribution \(\mathbb P_c\) on \(\Theta\). For each \(\theta_k\), we observe \(Q\) independent trajectories of \eqref{eq:sde} driven by \(u_{\theta_k}\), initialized from \(p_0\), and sampled at times \((t_\ell)_{\ell=1}^M\subset[0,T]\). The dataset is \begin{align} (\theta_k, X^{\theta_k}(\omega_i^k,t_\ell))_{k \in \llbracket 1,K \rrbracket,\; i \in \llbracket 1,Q \rrbracket,\; \ell \in \llbracket 1,M \rrbracket}. \end{align} A formal probabilistic construction is given in Appendix~\ref{app:prob_setting}.

\paragraph{Learning problem.} Our goal is to estimate, from the dataset above, coefficients \((b,a)\) whose
induced density flow reproduces that of the true controlled SDE.  Given a
hypothesis space \(\bmF\) of coefficient functions from
\([0,T]\times\mathbb R^n\times\mathbb R^d\) to \(\mathbb R^{n+n^2}\), we formulate
the estimation problem as
\begin{align}\label{eq:learn_obj} \min_{(\hat b, \hat a) \in \bmF}\: \|p_{\hat b, \hat a} - p\|_{L^2(\mathbb P_c)\,L^2([0,T]\times\mathbb R^n)}, \end{align} where \(p \triangleq p_{b,a}\) is the density flow of the true controlled SDE. 

% This objective compares the distributions induced by the true and estimated coefficients, rather than the coefficient error \((\hat b-b,\hat a-a)\) itself.

\begin{remark}[Learning at the level of density flows]\label{rk:non_id}
The objective in \eqref{eq:learn_obj} evaluates candidate coefficients through
the density flows they induce. Thus, the target is not pointwise recovery of
\((b,a)\) in a coefficient norm, but reproduction of the controlled dynamics at
the distributional level. As a consequence, two coefficient pairs that induce the
same density flow are indistinguishable for this criterion. Example~\ref{ex:nonidenti} gives a simple case where distinct SDE coefficients
induce the same density flow; see also
\citet{qiu2022identifiability, browning2020identifiability, miao2011identifiability,
wang2024generator} for discussions.
\end{remark}

% \begin{remark}[Learning at the level of probability densities]\label{rk:non_id} The objective in \eqref{eq:learn_obj} is formulated at the level of densities. Thus, two coefficient pairs may be indistinguishable for this objective if they induce the same density flow. The goal is not coefficient identifiability in a coefficient norm, but accurate reproduction of the controlled dynamics at the distributional level. See Example~\ref{ex:nonidenti} for an illustration and \citet{qiu2022identifiability, browning2020identifiability, miao2011identifiability, wang2024generator} for discussions. \end{remark}

\begin{example}[Distinct SDE coefficients inducing the same density]\label{ex:nonidenti}
Different SDE coefficients may induce the same probability density. Consider the uncontrolled Ornstein--Uhlenbeck process
\[
    dX(t)=\theta(\mu-X(t))\,dt+\sigma\,dW(t),
    \qquad X(0)\sim\mathcal N(m,s^2),
\]
with \(\theta>0\), \(\mu\in\mathbb R\), and \(\sigma>0\). For all \(t\ge 0\), \(X(t)\) is Gaussian, with mean \(\mu+(m-\mu)e^{-\theta t}\) and variance \(\frac{\sigma^2}{2\theta}(1-e^{-2\theta t})+s^2e^{-2\theta t}\). If \(m=\mu\) and \(s^2=\sigma^2/(2\theta)\), the process is stationary and has density \(\mathcal N(m,s^2)\) at all times. Hence different pairs \((\tilde\theta,\tilde\sigma)\) satisfying \(\tilde\sigma^2/(2\tilde\theta)=\sigma^2/(2\theta)\) define different SDEs with the same density flow.
\end{example}

\begin{remark}[Parametrized controls]\label{rk:control_measure} Controls are parametrized as \(\bmH=\{u_\theta:\theta\in\Theta\}\), and the sampling distribution \(\mathbb P_c\) is defined on \(\Theta\). Thus, sampling \(\theta\sim\mathbb P_c\) induces a random control \(u_\theta\in\bmH\). This avoids placing probability measures directly on the infinite-dimensional function space \(\bmH\). Extending the analysis to such measures, for instance Gaussian measures on separable Hilbert spaces of controls, is left for future work. \end{remark}

\section{Proposed Method}\label{sec:proposed_method}

We now present the estimator. It has two steps: first estimate the density flows associated with the sampled controls, then estimate the coefficients by Fokker--Planck matching. Section~\ref{subsec:fp_ineq} gives the well-posedness assumptions and the matching inequality. Section~\ref{subsec:estimator} defines the estimator.

\subsection{Well-Posedness and Fokker--Planck Matching Inequality}\label{subsec:fp_ineq} 

We first state the assumptions used to ensure well-posedness of the controlled
SDEs and to derive the Fokker--Planck matching inequality. Here and below, \(H^s\)
denotes the Sobolev space \(W^{s,2}\).

%  They impose
% ellipticity, regularity of the controls, initial density, and coefficients, and
% a localization condition on a compact state domain.

\begin{assumption}[Uniform ellipticity]\label{as:uniform_ellipticity}
A coefficient pair \((b,a)\) is uniformly elliptic if there exists \(\kappa>0\) such that
\[
a(t,x,v)\succeq \kappa I_n
\]
for all \((t,x,v)\in[0,T]\times\R^n\times\R^d\).
\end{assumption}

% \begin{assumption}[Uniform ellipticity]\label{as:uniform_ellipticity} A coefficient pair \((b,a)\) is uniformly elliptic if there exist \(\kappa>0\) and \(a_0:[0,T]\times\R^n\times\R^d\to\R^{n\times n}\) such that \[ a(t,x,v)=\kappa I_n+a_0(t,x,v), \qquad a_0(t,x,v)\succeq 0, \] for all \((t,x,v)\in[0,T]\times\R^n\times\R^d\). \end{assumption} 

% \begin{assumption}[Control regularity]\label{as:smooth_controls} Let \(s_0 \triangleq 5+n+\left\lceil \frac{d+1}{2}\right\rceil\). There exists a compact set \(V\subset\R^d\) such that, for every \(\theta\in\Theta\), \(u_\theta([0,T])\subset V\) and \[ \sup_{\theta\in\Theta}\|u_\theta\|_{W^{s_0}([0,T];\R^d)}<\infty . \] \end{assumption}

\begin{assumption}[Control regularity]\label{as:smooth_controls}
Let \(s_0 \triangleq 5+n+\left\lceil \frac{d+1}{2}\right\rceil\).
There exists a bounded Lipschitz domain \(V\subset\R^d\) such that, for every
\(\theta\in\Theta\), \(u_\theta([0,T])\subset V\) and
\[
\sup_{\theta\in\Theta}\|u_\theta\|_{H^{s_0}([0,T];\R^d)}<\infty .
\]
\end{assumption}

\begin{assumption}[Initial density regularity]\label{as:initial_density}
The initial density satisfies \(p_0\in H^3(\R^n)\).
\end{assumption}

\begin{assumption}[Coefficient localization]\label{as:p_s}
There exist a bounded Lipschitz domain \(D\subset\R^n\), a cutoff
\(\xi\in C_c^\infty(\R^n)\), with \(0\leq \xi\leq 1\) and
\(\operatorname{supp}(\xi)\subset D\), and functions \(b_0,a_0\) such that
\[
b(t,x,v)=\xi(x)b_0(t,x,v),
\qquad
a(t,x,v)=\kappa I_n+\xi(x)a_0(t,x,v).
\]
\end{assumption}

Assumption~\ref{as:p_s} means that the drift and the localized part of the diffusion vanish outside \([0,T]\times D\), with \(\xi b_0\) and \(\xi a_0\) understood to be extended by zero outside \(D\). There, the dynamics reduce to Brownian motion with diffusion matrix \(\kappa I_n\). This restricts the target class. It can also be viewed as a smooth localization of a larger non-localized model; in that case, our results apply to the localized model and do not quantify the localization bias.

\begin{assumption}[Coefficient regularity]\label{as:smooth_coeffs}
The functions \(b_0\) and \(a_0\) from Assumption~\ref{as:p_s} satisfy
\[
b_0 \in H^{s_0}([0,T]\times\R^n\times\R^d;\R^n),
\qquad
a_0 \in H^{s_0}([0,T]\times\R^n\times\R^d;\R^{n\times n}).
\]
\end{assumption}

% \begin{assumption}[Coefficient regularity]\label{as:smooth_coeffs}
% The coefficient pair \((b,a)\), with \(a=\kappa I_n+a_0\), satisfies
% \[
% b \in W^{s_0}([0,T]\times\mathbb R^n\times\mathbb R^d;\mathbb R^n),
% \qquad
% a_0 \in W^{s_0}([0,T]\times\mathbb R^n\times\mathbb R^d;\mathbb R^{n\times n}).
% \]
% \end{assumption}

% \begin{assumption}[Coefficient localization]\label{as:p_s} Let \(a=\kappa I_n+a_0\). There exist a compact set \(D\subset\R^n\) and a cutoff \(\xi\in C_c^\infty(\R^n)\), with \(0\leq \xi\leq 1\) and \(\operatorname{supp}(\xi)\subset D\), such that \[ \xi b=b, \qquad \xi a_0=a_0 . \] \end{assumption}

% \begin{assumption}[Coefficient localization]\label{as:p_s}
% Let \(a=\kappa I_n+a_0\). There exists a compact set \(D\subset\R^n\) such
% that, for every \(v\in\R^d\), the functions
% \[
% (t,x)\mapsto b(t,x,v), \qquad (t,x)\mapsto a_0(t,x,v),
% \]
% and the derivatives involved in the dual Kolmogorov generator are supported in
% \([0,T]\times D\).
% \end{assumption}

% Under these assumptions, the following lemma controls the density-flow discrepancy between the true and estimated SDEs by a Fokker--Planck residual.

\paragraph{Fokker--Planck operator.}
For \(\theta\in\Theta\), the coefficients \((b,a)\) and the control \(u_\theta\) define a
Fokker--Planck operator acting on sufficiently smooth functions
\(q:[0,T]\times\mathbb R^n\to\mathbb R\) by
\[
\bmL^{(b,a), \theta} q(t,x)
=
\frac12 \sum_{i,j=1}^n
\partial_{x_i x_j}\!\big(a_{ij}(t,x,u_\theta(t))q(t,x)\big)
-
\sum_{i=1}^n
\partial_{x_i}\!\big(b_i(t,x,u_\theta(t))q(t,x)\big).
\]
The objective in \eqref{eq:learn_obj} depends on \((\hat b,\hat a)\) through the density flow \(p_{\hat b,\hat a}\), which is defined implicitly by the SDE. The next lemma replaces this implicit quantity by an explicit residual involving the true density flow \(p\) and the candidate coefficients \((\hat b,\hat a)\). We call this residual the Fokker--Planck matching objective.

\begin{lemma}[Fokker--Planck matching inequality]\label{lem:FP_ineq_main}
Let \((b,a)\) and \((\hat b,\hat a)\) satisfy
Assumptions~\ref{as:uniform_ellipticity},
\ref{as:smooth_controls}, \ref{as:initial_density},
\ref{as:p_s}, and~\ref{as:smooth_coeffs}, with the same
localization domain \(D\), the same cutoff function \(\xi\), and the same ellipticity
constant \(\kappa>0\). Then the associated controlled SDEs admit densities \(p\) and
\(p_{\hat b,\hat a}\). Moreover, there exists a constant
\(C_{\mathrm{FP}}>0\), depending on the \(H^{s_0}\)-norms of
\(\hat b\) and \(\hat a-\kappa I_n\), such that
\begin{align}\label{eq:fp_ineq_main}
\|p_{\hat b,\hat a}-p\|_{L^2(\mathbb P_c)L^2([0,T]\times\R^n)}
\le C_{\mathrm{FP}}\, \mathrm{FP}(\hat b,\hat a),
\end{align}
where
\begin{align}\label{eq:fp_matching}
\mathrm{FP}(\hat b,\hat a)
\triangleq
\left(
\E_{\theta\sim\mathbb P_c}
\left\|
\partial_t p(\theta,\cdot,\cdot)
-
\bmL^{(\hat b,\hat a), \theta} p(\theta,\cdot,\cdot)
\right\|_{L^2([0,T]\times D)}^2
\right)^{1/2}.
\end{align}
The proof is given in Lemma~\ref{lem:FP_ineq}.
\end{lemma}

Lemma~\ref{lem:FP_ineq_main} provides the statistical reason for using
Fokker--Planck matching: a small residual
\(\mathrm{FP}(\hat b,\hat a)\) implies a small density-flow error through
\eqref{eq:fp_ineq_main}. The same formulation also leads to a convenient optimization
problem. Once the density flow \(p\) is fixed, the residual
\(\partial_t p-L^{(\hat b,\hat a)}p\) depends affinely on the unknown coefficients.
Therefore, when the coefficient class is represented in a Hilbert space, minimizing the
squared residual together with a Hilbert-norm regularization term yields a convex
quadratic least-squares objective. This is the structure exploited in the next section to
define the estimator.

\subsection{Proposed Estimator}\label{subsec:estimator}

We now define the estimator, combining density estimation with Fokker--Planck matching.

\paragraph{Controlled SDE estimator.} For each sampled parameter \(\theta_k\), let \(\hat p_k\) be the density-flow estimator of \citet{bonalli2023non} trained on \[ \big(X^{\theta_k}(\omega_i^k,t_\ell)\big)_ {i\in\llbracket 1,Q\rrbracket,\;\ell\in\llbracket 1,M\rrbracket}. \] Let \(z_i=(s_i,y_i)\), \(i=1,\ldots,N\), be sampled uniformly from \([0,T]\times D\). Given \((\hat p_k)_{k=1}^K\), we define \begin{align}\label{eq:proposed_estimator} (\hat b,\hat a) \in \argmin_{(b,a)\in\bmF} \frac{1}{KN} \sum_{k=1}^K\sum_{i=1}^N \left( \partial_t\hat p_k(s_i,y_i) - \bmL^{(b,a),\theta_k}\hat p_k(s_i,y_i) \right)^2 + \lambda \|(b,a)\|_{\mathcal F}^2 . \end{align}

\paragraph{Hypothesis space.} It remains to specify the hypothesis space \(\bmF\) and the regularization norm in \eqref{eq:proposed_estimator}. We build \(\bmF\) so that its elements satisfy the structural assumptions of Lemma~\ref{lem:FP_ineq_main}: smoothness, uniform ellipticity, and localization. Let \(k\) be a continuous positive definite kernel on \([0,T]\times D\times V\), with scalar RKHS \(\mathcal H_k\) continuously embedded into \(H^{s_0}([0,T]\times D\times V)\). Typical choices include Matérn kernels of sufficiently high order. We model the drift in the vector-valued RKHS \(\mathbb R^n\otimes\mathcal H_k\). For the diffusion, write \(a=\kappa I_n+a_0\). We enforce \(a_0\succeq0\) through a positive-semidefinite kernel parametrization \citep{marteau2020non, muzellec2021learning, rudi2021psd, rudi2024finding, berthier2022infinite, vacher2021dimension}. With \(\phi(t,x,v)=k((t,x,v),\cdot)\in\mathcal H_k\), set \[ a_0(t,x,v) = \left( \big\langle \phi(t,x,v),\, w_{ij}\,\phi(t,x,v)\big\rangle_{\mathcal H_k} \right)_{i,j=1}^n, \] where \(w=(w_{ij})_{i,j=1}^n\) is self-adjoint and positive semidefinite. This guarantees \(a(t,x,v)\succeq\kappa I_n\). The embedding of \(\mathcal H_k\) into \(H^{s_0}\), together with the algebra property of \(H^{s_0}\), gives the required Sobolev regularity. Finally, we enforce localization with the cutoff \(\xi\) from Assumption~\ref{as:p_s}. The hypothesis space \(\bmF\) is the set of all pairs \((b,a)\) of the form \[ b(t,x,v)=\xi(x)\, b_0(t,x,v), \qquad a(t,x,v)=\kappa I_n+\xi(x)\, a_0(t,x,v), \] where \(b_0\in\mathbb R^n \otimes \mathcal H_{k}\) and \(a_0\) has the positive-semidefinite kernel parametrization above.  See Appendix~\ref{subsec:preli} for the definition of \(\|\cdot\|_{\bmF}\).

\begin{remark}[Alternative ellipticity constraints]\label{rem:alternative_ellipticity_constraints} Uniform ellipticity can also be enforced by pointwise constraints at sampled locations. Given points \(\zeta_i\in[0,T]\times D\times V\) and directions \(r_j\in\mathbb R^n\), one may impose \[ r_j^\top a(\zeta_i)r_j \ge \kappa, \qquad (i,j)\in\llbracket 1,q\rrbracket\times\llbracket 1,r\rrbracket . \] In the isotropic case \(a=a_0 I_n\), this reduces to \(a_0(\zeta_i)\ge \kappa\). The guarantees in this paper are proved for the hard positive-semidefinite parametrization above; the analysis of such pointwise constraints is left for future work. \end{remark}

\section{Learning Guarantees}\label{sec:learning_guarantees}

We now state finite-sample guarantees for the estimator of Section~\ref{sec:proposed_method}. The error compares the density flows of the learned and true SDEs, and is measured with respect to the sampling distribution over controls. This quantifies generalization across the control family. We first give a baseline \(L^2\) bound under attainability (Section~\ref{subsec:L2_rates}), then refine it using source and embedding assumptions tailored to Fokker--Planck matching (Section~\ref{subsec:refined_L2_rates}). We next derive uniform-in-control bounds (Section~\ref{subsec:Linfty_rates}) and CVaR-type guarantees (Section~\ref{subsec:cvar_lr}). Finally, we instantiate the assumptions and rates for Sobolev coefficients (Section~\ref{subsec:assumptions_examples}).

\subsection{\texorpdfstring{\(L^{2}\)}{L2} Learning Rates}\label{subsec:L2_rates}

We start with a baseline \(L^2\) guarantee. It relies on the following well-specifiedness condition. 

\begin{assumption}[Attainability]\label{as:attain} The true coefficients belong to the hypothesis space: \begin{align} (b,a)\in\bmF . \end{align} \end{assumption}

Assumption~\ref{as:attain} expresses prior information on the target coefficients: the estimator searches over a class that contains the truth. This is the standard well-specified setting in regularized least-squares regression \citep{caponnetto2007optimal}. Some restriction of this kind is unavoidable for uniform finite-sample guarantees, by no-free-lunch phenomena \citep{devroye2013probabilistic}. Concrete examples are given in Section~\ref{subsec:assumptions_examples}.

% \begin{theorem}[\(L^2\) learning rates]\label{thm:lr}
% Let \(N,K\in\mathbb N^*\). Let
% \((\theta_k)_{k=1}^K\overset{\mathrm{i.i.d.}}{\sim}\mathbb{P}_c\), set
% \(u_k=u_{\theta_k}\), and let
% \(z_i=(s_i,y_i)\overset{\mathrm{i.i.d.}}{\sim}
% \mathrm{Unif}([0,T]\times D)\), where \(D\) is the localization domain of
% Assumption~\ref{as:p_s}. Under
% Assumptions~\ref{as:uniform_ellipticity}--\ref{as:attain}, let \((\hat b,\hat a)\) be the estimator defined in
% Section~\ref{sec:proposed_method} from these samples. 
% Then there exist constants \(c_1,c_2>0\), independent of \(N\), \(K\), and
% \(\delta\), such that the following holds for any \(\delta\in(0,1]\). Define
% \(\varepsilon \triangleq
% \sup_{k\in\llbracket 1,K\rrbracket} \mathcal{E}(\hat p_k,p(\theta_k))\), where \(\mathcal{E}\) is the
% squared Sobolev norm defined in Appendix~\ref{subsec:proof_theorem}, and set
% \(\lambda =
% c_2\big(N^{-1}\log(N\delta^{-1})
% +K^{-1}\log(K\delta^{-1})\big)\). If
% \begin{align}
% \varepsilon
% \leq
% \frac{1}{N}\log\frac{N}{\delta}
% +
% \frac{1}{K}\log\frac{K}{\delta},
% \end{align}
% then, with probability at least \(1-\delta\),
% \begin{align}
% \|p_{\hat b,\hat a}-p\|_{L^2(\mathbb{P}_c)L^2([0,T]\times\mathbb R^n)}
% \leq
% c_1\log\frac{2}{\delta}
% \left(
% \frac{\log(N/\delta)}{\sqrt N}
% +
% \frac{\log(K/\delta)}{\sqrt K}
% \right).
% \end{align}
% \end{theorem}

\begin{theorem}[\(L^2\) learning rates]\label{thm:lr}
Let \(N,K\in\mathbb N^*\). Let
\((\theta_k)_{k=1}^K\overset{\mathrm{i.i.d.}}{\sim}\mathbb{P}_c\), and let
\(z_i=(s_i,y_i)\overset{\mathrm{i.i.d.}}{\sim}
\mathrm{Unif}([0,T]\times D)\), where \(D\) is the localization domain of
Assumption~\ref{as:p_s}. Under
Assumptions~\ref{as:uniform_ellipticity}--\ref{as:attain}, let \((\hat b,\hat a)\) be the estimator defined in
Section~\ref{sec:proposed_method} from these samples.
Then there exist constants \(c_1,c_2>0\), independent of \(N\), \(K\), and
\(\delta\), such that the following holds for any \(\delta\in(0,1]\). Define
\(\varepsilon \triangleq
\sup_{k\in\llbracket 1,K\rrbracket} \mathcal{E}(\hat p_k,p(\theta_k))\), where \(\mathcal{E}\) is the
Sobolev-type error defined in Appendix~\ref{subsec:proof_theorem}, and set
\(\lambda =
c_2\log\frac{2}{\delta}
\big(N^{-1}\log^2(N\delta^{-1})
+K^{-1}\log^2(K\delta^{-1})\big)\). If
\begin{align}
\varepsilon
\leq
\frac{1}{N}\log^2\frac{N}{\delta}
+
\frac{1}{K}\log^2\frac{K}{\delta},
\end{align}
then, with probability at least \(1-\delta\),
\begin{align}
\|p_{\hat b,\hat a}-p\|_{L^2(\mathbb{P}_c)L^2([0,T]\times\mathbb R^n)}
\leq
c_1\log\frac{2}{\delta}
\left(
\frac{\log(N/\delta)}{\sqrt N}
+
\frac{\log(K/\delta)}{\sqrt K}
\right).
\end{align}
\end{theorem}

\begin{sproof}
The proof decomposes the error of the proposed estimator into three components:
\begin{enumerate}
    \item the error due to the estimation of the probability densities
    \((\hat p_k)_{k=1}^K \approx (p(\theta_k))_{k=1}^K\);
    \item the error due to the finite-sample approximation of the
    Fokker--Planck residual over \([0,T]\times D\);
    \item the error due to the finite-sample approximation over the control
    space \(\bmH\).
\end{enumerate}
Each error component is then bounded by representing all quantities as norms of random linear operators, followed by appropriate decomposition and the application of Bernstein inequalities for sums of operators.
\end{sproof}

Theorem~\ref{thm:lr} gives a baseline finite-sample guarantee for the density flow induced by the learned coefficients. Provided the first-stage density estimators are accurate enough, the error decreases at rate \(N^{-1/2}+K^{-1/2}\), up to logarithmic factors. The two terms correspond to the two empirical approximations in Fokker--Planck matching: sampling state-time points and sampling controls.

\begin{remark}[Dependence on \ensuremath{K} and \ensuremath{N}] Although the Fokker--Planck matching objective is evaluated on \(KN\) pairs \((\theta_k,z_i)\), the rate in Theorem~\ref{thm:lr} is \(K^{-1/2}+N^{-1/2}\), up to logarithmic factors, not \((KN)^{-1/2}\). This reflects the product structure of the empirical approximation: controls and state-time points approximate two different integrals. Increasing \(N\) improves the approximation over \([0,T]\times D\), while increasing \(K\) improves the approximation over the control distribution. Thus, for fixed \(K\), the bound does not vanish as \(N\to\infty\). \end{remark}

% \begin{remark}[Role of \ensuremath{\mathbb{P}_c}] The bound in Theorem~\ref{thm:lr} is measured in \(L^2(\mathbb{P}_c)\) over controls. Its strength therefore depends on \(\mathbb{P}_c\): errors on controls with small \(\mathbb{P}_c\)-mass have limited effect, while controls outside the support of \(\mathbb{P}_c\) are not controlled. Thus, \((\bmH,\mathbb{P}_c)\) should reflect the controls for which the learned dynamics will be used. In this work, this pair is fixed externally; we do not study adaptive input selection or optimal experiment design. Similarly, \(p_0\) determines which part of the state space is explored by the observed trajectories. \end{remark}

\begin{remark}[Role of \ensuremath{\mathbb{P}_c}]
The bound in Theorem~\ref{thm:lr} is measured in \(L^2(\mathbb{P}_c)\)
over control parameters, equivalently over controls through the pushforward
of \(\mathbb{P}_c\) by \(\theta\mapsto u_\theta\). Its strength therefore
depends on \(\mathbb{P}_c\): errors on controls with small
\(\mathbb{P}_c\)-mass have limited effect, while controls outside its support
are not controlled. Thus, \((\Theta,\mathbb{P}_c)\) should reflect the
controls for which the learned dynamics will be used. Similarly, \(p_0\)
affects which regions of the state space are explored, and therefore the
parts of the dynamics that can be accurately learned. We treat both choices
as fixed and do not study experimental design.
\end{remark}

% \begin{remark}[Role of \ensuremath{\mathbb{P}_c}]
% The bound in Theorem~\ref{thm:lr} is measured in \(L^2(\mathbb{P}_c)\) over
% controls. Hence its strength depends on the sampling distribution
% \(\mathbb{P}_c\): errors on controls that have small \(\mathbb{P}_c\)-mass have
% limited effect on the bound, while controls outside the support of
% \(\mathbb{P}_c\) are not controlled. Thus, \((\bmH,\mathbb{P}_c)\) should be chosen
% to reflect the controls for which the learned dynamics will be used. In this
% work, this pair is fixed externally; we do not study adaptive input selection or
% optimal experiment design. Similarly, the initial distribution \(p_0\) determines
% which part of the state space is explored by the observed trajectories.
% \end{remark}

\begin{remark}[Sampling over \ensuremath{[0,T]\times D}]
In Theorem~\ref{thm:lr}, the Fokker--Planck residual is sampled on the
same domain \([0,T]\times D\) for all controls. This follows
Assumption~\ref{as:p_s}: the coefficients are localized on a fixed bounded
state domain \(D\subset\mathbb R^n\), independently of the control. One
could instead allow a control-dependent domain \(D(u)\), provided the
localization and sampling assumptions are restated accordingly. We leave
this extension to future work.
\end{remark}

\subsection{Refined \texorpdfstring{\(L^{2}\)}{L2} Learning Rates}\label{subsec:refined_L2_rates}

% Applying additional regularity conditions that finely measure the effective dimension of the learning problem generally enables the derivation of refined learning rates. In the context of kernel ridge regression, this is standardly achieved by measuring the regularity of the features and of the target, leading to learning rates that are adaptive to the strength of these regularities \citep{caponnetto2007optimal, ciliberto2020general}. This section focuses on such assumptions tailored to our specific least-squares problem: FP matching.

Refined rates in kernel ridge regression usually follow from two quantities:
the regularity of the regression target and the effective dimension of the
feature distribution \citep{caponnetto2007optimal, ciliberto2020general}.
We apply this principle to the kernel least-squares problem induced by
Fokker--Planck matching.

\paragraph{Fokker--Planck regression representation.}
Once the density \(p\) is fixed, the RKHS coefficient model of
Section~\ref{sec:proposed_method} represents every candidate pair
\((b,a)\in\mathcal F\) by a parameter \(w_{b,a}\in\bmG_n\) such that
\[
    L_{(b,a),\theta}p(\theta,t,x)
    =
    \langle w_{b,a},\tilde\phi(\theta,t,x)\rangle_{\bmG_n}
\]
for almost every \((\theta,t,x)\in\Theta\times[0,T]\times D\). Here
\(\bmG_n\) is the induced Hilbert space and
\(\tilde\phi:\Theta\times[0,T]\times D\to\bmG_n\) is the induced
Fokker--Planck feature map; both are explicitly constructed in
Appendix~\ref{subsec:preli}. Thus, the regression input is
\((\theta,t,x)\), the target is \(\partial_t p(\theta,t,x)\), and the
feature map is \(\tilde\phi\). Moreover, under Assumption~\ref{as:attain}, let \(w\in\bmG_n\) be the parameter
associated with the true coefficients. Then, because \(p\) solves the
Fokker--Planck equation,
\(\partial_t p(\theta,t,x)=
\langle w,\tilde\phi(\theta,t,x)\rangle_{\bmG_n}\) almost everywhere.

% Under Assumption~\ref{as:attain}, let \(w\triangleq w_{b,a}\) denote the
% parameter associated with the true coefficients. Since \(p\) solves the
% Fokker--Planck equation,
% \[
%     \partial_t p(\theta,t,x)
%     =
%     \langle w,\tilde\phi(\theta,t,x)\rangle_{\bmG_n}
% \]
% for almost every \((\theta,t,x)\in\Theta\times[0,T]\times D\).

% \paragraph{Induced regression representation.} The RKHS coefficient model of Section~\ref{sec:proposed_method} induces a least-squares regression form for Fokker--Planck matching. Its feature map \(\tilde\phi\), derived in Section~\ref{subsec:preli}, is built from the RKHS feature map for the coefficients, the density \(p\), and the spatial derivatives in the Fokker--Planck operator. Under Assumption~\ref{as:attain}, there exists \(w\in\bmG_n\) such that \[ \partial_t p(\theta,t,x) = \langle w,\tilde\phi(\theta,t,x)\rangle_{\bmG_n} \] for almost every \((\theta,t,x)\in\Theta\times[0,T]\times D\).

\begin{assumption}[Source condition]\label{as:source} Define \[ C \triangleq \mathbb E_{\theta\sim\mathbb{P}_c,\,(t,x)\sim\mathrm{Unif}([0,T]\times D)} \big[(\tilde\phi\otimes\tilde\phi)(\theta,t,x)\big]. \] There exists \(\alpha\in[0,1]\) such that \begin{align} \|C^{-\alpha/2}w\|_{\bmG_n}<\infty . \end{align} \end{assumption}

Assumption~\ref{as:source} is the standard kernel-ridge source condition, adapted to Fokker--Planck matching. The case \(\alpha=0\) reduces to Assumption~\ref{as:attain}. Concrete examples are discussed in Section~\ref{subsec:assumptions_examples}.

\begin{assumption}[Embedding conditions]\label{as:emb} For any \(\theta\in\Theta\) and \((t,x)\in[0,T]\times D\), define \[ C(\theta) \triangleq \E_{(t,x)\sim \mathrm{Unif}([0,T]\times D)} \big[(\tilde\phi\otimes\tilde\phi)(\theta,t,x)\big], \qquad C(t,x) \triangleq \mathbb E_{\theta\sim\mathbb{P}_c} \big[(\tilde\phi\otimes\tilde\phi)(\theta,t,x)\big]. \] \begin{itemize}[label={}] \item \textbf{(A8.1)} There exist \(r\in[0,1]\) and \(c>0\) such that \[ \|C(\theta)^{-r/2}\tilde\phi(\theta,t,x)\|_{\bmG_n}\le c \quad \mathbb{P}_c\text{-a.e. } \theta,\ \text{Lebesgue-a.e. }(t,x)\in[0,T]\times D . \] \item \textbf{(A8.2)} There exist \(s\in[0,1]\) and \(c>0\) such that \[ \|C(t,x)^{-s/2}\tilde\phi(\theta,t,x)\|_{\bmG_n}\le c \quad \mathbb{P}_c\text{-a.e. } \theta,\ \text{Lebesgue-a.e. }(t,x)\in[0,T]\times D . \] \end{itemize} \end{assumption}

Assumption~\ref{as:emb} adapts standard embedding conditions for kernel methods \citep{pillaud2018statistical, fischer2020sobolev, berthier2020tight} to our product-sampling structure. The conditions are imposed on the Fokker--Planck feature \(\tilde\phi\), not on the original RKHS feature map. They are also separated: \(C(\theta)\) averages over state-time points for a fixed control, while \(C(t,x)\) averages over controls for a fixed state-time point. Thus \(r\) and \(s\) measure embedding strength in the state-time and control directions. The case \(r=s=0\) holds whenever \(\tilde\phi\) is bounded; larger values encode stronger regularity. Concrete examples are given in Section~\ref{subsec:assumptions_examples}.

% Assumption~\ref{as:emb} adapts standard embedding conditions for kernel methods
% \citep{pillaud2018statistical, fischer2020sobolev, berthier2020tight}
% to two specific features of our problem. First, the conditions are imposed on
% the Fokker--Planck matching feature \(\tilde\phi\), rather than on the original
% RKHS feature map used to model the coefficients. Second, they are separated
% according to the product sampling structure of the estimator: \(C(\theta)\) averages
% the feature covariance over state-time points for a fixed control, whereas
% \(C(t,x)\) averages it over controls for a fixed state-time point. Thus \(r\)
% and \(s\) measure the embedding strength in the state-time and control
% directions, respectively. This separation allows the analysis to track the two
% sources of statistical error separately. These conditions hold with \(r=s=0\)
% whenever \(\tilde\phi\) is bounded, and become stronger as \(r\) and \(s\)
% increase. Concrete examples are discussed in
% Section~\ref{subsec:assumptions_examples}.

\begin{remark}[Product-space embedding and capacity] Assumption~\ref{as:emb} implies a corresponding product-space condition. By Jensen's operator inequality \citep{hansen2003jensen}, \( C(\theta) \preccurlyeq c\,\mathbb E_{t,x}[C(t,x)^s] \preccurlyeq c\,C^s, \qquad \|C^{-rs/2}\tilde\phi(\theta,t,x)\|_{\bmG_n}\le c . \) Consequently, \[ \Tr(C^{1-rs}) = \mathbb E\!\left[ \|C^{-rs/2}\tilde\phi\|_{\bmG_n}^2 \right] \le c. \] Thus Assumption~\ref{as:emb} implies the usual capacity condition on the product space with exponent \(rs\). The embedding condition is stronger than this trace condition: in noisy kernel ridge regression, capacity conditions improve rates from \(N^{-1/4}\) to \(N^{-1/2}\), while embedding conditions can yield arbitrarily fast polynomial rates in the noiseless setting. \end{remark}

We now derive refined rates for the estimator. Throughout, for real numbers \(x,y\), we write
\(x\wedge y\triangleq\min\{x,y\}\) and
\(x\vee y\triangleq \max\{x,y\}\).

\begin{theorem}[Refined \(L^2\) learning rates]\label{th:refined}
Suppose that the assumptions of Theorem~\ref{thm:lr} hold, together with
Assumptions~\ref{as:source} and~\ref{as:emb}. Then there exist constants
\(c_1,c_2>0\), independent of \(N\), \(K\), and \(\delta\), such that the
following holds for any \(\delta\in(0,1]\). Define
\(
\varepsilon_\infty
\triangleq
\sup_{k\in\llbracket 1,K\rrbracket}
\mathcal{E}_{\infty}(\hat p_k,p(\theta_k)),
\)
where \(\mathcal{E}_{\infty}\) is the Sobolev-type error defined in
Appendix~\ref{subsec:proof_refined_L2}. Set
\(
\lambda
=
c_2\left[
\left(\frac{18c}{N}\log^2\frac{N}{\delta}\right)^{\frac{1}{1-r}}
+
\left(\frac{18c}{K}\log^2\frac{K}{\delta}\right)^{\frac{1}{1-s}}
\right].
\)
If
\begin{align}
    \varepsilon + N^{-1}\varepsilon_\infty
    \leq
    \left(\frac{\log(N/\delta)}{\sqrt N}\right)^{\frac{2+\alpha}{1-r}}
    +
    \left(\frac{\log(K/\delta)}{\sqrt K}\right)^{\frac{2+\alpha}{1-s}},
\end{align}
then, with probability at least \(1-\delta\),
\begin{align}
    \|p_{\hat b,\hat a}-p\|_{L^2(\mathbb{P}_c)L^2([0,T]\times\R^n)}
    \leq
    c_1\log\frac{2}{\delta}
    \left(
    \frac{\log((K\wedge N)/\delta)}{\sqrt{K\wedge N}}
    \right)^{\frac{\alpha+1}{1-(r\wedge s)}}.
\end{align}
\end{theorem}

Theorem~\ref{th:refined} shows how Fokker--Planck matching benefits from
regularity. The source parameter \(\alpha\) measures the regularity of
\(\partial_t p\), while \(r\) and \(s\) measure the embedding strength of
the induced feature map in the state-time and control directions. The
smaller of \(r\) and \(s\) is the bottleneck. As \(r\wedge s\) approaches
one, the Fokker--Planck matching rate can become arbitrarily fast
polynomial. As before, the overall accuracy is
limited by the first-stage density estimation errors.

\subsection{\texorpdfstring{\(L^{\infty}\)}{L-infinity} Learning Rates}\label{subsec:Linfty_rates}

The previous bounds control the density-flow error on average over controls.
This may be insufficient when the learned dynamics are later evaluated for a
particular control. Below, we provide a uniform-in-control
guarantee, while remaining \(L^2\) in time and state.

\begin{theorem}[Uniform-in-control learning rates]
\label{thm:Linfinite}
Under the same conditions as Theorem~\ref{th:refined}, there exists a constant
\(c>0\), independent of \(N\), \(K\), and \(\delta\), such that for any
\(\delta \in (0,1]\), with probability at least \(1-\delta\),
\begin{align}
    \|p_{\hat b,\hat a} - p\|_
    {L^{\infty}(\mathbb{P}_c)L^2([0,T]\times\mathbb R^n)}
    \leq
    c \log \frac{2}{\delta}
    \left(
        \frac{\log \frac{K \wedge N}{\delta}}
        {\sqrt{K \wedge N}}
    \right)^{
        \frac{\alpha + r \wedge s}{1 - r \wedge s}
    }.
\end{align}
\end{theorem}

The \(L^\infty(\mathbb{P}_c)\) norm is an essential supremum over controls sampled from \(\mathbb{P}_c\). If \(\mathbb{P}_c\) has full support on \(\Theta\) and the density-flow error is continuous in \(\theta\), this corresponds to a uniform bound over the admissible control family in the \(L^2([0,T]\times\mathbb R^n)\) norm. Theorem~\ref{thm:Linfinite} strengthens Theorems~\ref{thm:lr} and~\ref{th:refined} by replacing the averaged-in-control norm with an essential supremum over controls, but at a slower rate: the exponent \((\alpha+1)/(1-r\wedge s)\) becomes \((\alpha+r\wedge s)/(1-r\wedge s)\).

\subsection{CVaR Learning Rates}\label{subsec:cvar_lr}

The \(L^\infty\)-in-control guarantee controls the density-flow error under each admissible control. In safety-sensitive settings, one may also want to control tail-sensitive quantities of the learned dynamics. We show that, under a moment condition, the \(L^\infty\) rates imply guarantees for Conditional Value at Risk.

\paragraph{Conditional Value at Risk.} For \(\rho\in(0,1]\), the Conditional Value at Risk of a real-valued random variable \(Y\) is \begin{align} \mathrm{CVaR}_{\rho}(Y) \triangleq \inf_{\eta\in\mathbb R} \left\{ \eta + \rho^{-1}\mathbb E[(Y-\eta)_+] \right\}. \end{align} When \(Y\) is a loss, \(\mathrm{CVaR}_{\rho}(Y)\) is the expected loss in the worst \(\rho\)-fraction of outcomes.

% \begin{remark}[Risk-averse control] \label{rk:risk_averse} A typical risk-averse control objective is \begin{align} \inf_{u\in\bmH} \mathbb E[f(X_u(T))] + \lambda D(f(X_u(T))), \end{align} where \(f\) is a terminal loss, \(D\) is a risk measure, and \(\lambda>0\) balances average performance and risk sensitivity. Compared with the risk-neutral case \(\lambda=0\), the second term penalizes risk. Choosing \(D\) as CVaR emphasizes adverse tail events and is widely used in risk management and stochastic control \citep{shapiro2021lectures, miller2017optimal}. \end{remark}

\begin{remark}[Risk-averse control] \label{rk:risk_averse}
A typical risk-averse control objective is
\begin{align}
\inf_{\theta\in\Theta}
\mathbb E[f(X_{b,\sigma}^{\theta}(T))]
+
\lambda D(f(X_{b,\sigma}^{\theta}(T))) ,
\end{align}
where \(f\) is a terminal loss, \(D\) is a risk measure, and
\(\lambda>0\) balances average performance and risk sensitivity. Compared
with the risk-neutral case \(\lambda=0\), the second term penalizes risk.
Choosing \(D\) as CVaR emphasizes adverse tail events and is widely used in
risk management and stochastic control
\citep{shapiro2021lectures, miller2017optimal}.
\end{remark}

\begin{lemma}[CVaR stability in total variation] \label{lem:cvar_l1}
Let \(f\in L^\infty(\mathbb R^n)\), let \(\rho\in(0,1]\), and let \(X_1,X_2\) be
\(\mathbb R^n\)-valued random variables with densities \(p_1,p_2\). Assume that the laws
of \(f(X_1)\) and \(f(X_2)\) have no atom at their value-at-risk levels. Then
\begin{align}
\left| \mathrm{CVaR}_{\rho}(f(X_1)) - \mathrm{CVaR}_{\rho}(f(X_2)) \right|
\leq
\frac{2\|f\|_\infty}{\rho}
\|p_1-p_2\|_{L^1(\mathbb R^n)} .
\end{align}
\end{lemma}

\begin{lemma}[\(L^2\)-to-\(L^1\) interpolation with moments] \label{lem:L1_fastdecay} Let \(g\in L^2(\mathbb R^n)\). If, for some \(\beta>0\), \[ \int_{\mathbb R^n} \|x\|^\beta |g(x)|\,dx < \infty, \] then \begin{align} \|g\|_{L^1(\mathbb R^n)} \leq \|g\|_{L^2(\mathbb R^n)}^{\frac{\beta}{\beta+n/2}} \left( 3 + \int_{\mathbb R^n} \|x\|^\beta |g(x)|\,dx \right). \end{align} \end{lemma}

\begin{lemma}[Uniform moment bound] \label{lem:bounded_moment}
Let \(X_{\hat b,\hat\sigma}^{\theta}\) solve the SDE with coefficients
\((\hat b,\hat\sigma)\), where
\(\hat a=\hat\sigma\hat\sigma^\top\), driven by the control \(u_\theta\),
with \(\theta\in\Theta\). Suppose that the well-posedness and localization
assumptions of Section~\ref{subsec:fp_ineq} hold for \((\hat b,\hat a)\).
If \(\mathbb E[\|X_0\|^\beta]<\infty\) for some \(\beta>2\), then there
exists \(C>0\), depending only on the constants defining \(\bmF\), such that
\begin{align}
\sup_{\theta\in\Theta}\sup_{t\in[0,T]}
\mathbb E[\|X_{\hat b,\hat\sigma}^{\theta}(t)\|^\beta]
\leq C .
\end{align}
\end{lemma}

% \begin{lemma}[Uniform moment bound] \label{lem:bounded_moment} Let \(X_{\hat b,\hat\sigma}(u)\) solve the SDE with coefficients \((\hat b,\hat\sigma)\), where \(\hat a=\hat\sigma\hat\sigma^\top\) and \((\hat b,\hat a)\in\bmF\), under a control \(u_\theta\), \(\theta\in\Theta\). Suppose that the well-posedness and localization assumptions of Section~\ref{subsec:fp_ineq} hold for \((\hat b,\hat a)\). If \(\mathbb E[\|X_0\|^\beta]<\infty\) for some \(\beta>2\), then there exists \(C>0\), depending only on the constants defining \(\bmF\), such that \begin{align} \sup_{t\in[0,T]} \mathbb E[\|X_{\hat b,\hat\sigma}(u,t)\|^\beta] \leq C . \end{align} \end{lemma}

Combining these lemmas gives CVaR stability under \(L^2\) density errors.

% \begin{theorem}[CVaR stability] \label{th:cvar} Suppose that the well-posedness and localization assumptions of Section~\ref{subsec:fp_ineq} hold for both \((b,a)\) and \((\hat b,\hat a)\), and that \(\mathbb E[\|X_0\|^\beta]<\infty\) for some \(\beta>2\). Let \(\sigma\) and \(\hat\sigma\) satisfy \(a=\sigma\sigma^\top\) and \(\hat a=\hat\sigma\hat\sigma^\top\). Then, for any \(\theta\in\Theta\), any \(\rho\in(0,1]\), any \(f\in L^\infty(\mathbb R^n)\), and any \(t\in[0,T]\), \begin{align} \left| \mathrm{CVaR}_{\rho} \bigl(f(X_{\hat b,\hat\sigma}(u,t))\bigr) - \mathrm{CVaR}_{\rho} \bigl(f(X_{b,\sigma}(u,t))\bigr) \right| \leq \frac{c\|f\|_\infty}{\rho} \|p_{\hat b,\hat a}(u,t)-p_{b,a}(u,t)\|_ {L^2(\mathbb R^n)}^{\frac{\beta}{\beta+n/2}}, \end{align} where \(c>0\) depends only on the constants in the assumptions, not on \(u\), \(t\), \(f\), or \(\rho\). \end{theorem}

\begin{theorem}[CVaR stability] \label{th:cvar}
Suppose that the well-posedness and localization assumptions of
Section~\ref{subsec:fp_ineq} hold for both \((b,a)\) and
\((\hat b,\hat a)\), and that
\(\mathbb E[\|X_0\|^\beta]<\infty\) for some \(\beta>2\).
Let \(\sigma\) and \(\hat\sigma\) satisfy
\(a=\sigma\sigma^\top\) and
\(\hat a=\hat\sigma\hat\sigma^\top\).
Assume moreover that the atomlessness condition of
Lemma~\ref{lem:cvar_l1} holds for
\(f(X_{\hat b,\hat\sigma}^{\theta}(t))\) and
\(f(X_{b,\sigma}^{\theta}(t))\). Then, for any \(\theta\in\Theta\), any
\(\rho\in(0,1]\), any \(f\in L^\infty(\mathbb R^n)\), and any \(t\in[0,T]\),
\begin{align}
\left|
\mathrm{CVaR}_{\rho}
\bigl(f(X_{\hat b,\hat\sigma}^{\theta}(t))\bigr)
-
\mathrm{CVaR}_{\rho}
\bigl(f(X_{b,\sigma}^{\theta}(t))\bigr)
\right|
\leq
\frac{c\|f\|_\infty}{\rho}
% \|p_{\hat b,\hat a}(\theta,t)-p_{b,a}(\theta,t)\|_
% {L^2(\mathbb R^n)}^{\frac{\beta}{\beta+n/2}} .
\|p_{\hat b,\hat a}(\theta,t,\cdot)
-
p_{b,a}(\theta,t,\cdot)\|_
{L^2(\mathbb R^n)}^{\frac{\beta}{\beta+n/2}} .
\end{align}
The constant \(c>0\) depends only on the constants in the assumptions and on
the control family, and is independent of \(\theta\), \(t\), \(f\), and \(\rho\).
\end{theorem}

Combining Theorem~\ref{th:cvar} with Theorem~\ref{thm:Linfinite} gives
high-probability CVaR learning rates in \(L^2([0,T])\) over time. Up to multiplicative constants, the resulting bound is
\[
\frac{1}{\rho}
\left[
\log\frac{2}{\delta}
\left(
\frac{\log((K\wedge N)/\delta)}{\sqrt{K\wedge N}}
\right)^{\frac{\alpha+r\wedge s}{1-r\wedge s}}
\right]^{\frac{\beta}{\beta+n/2}} .
\]

% Combining Theorem~\ref{th:cvar} with Theorem~\ref{thm:Linfinite} gives high-probability CVaR learning rates. Under the assumptions of Theorem~\ref{thm:Linfinite} and the moment condition above, the CVaR error is bounded, up to multiplicative constants, by \begin{align} \frac{1}{\rho} \left[ \log\frac{2}{\delta} \left( \frac{\log\frac{K\wedge N}{\delta}} {\sqrt{K\wedge N}} \right)^{ \frac{\alpha+r\wedge s}{1-r\wedge s} } \right]^{ \frac{\beta}{\beta+n/2} } . \end{align}

\subsection{Application to Sobolev Spaces}\label{subsec:assumptions_examples}

We now give a Sobolev setting in which the abstract assumptions of the previous sections
can be checked explicitly. Sobolev regularity has already been used above to state the
well-posedness and Fokker--Planck stability assumptions. Here the point is different: we
use Sobolev spaces to verify the source and embedding assumptions of Section~4.2, and
therefore to obtain rates with explicit dependence on the dimensions \(m\) and \(n\) and on
the smoothness parameter \(\nu\).
% We now instantiate the abstract assumptions of the previous sections for Sobolev classes of coefficients and smoothly parametrized controls. This example makes explicit how the learning rates depend on the state dimension, the dimension of the control parametrization, and the Sobolev regularity.

% Sobolev regularity plays two different roles in the paper. For the Fokker--Planck analysis, it is an analytic assumption on the SDE coefficients and densities. For the learning rates, it is used to verify statistical assumptions on the induced regression problem, namely the source and embedding conditions.

\begin{lemma}[Sobolev instantiation] \label{lem:induced_emb} Let \(\Theta\subset\mathbb R^m\), \(D\subset\mathbb R^n\), and \(V\subset\mathbb R^d\) be bounded Lipschitz domains. Let \(\nu,\tau\in\mathbb N\) satisfy \(\nu>\max\{m,n+1\}/2\) and \(\tau>2\nu+2+(n+d+1)/2\). Assume that \[ (\theta,t)\mapsto u_\theta(t) \in H^{2\nu+2}(\Theta\times[0,T];\mathbb R^d), \qquad u_\theta(t)\in V, \] and \[ (\theta,t,x)\mapsto p(\theta,t,x) \in H^{2\nu+2}(\Theta\times[0,T]\times D). \] Assume that \(\mathbb{P}_c\) is the uniform probability measure on \(\Theta\). Choose the coefficient hypothesis space to be the Sobolev RKHS of order \(\tau\) on \([0,T]\times D\times V\), with the same ellipticity and localization constraints as in the estimator, and assume attainability. Then the source condition holds with \(\alpha=0\), and the embedding assumptions hold for \[ r<1-\frac{n+1}{2\nu}, \qquad s<1-\frac{m}{2\nu}. \] \end{lemma}

\begin{remark}[Regularity of the density flow]\label{rk:density_regularity} Lemma~\ref{lem:induced_emb} assumes Sobolev regularity of the density flow with respect to the control parameter: \[ (\theta,t,x)\mapsto p(\theta,t,x) \in H^{2\nu+2}(\Theta\times[0,T]\times D). \] This is a regularity assumption on the dependence of the Fokker--Planck solution on the parametrized control. It is natural when the composed coefficients \((t,x,\theta)\mapsto (b,a)(t,x,u_\theta(t))\) are sufficiently smooth in \((t,x,\theta)\). A full proof of this parameter-regularity property is beyond the scope of this paper; related parabolic regularity results can be found in \citet[Section~8.3.1]{lunardi2012analytic} and \citet[Theorem~7]{friedman2008partial}. \end{remark}

We obtain the following Sobolev-specialized rate.

\begin{corollary}[Sobolev coefficients]\label{cor:sob}
Under the assumptions of Lemma~\ref{lem:induced_emb}, suppose that the density
estimation error is sufficiently small in the sense required by
Theorem~\ref{thm:Linfinite}. Then, for every \(\zeta>0\), there exists
\(c>0\), independent of \(N\), \(K\), and \(\delta\), such that, with
probability at least \(1-\delta\),
\begin{align}
\|p_{\hat b,\hat a}-p\|_ {L^\infty(\mathbb{P}_c)L^2([0,T] \times \mathbb R^n)}
\leq c\log\frac{2}{\delta}
\left(
\frac{\log\frac{K\wedge N}{\delta}}
{\sqrt{K\wedge N}}
\right)^{
\frac{2\nu}{m\vee(n+1)}-1-\zeta
}.
\end{align}
\end{corollary}

\begin{proof}
Apply Lemma~\ref{lem:induced_emb} to Theorem~\ref{thm:Linfinite}. By
Lemma~\ref{lem:induced_emb}, \(r\wedge s\) can be chosen arbitrarily close to
\[
1-\frac{m\vee(n+1)}{2\nu},
\]
but strictly smaller. Since \(\alpha=0\), Theorem~\ref{thm:Linfinite} gives
any exponent strictly smaller than
\[
\frac{1-\frac{m\vee(n+1)}{2\nu}}
{\frac{m\vee(n+1)}{2\nu}}
=
\frac{2\nu}{m\vee(n+1)}-1.
\]
This gives the claim.
\end{proof}

% \begin{corollary}[Sobolev coefficients]\label{cor:sob} Under the assumptions of Lemma~\ref{lem:induced_emb}, suppose that the density estimation error is sufficiently small in the sense required by Theorem~\ref{thm:Linfinite}. Then there exists \(c>0\), independent of \(N\), \(K\), and \(\delta\), such that, with probability at least \(1-\delta\), \begin{align} \|p_{\hat b,\hat a}-p\|_ {L^\infty(\mathbb{P}_c)L^2([0,T] \times \mathbb R^n)} \leq c\log\frac{2}{\delta} \left( \frac{\log\frac{K\wedge N}{\delta}} {\sqrt{K\wedge N}} \right)^{ \frac{2\nu}{m\vee(n+1)}-1 }. \end{align} \end{corollary}
% \begin{proof} Apply Lemma~\ref{lem:induced_emb} to Theorem~\ref{thm:Linfinite}. Since \(r\wedge s=1-\frac{m\vee(n+1)}{2\nu}\) and \(\alpha=0\), the exponent becomes \[ \frac{r\wedge s}{1-r\wedge s} = \frac{2\nu}{m\vee(n+1)}-1 . \] \end{proof}

Corollary~\ref{cor:sob} shows the nonparametric nature of the problem: the exponent worsens with the state-time and control-parametrization dimensions, and improves with the Sobolev regularity \(\nu\). Below, we suppress the arbitrarily small loss \(\zeta>0\) in the exponent and write the endpoint exponent \(\frac{2\nu}{m\vee(n+1)}-1\) for readability.

% \begin{remark}[Statistical optimality] \label{rk:optimal_rate} When \(N\) is large, Corollary~\ref{cor:sob} gives, up to logarithmic factors, the control-sampling rate \[ K^{-\frac{\nu}{m\vee(n+1)}+\frac12}. \] Here \(\nu\) is the Sobolev regularity of the induced Fokker--Planck regression problem. Obtaining this effective regularity requires stronger upstream smoothness: regularity \(2\nu+2\) for the controls and density flow, and Sobolev order \(Q>2\nu+2+(n+d+1)/2\) for the coefficient hypothesis space. These requirements reflect the composition with the parametrized controls and the spatial derivatives in the Fokker--Planck operator. For comparison, minimax \(L^\infty\) rates for noiseless regression of \(\nu\)-smooth functions on \(d\)-dimensional compact domains are of order \[ \bmO\!\left( \left(\frac{\log^{1/2}K}{\sqrt K}\right)^{2\nu/d} \right) \] \citep{bauer2017nonparametric}. Thus, with \(d=m\vee(n+1)\), our rate has the same dimension--regularity scaling but loses one power in the exponent of \(\log K/\sqrt K\). The Sobolev specialization is only one way to verify the abstract source and embedding assumptions. The estimator does not require \(\alpha,r,s\) to be known in advance: if the induced Fokker--Planck regression problem satisfies better source, embedding, or effective-dimension properties, Theorems~\ref{th:refined} and~\ref{thm:Linfinite} reflect the corresponding faster rates. Identifying such structures, and understanding whether the Sobolev smoothness-transfer requirements and remaining rate loss are intrinsic, is left for future work. \end{remark}

\begin{remark}[Statistical optimality] \label{rk:optimal_rate}
When \(N\) is large, Corollary~\ref{cor:sob} gives, up to logarithmic
factors, the control-sampling rate
\[
    K^{-\frac{\nu}{m\vee(n+1)}+\frac12}.
\]
Here \(\nu\) is the Sobolev regularity of the induced Fokker--Planck
regression problem. Obtaining this effective regularity requires stronger
upstream smoothness: regularity \(2\nu+2\) for the controls and density
flow, and Sobolev order
\(\tau>2\nu+2+(n+d+1)/2\) for the coefficient hypothesis space. These
requirements reflect the composition with the parametrized controls and
the spatial derivatives in the Fokker--Planck operator.

For comparison, minimax \(L^\infty\) rates for noiseless regression of
\(\nu\)-smooth functions on \(d\)-dimensional compact domains are of order
\[
    \bmO\!\left(
    \left(\frac{\log^{1/2}K}{\sqrt K}\right)^{2\nu/d}
    \right)
\]
\citep{bauer2017nonparametric}. Thus, with \(d=m\vee(n+1)\), our rate has
the same dimension--regularity scaling but loses one power in the exponent
of \(\log K/\sqrt K\). The Sobolev specialization is only one way to instantiate the abstract
source and embedding assumptions with explicit exponents. Studying other
RKHS choices for the coefficients, such as Gaussian kernels, and the rates
they imply, is left for future work.
\end{remark}

We also obtain the following sample-complexity consequence.

\begin{corollary}[Number of sampled controls]\label{cor:comp_complexity} Under the assumptions of Lemma~\ref{lem:induced_emb}, let \(N\geq K\). To achieve precision \(\eta>0\) in the norm \[ \|p_{\hat b,\hat a}-p\|_ {L^\infty(\mathbb{P}_c)L^2([0,T] \times \mathbb R^n)} \leq \eta, \] it is sufficient, up to logarithmic factors, to take \begin{align} K = \mathcal O\!\left( \eta^{ -\frac{2(m\vee(n+1))} {2\nu-m\vee(n+1)} } \right). \end{align} \end{corollary}

In particular, if \(\nu\geq m\vee(n+1)\), then \(K=\mathcal O(\eta^{-2})\), up to logarithmic factors.

\begin{remark}[Computational scaling]\label{rk:comp_scaling} The Fokker--Planck matching step uses \(KN\) regression points: \(K\) sampled controls and \(N\) state-time collocation points. The resulting kernel system costs \(O((KN)^3)\) time and \(O((KN)^2)\) memory. Taking \(N=K\) and using \(K=O(\eta^{-2})\) from Corollary~\ref{cor:comp_complexity} gives cost \(O(\eta^{-12})\), up to logarithmic factors. Kernel approximations, such as Nystr\"om methods and random features, can substantially reduce this cost while preserving statistical rates under suitable regularity assumptions, and have scaled kernel methods to datasets with billions of samples \citep{rudi2015less,rudi2017generalization,meanti2020kernel}. Extending the present analysis to such approximations is left for future work. \end{remark}

\begin{remark}[Dependence on \(M\) and \(Q\)] The results above focus on the number \(K\) of sampled controls and on the Fokker--Planck matching step. They assume that the first-stage density estimators \(\hat p_k\) are sufficiently accurate, and therefore do not give a full statistical or computational analysis in terms of the number \(Q\) of trajectories per control and the number \(M\) of observation times. Finite-sample guarantees for density estimation are studied in \citet{bonalli2023non} for the error measure \(\varepsilon\). Extending these bounds to the stronger uniform error \(\varepsilon_\infty\) required by the \(L^\infty\)-in-control results would require additional uniform-in-time and uniform-in-space control, for instance through Sobolev embedding arguments. Here we focus on control sampling and Fokker--Planck matching, the main new elements of the controlled setting and, in our experiments, the main computational cost; see Section~\ref{sec:expe}. \end{remark}

% \begin{remark}[Dependence on \(M\) and \(Q\)]
% The results above focus on the number \(K\) of sampled controls and on the
% Fokker--Planck matching step. They assume that the first-stage density
% estimators \(\hat p_k\) are sufficiently accurate, and therefore do not provide
% a full statistical or computational analysis in terms of the number \(Q\) of
% trajectories per control and the number \(M\) of observation times.
% Finite-sample guarantees for the density-estimation step are studied in
% \citet{bonalli2023non} for the error measure \(\varepsilon\). Extending these
% bounds to the stronger uniform error \(\varepsilon_\infty\) required by the
% \(L^\infty\)-in-control results would require additional uniform-in-time and
% uniform-in-space control, through suitable Sobolev embedding
% arguments. In this work, we instead focus on the control-sampling and
% Fokker--Planck matching step, which are the main new elements of the
% controlled setting and, in our numerical experiments, constitute the main
% computational cost; see Section~\ref{sec:expe}.
% \end{remark}

\begin{remark}[State-dependent controls] \label{rk:closed_loop} For clarity, we focus on deterministic open-loop controls \(u_\theta(t)\). An important extension is to allow state-dependent controls \(u_\theta(t,x)\), as in feedback or closed-loop settings. Then the Fokker--Planck operator is evaluated along \(u_\theta(t,x)\), and the induced feature map contains additional chain-rule terms involving spatial derivatives of the control. Establishing the corresponding guarantees would require additional control regularity assumptions and careful tracking of constants. We leave this extension to future work. \end{remark}

\section{Numerical Experiments}\label{sec:expe}

We present numerical experiments for uncontrolled SDE estimation (Section~\ref{subsec:exp_uncontrolled}) and controlled SDE estimation (Section~\ref{subsec:exp_controlled}). The goal is to illustrate the main behavior of the proposed method.

\subsection{Uncontrolled SDE Estimation}\label{subsec:exp_uncontrolled}

% \paragraph{Purpose of uncontrolled SDE experiments.} We initiate our numerical study by implementing and evaluating the method proposed in \citet{bonalli2023non} for uncontrolled SDE estimation. Note that the method proposed in \citet{bonalli2023non} corresponds to our approach in the limit case where the control space is reduced to a singleton \(\bmH = \{u\}\). Therefore, this study provides important insights for the more intricate scenario of controlled SDEs considered in Section \ref{subsec:exp_controlled}.

\paragraph{Purpose of uncontrolled SDE experiments.}
We begin our numerical study by implementing and evaluating the method proposed in \citet{bonalli2023non} for uncontrolled SDE estimation. The method of \citet{bonalli2023non} corresponds to our approach in the limiting case where the control space is reduced to a singleton \(\bmH = \{u\}\). This study therefore provides useful insight before moving to the controlled SDE setting of Section~\ref{subsec:exp_controlled}.

\paragraph{Python open-source library.}
Our implementation is available as an open-source Python library at \url{lmotte/sde-learn}. The library includes documentation and example scripts with light computational demands. In particular, it supports the Nyström approximation \citep{rudi2015less}, which helps reduce the computational cost of Fokker--Planck matching. Further details, including derivations of the formulas used to implement the estimator and the computational complexity of each step, are given in Appendix~\ref{sec:implementation_uncontrolled}.

\paragraph{Experimental setup.} We consider an isotropic diffusion and use the pointwise positivity constraint described in Section~\ref{sec:proposed_method} for the diffusion coefficient. All kernels are Gaussian kernels \(k(x,y) = \exp(-\gamma \|x-y\|^2)\) with different parameters \(\gamma>0\). We select all hyperparameters by grid search on validation sets. The selection metrics are the log-likelihood for the probability density estimation step, i.e., \(\hat p \mapsto \sum_{l=1}^{M_{val}} \sum_{i=1}^{Q_{val}} \log(\hat p(t_l, x_i))\), and the mean squared error for the Fokker--Planck matching step, i.e., \((\hat b, \hat a) \mapsto \sum_{i=1}^{N} [\frac{\partial p}{\partial t}(t_i, x_i) - \bmL^{\hat b, \hat a} p(t_i, x_i)]^2\). Further details can be found in the code repository.

\paragraph{Computational considerations.} While the computational complexity of each step is given in Appendix~\ref{sec:implementation_uncontrolled}, we also report observed execution times to give a practical sense of scale. These experiments were run on a machine equipped with an Apple M3 Pro processor and 18 GB of RAM. For the 1D case, the probability density estimation step takes 0.015 seconds for training with 1000 sample paths and 100 time steps, and 0.11 seconds for prediction over a Fokker--Planck training set of size 2500. The Fokker--Planck matching step takes 6.0 seconds for training with 2500 data points and 3.7 seconds to generate 100 sample paths with 100 time steps. For the two 2D problems, probability density estimation training takes around 0.012 seconds for 3000 sample paths and 100 time steps, with prediction taking approximately 60 seconds over a Fokker--Planck training set of size 3000. The Fokker--Planck matching step takes about 10 seconds for training and 6 seconds to generate 100 sample paths with 100 time steps. These timings suggest that the method is computationally feasible for the moderate-size uncontrolled problems considered here.

\subsubsection{Linear Scalar SDE}

We first consider Ornstein--Uhlenbeck processes.

\paragraph{Ornstein--Uhlenbeck process.}
The Ornstein--Uhlenbeck (OU) process is a simple mean-reverting stochastic process. It is useful, for example, for modeling phenomena such as price volatility. It is defined by the SDE
\begin{align}
    dX(t) = \theta(\mu - X(t)) dt + \sigma dW(t),
\end{align}
where \(\mu\) is the mean to which the process reverts, \(\theta\) is the mean-reversion coefficient, and \(\sigma\) is the noise amplitude. The probability density of the OU process can be obtained explicitly from the Fokker--Planck equation. In particular, if \(X(0)\) is Gaussian with mean \(\mu_0 \in \R\) and covariance \(\sigma_0^2 > 0\), then \(X(t)\) remains Gaussian, with mean and covariance
\begin{align}
    &\mu(t) \triangleq e^{-\theta t} (\mu_0 -\mu) + \mu,
    &\Sigma(t) \triangleq (\sigma^2_0 - \frac{\sigma^2}{2 \theta})e^{-2 \theta t}  + \frac{\sigma^2}{2 \theta}.
\end{align}

\paragraph{Datasets.}
We consider an Ornstein--Uhlenbeck process with constant variance and mean increasing from \(0.5\) toward \(2.5\) over \(t\in[0,10]\). Specifically, we set \(\mu = 2.5, \theta = 0.5, \sigma^2 = \theta/4, \mu_0 = 0.5, \sigma_0^2 = \sigma^2 / (2\theta)\), and \(T=10\). In Figure~\ref{fig:ornstein_uhlenbeck_samples}, we plot 100 sample paths generated from this OU process. For the probability density estimation steps, we draw training and validation sets with \(Q/Q_{val}=1000/100\) sample paths and \(M/M_{val}=100/100\) time steps, respectively. For the Fokker--Planck matching step, we draw a training set \((t_i, x_i)_{i=1}^N = \{t_i\}_{i=1}^{50} \times \{x_i\}_{i=1}^{50}\) by drawing times and positions uniformly within well-chosen intervals, and we draw a validation set with the same size and distribution.

\paragraph{Step 1 (probability density estimation).}
In Figure~\ref{fig:kde}, we plot the true and estimated probability densities, \(p(t,x)\) and \(\hat{p}(t,x)\), respectively. These densities are plotted on a uniformly spaced temporal grid, offset by \(1/2\) unit from the training time discretization, and on a spatial grid with positions drawn uniformly within an interval. The estimated density gives a good visual approximation of the true dynamics.

\begin{figure}[H]
    \centering
    \begin{minipage}[t]{0.47\textwidth}
        \centering
        \includegraphics[width=\textwidth]{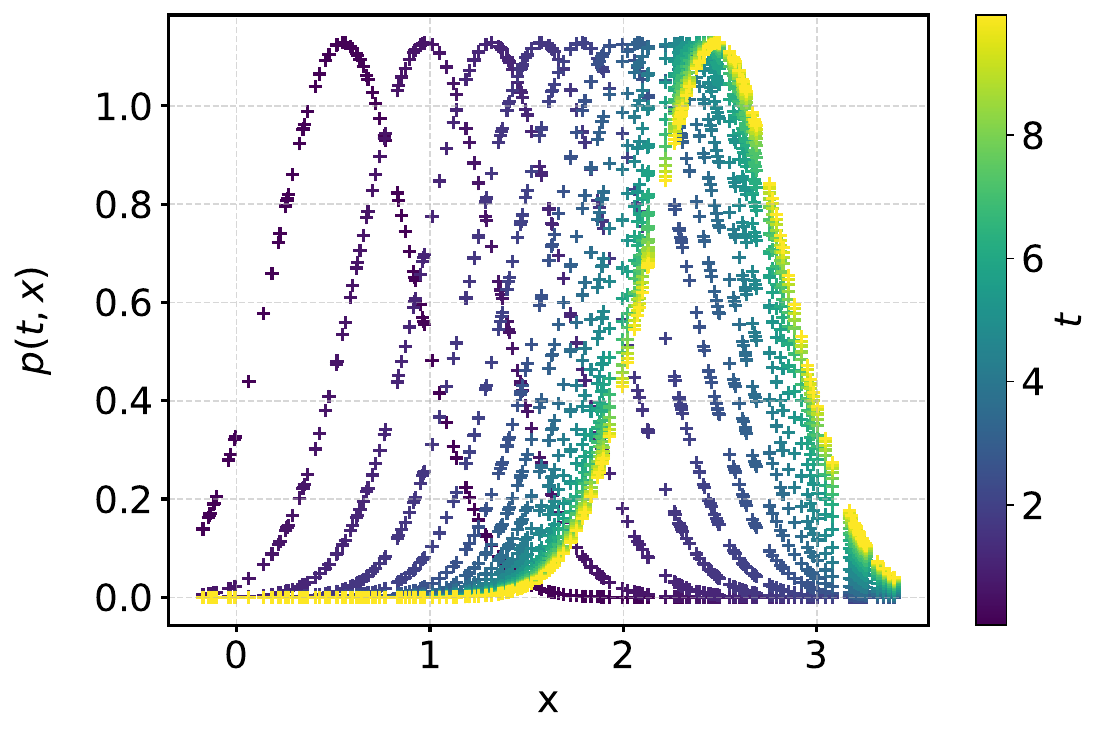}
    \end{minipage}\hfill
    \begin{minipage}[t]{0.47\textwidth}
        \centering
        \includegraphics[width=\textwidth]{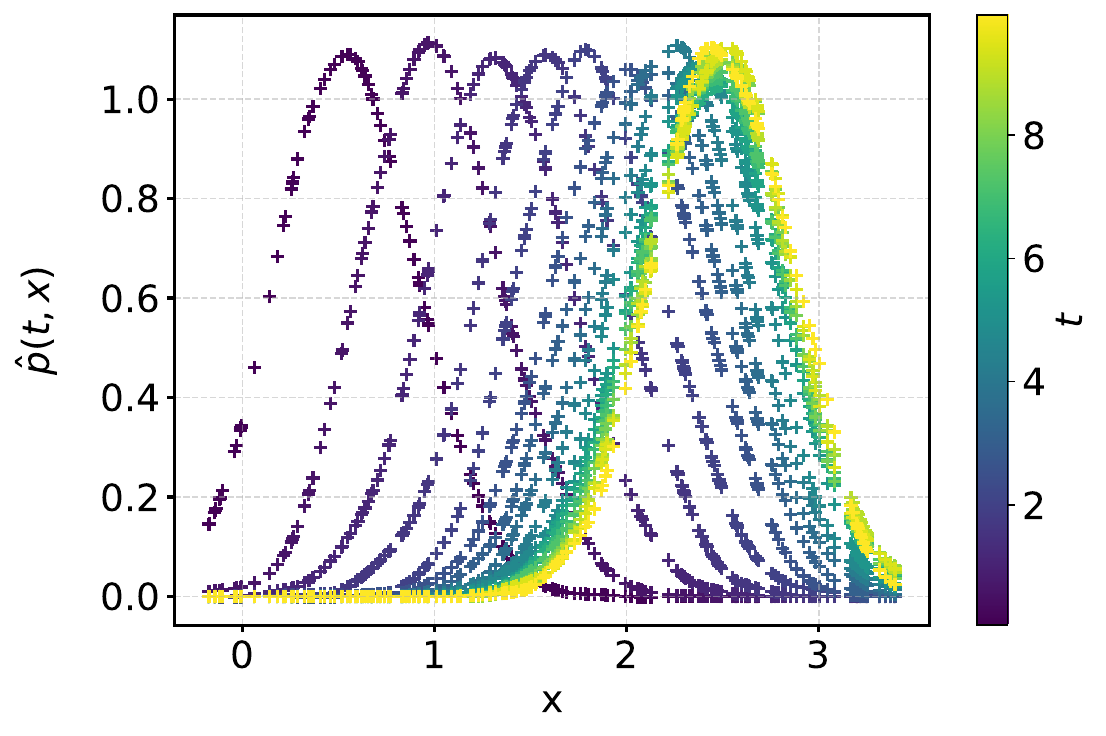}
    \end{minipage}
    \caption{OU process. True and estimated probability densities, $p(t,x)$ \textit{(left)} and $\hat p(t,x)$ \textit{(right)}, with respect to $x$ for several $t \in [0, 10]$.}
    \label{fig:kde}
\end{figure}

\paragraph{Step 2 (coefficient estimation via FP matching).}
In Figure~\ref{fig:estimated_coeff}, we plot the estimated coefficients obtained by FP matching. These coefficients differ significantly from the true coefficients. As discussed in Remark~\ref{rk:non_id}, measuring accuracy through the \(L^2\) distance to the true coefficients is not appropriate because the coefficients are not identifiable from densities alone. Our guarantees do not ensure recovery of the true coefficients; they ensure that the estimated coefficients reproduce the true dynamics at the distributional level. Specifically, they aim for the induced distribution \(p_{\hat{b}, \hat{a}}\) to be close to the true distribution \(p_{b, a}\) in \(L^2\).

\begin{figure}[H]
    \centering
    \begin{minipage}[t]{0.47\textwidth}
        \centering
        \includegraphics[width=\textwidth]{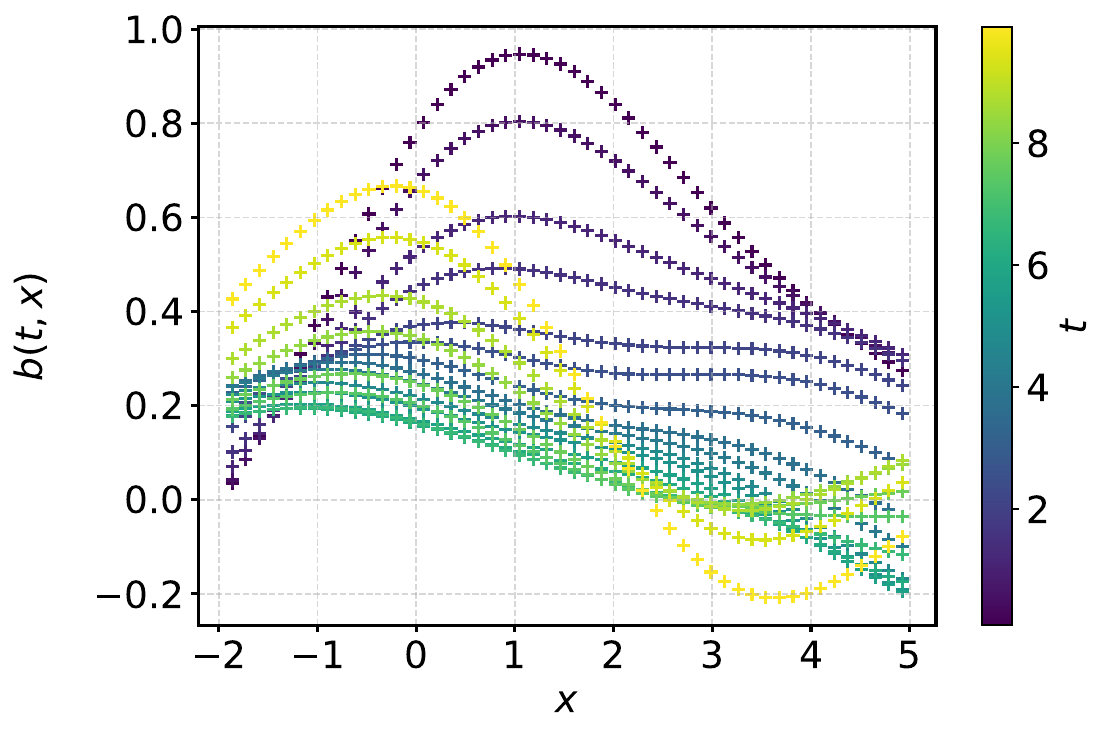} 
    \end{minipage}\hfill
    \begin{minipage}[t]{0.47\textwidth}
        \centering
        \includegraphics[width=\textwidth]{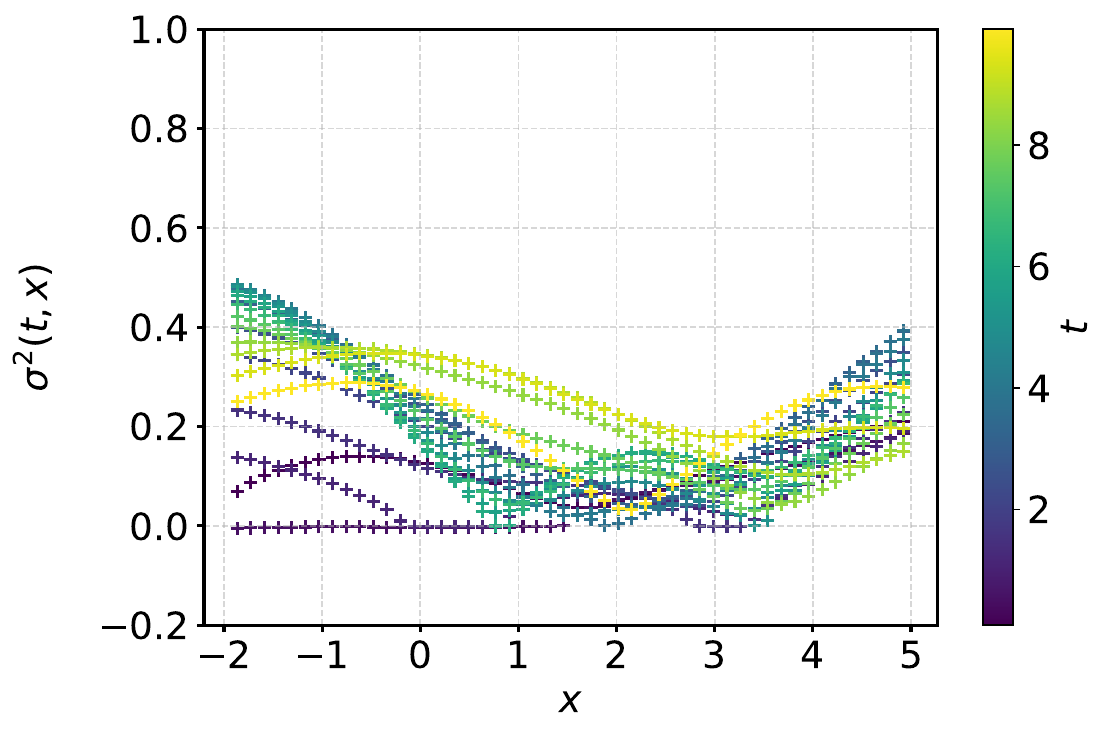} 
    \end{minipage}
    \caption{OU process. Estimated coefficients $\hat{b}(t,x)$ \textit{(left)} and $\hat{a}(t, x)$ \textit{(right)} as functions of $x$ for several \(t \in [0, 10]\).}
    \label{fig:estimated_coeff}
\end{figure}

\paragraph{Recovering the true dynamics.}
We draw and plot 100 sample paths from the true and estimated coefficients in Figure~\ref{fig:ornstein_uhlenbeck_samples} to assess recovery of the true dynamics. The estimated coefficients lead to visually comparable probability distributions, with close means and variances over time. Interestingly, while the distributions are similar, individual paths can look quite different.

\begin{figure}[H]
    \centering
    \begin{minipage}[t]{0.47\textwidth}
        \centering
        \includegraphics[width=\textwidth]{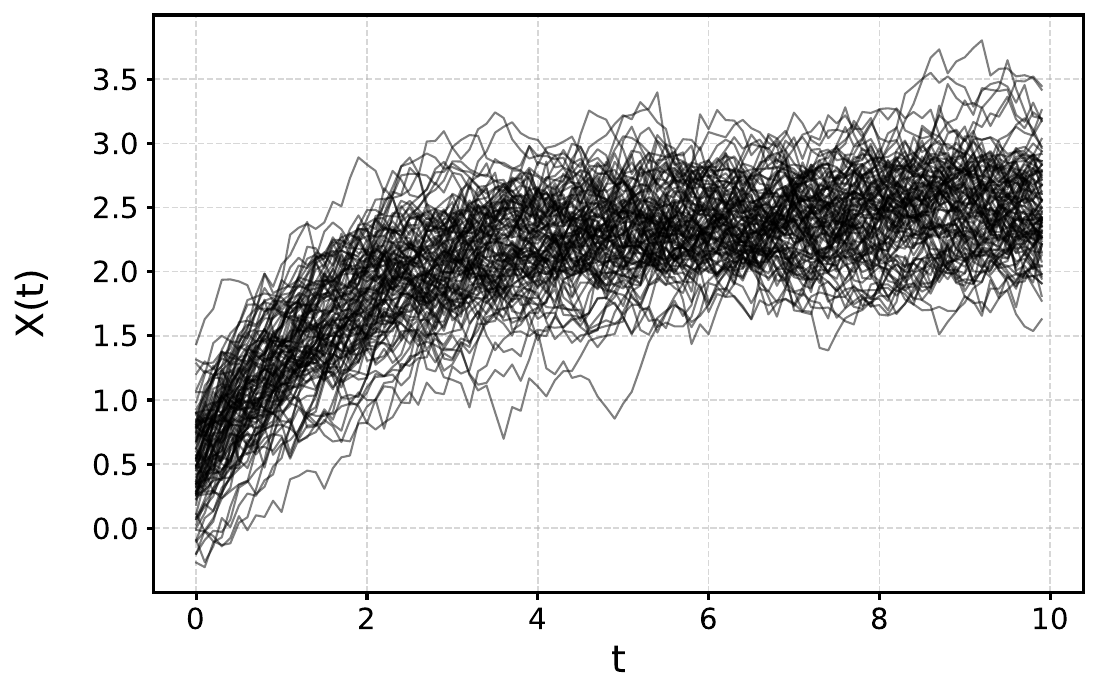} 
    \end{minipage}\hfill
    \begin{minipage}[t]{0.47\textwidth}
        \centering
        \includegraphics[width=\textwidth]{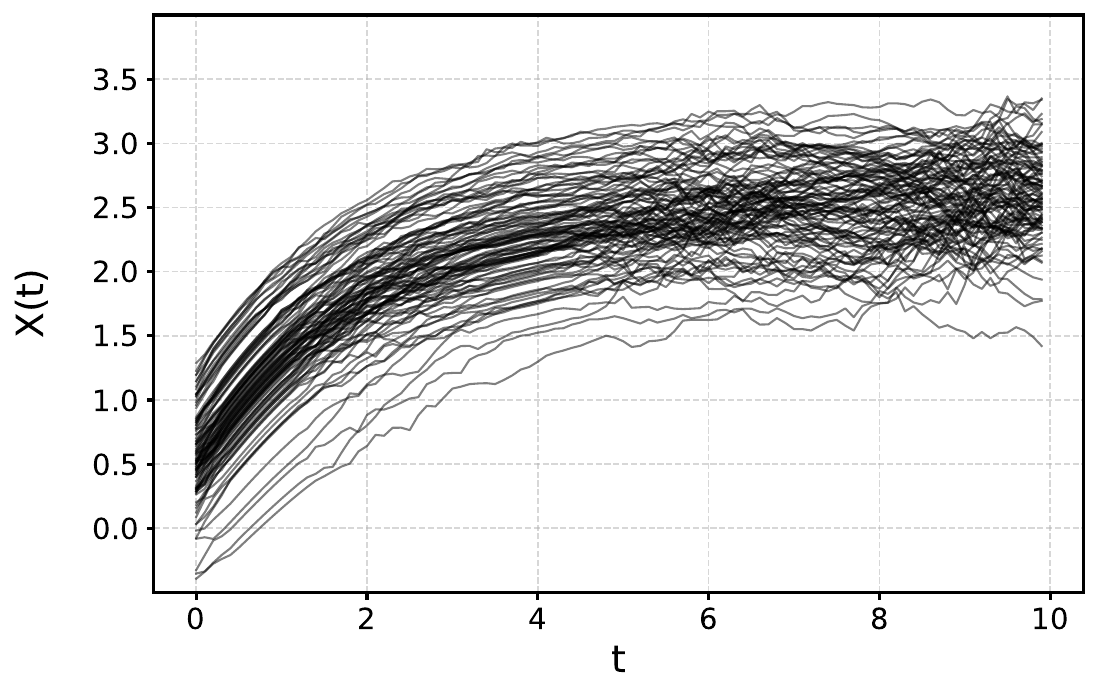} 
    \end{minipage}
    \caption{OU process. 100 sample paths from the SDEs associated with the true \textit{(left)} and estimated \textit{(right)} coefficients.}
    \label{fig:ornstein_uhlenbeck_samples}
\end{figure}

\subsubsection{Nonlinear Multivariate SDE}\label{subsubsec:dubins_finite_exp}

For nonlinear multivariate SDE experiments, we consider two examples: a Dubins process and a finite exponential sum process.

\paragraph{Dubins process.}
The Dubins process is defined by the SDE
\begin{align}\label{eq:mult_sde}
    dX(t) = v \times (\cos(u(t)), \sin(u(t))) dt + \sigma dW(t) \in \R^2,
\end{align}
with \(u(t)=\theta \sin(\frac{\pi t}{10})\), \(v, \theta \in \R\), and \(\sigma > 0\).

\begin{figure}[H]
    \centering

    \begin{subfigure}[t]{0.26\textwidth}
        \centering
        \includegraphics[height=4.5cm]{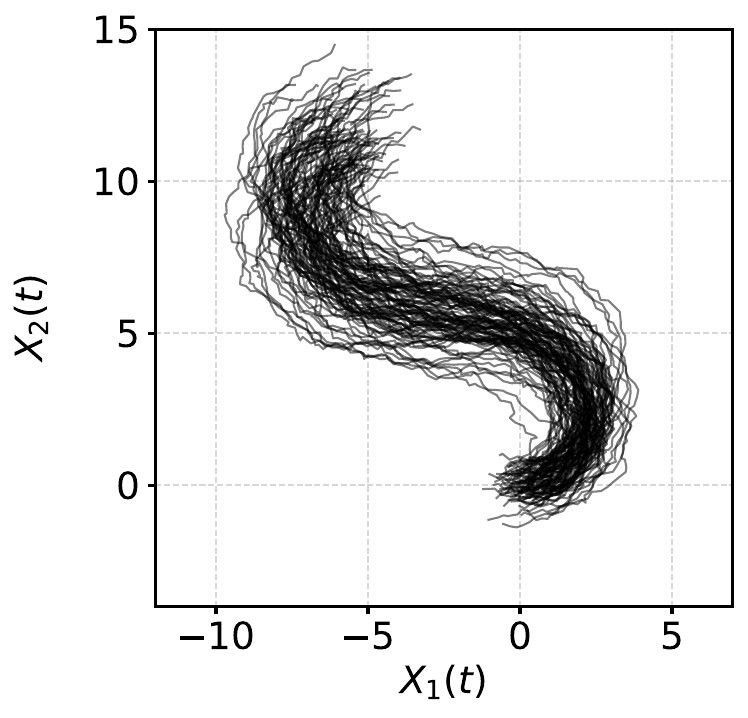}
        \makebox[\linewidth][c]{\hspace*{+1.4cm}\small True SDE samples.}
    \end{subfigure}%
    \hspace{1.2cm}%
    \begin{subfigure}[t]{0.31\textwidth}
        \centering
        \includegraphics[height=4.5cm]{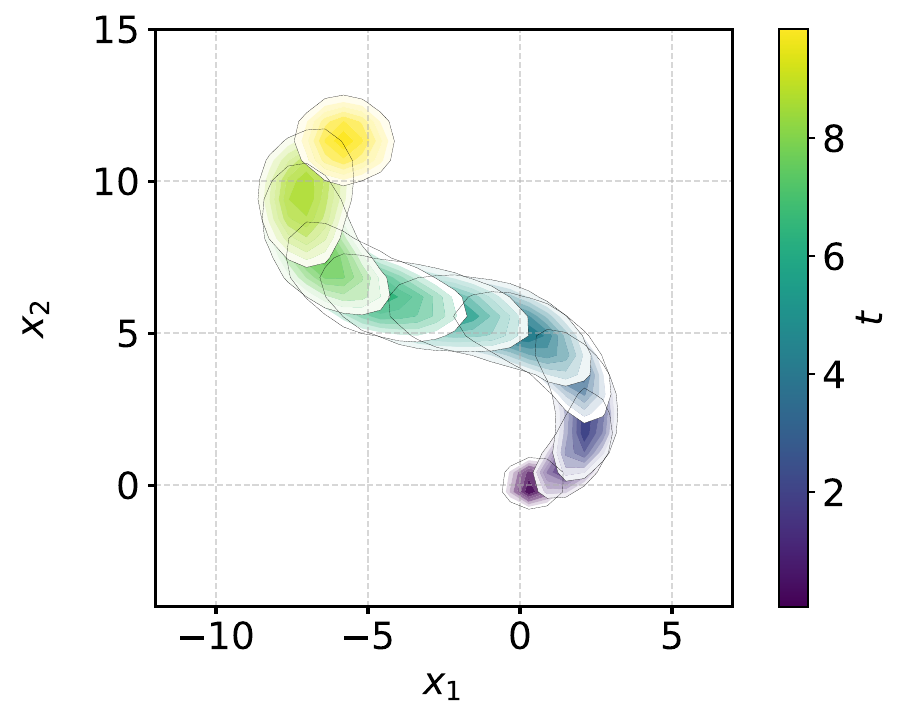}
        \makebox[\linewidth][c]{\hspace*{1.4cm}\small Estimated density.}
    \end{subfigure}%
    \hspace{1.2cm}%
    \begin{subfigure}[t]{0.26\textwidth}
        \centering
        \includegraphics[height=4.5cm]{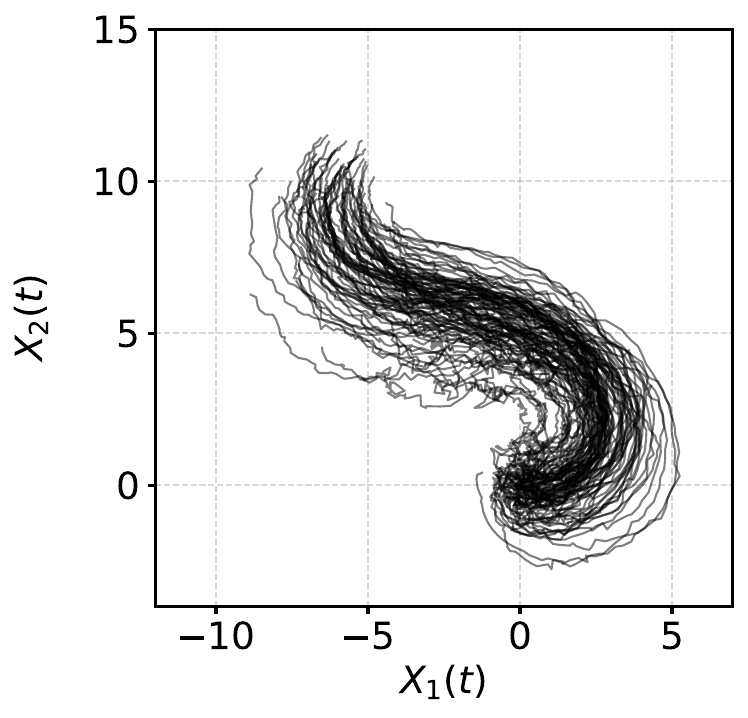}
        \makebox[\linewidth][c]{\hspace*{1.7cm}\small Estimated SDE samples.}
    \end{subfigure}%

    \caption{Dubins process. 100 sample paths from the true and estimated SDEs.}
    \label{fig:2d_exp_dubins}
\end{figure}
% \begin{figure}[H]
%     \centering
%     \begin{minipage}[t]{0.3\textwidth}
%         \centering
%         \includegraphics[height=4.5cm]{figures/512_true_samples_22.pdf} 
%         \caption*{True SDE samples.}
%     \end{minipage}\hfill
%     \begin{minipage}[t]{0.4\textwidth}
%         \centering
%         \includegraphics[height=4.5cm]{figures/512_p_pred_22.pdf}
%          \caption*{Estimated density.}
%     \end{minipage}\hfill
%     \begin{minipage}[t]{0.3\textwidth}
%         \centering
%         \includegraphics[height=4.5cm]{figures/512_samples_22.pdf} 
%         \caption*{Estimated SDE samples.}
%     \end{minipage}
%     \caption{Dubins process. 100 samples from the true and estimated SDEs.}
%     \label{fig:2d_exp_dubins}
% \end{figure}

\paragraph{Finite exponential sum (FES) process.}
We introduce the process defined by the SDE associated with the coefficients
\begin{align}
    &b(t, x) = \sum_{i=1}^{n_b} b_i \exp(-\gamma \|x-x_i\| |t-t_i| ),
    &\sigma(t, x) = \sum_{i=1}^{n_\sigma} s_i \exp(-\gamma \|x-x_i\| |t-t_i| ) I_{2},
\end{align}
for \((b_i)_i \in (\R^2)^{n_b}, (s_i)_i \in (\R)^{n_{\sigma}}, \gamma>0, n_b, n_\sigma\in \N^*\).

\begin{figure}[H]
    \centering
    \begin{minipage}[t]{0.3\textwidth}
        \centering
        \includegraphics[height=5.5cm]{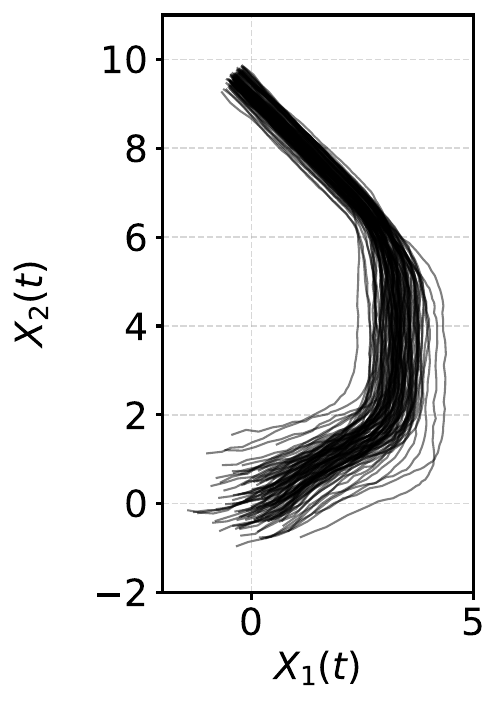} 
        \makebox[\linewidth][c]{\hspace*{+1.cm}\small True SDE samples.}
    \end{minipage}\hfill
    \begin{minipage}[t]{0.4\textwidth}
        \centering
        \includegraphics[height=5.5cm]{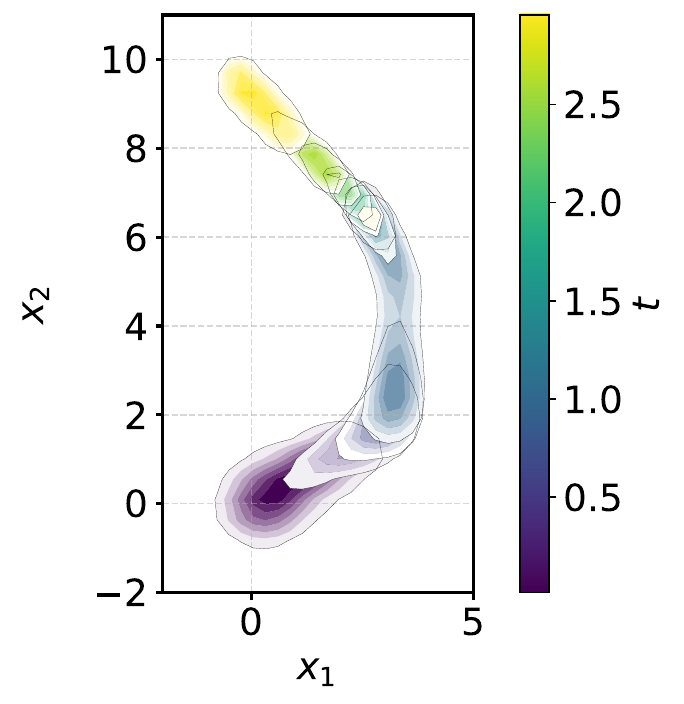}
        \makebox[\linewidth][c]{\hspace*{0.cm}\small Estimated density.}    \end{minipage}\hfill
    \begin{minipage}[t]{0.3\textwidth}
        \centering
        \includegraphics[height=5.5cm]{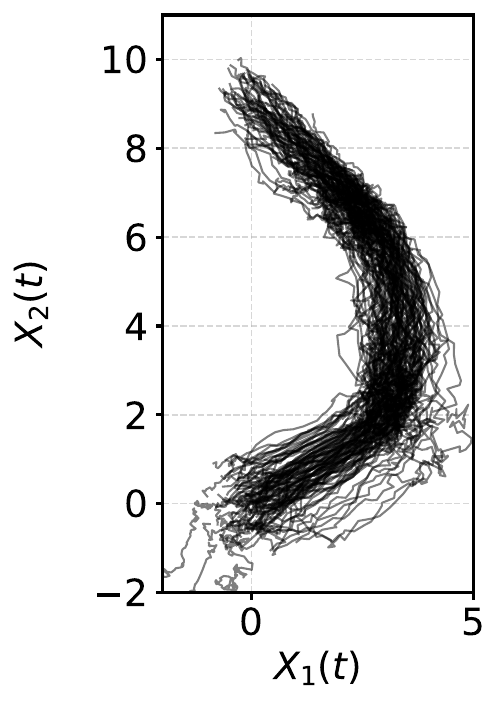} 
        \makebox[\linewidth][c]{\hspace*{1.2cm}\small Estimated SDE samples.}
    \end{minipage}
    \caption{FES process. 100 sample paths from the true and estimated SDEs.}
    \label{fig:2d_exp_fes}
\end{figure}

\paragraph{Datasets.}
For the Dubins process, we set \( T = 10 \), \( n = 2 \), \( M = 100 \), \( Q = 3000 \), \( v = 2 \), \( \theta = 3 \), and \( \sigma = 0.3 \). For the probability density estimation steps, the training and validation sets consist of \( Q/Q_{val} = 3000/100 \) sample paths and \( M/M_{val} = 100/10 \) time steps, respectively. For the Fokker--Planck matching step, training and validation sets are drawn by sampling time-position pairs from the training sample paths, with sizes \(N=3000\) and \(N_{val}=1000\). For the finite exponential sum process, the parameters are \( T = 3 \), \( n = 2 \), \( M = 100 \), \( Q = 3000 \), \( n_b = n_\sigma = 3 \), \( \gamma = 1 \), and \( (t_1, x_1) = (0, (0, 0)) \), \( (t_2, x_2) = (1, (1, 0)) \), and \( (t_3, x_3) = (3, (4, 6)) \). For the probability density estimation steps, the training and validation sets consist of \( Q/Q_{val} = 3000/100 \) sample paths and \( M/M_{val} = 100/100 \) time steps, respectively. For the Fokker--Planck matching step, we use a training set \(\{t_i\}_{i=1}^{30} \times \{x_i\}_{i=1}^{100}\), where the sets of times and positions are drawn uniformly in a well-chosen interval and two-dimensional box, respectively. We use a validation set \(\{t_i\}_{i=1}^{10} \times \{x_i\}_{i=1}^{100}\), where times and positions form uniform grids in a well-chosen interval and two-dimensional box. For both processes, the initial condition is \( X(0) \sim \bmN(0, 1/4 I_{\R^2}) \).

% \paragraph{Datasets.} For the Dubins process, we set the parameters as follows: \( T = 10 \), \( n = 2 \), \( M = 100 \), \( Q = 3000 \), \( v = 2 \), \( \theta = 3 \), and \( \sigma = 0.3 \).  For the probability density estimation steps, the training and validation sets consist of \( Q/Q_{val} = 3000/100 \) sample paths and \( M/M_{val} = 100/10 \) time steps, respectively. For the Fokker-Planck matching step, training and a validation set are drawn by sampling the set of training sample paths' time-position pairs, with size $N=3000/N_{val}=1000$. For the finite exponential sum process, the parameters are \( T = 3 \), \( n = 2 \), \( M = 100 \), \( Q = 3000 \), \( n_b = n_\sigma = 3 \), \( \gamma = 1 \), and \( (t_1, x_1) = (0, (0, 0)) \), \( (t_2, x_2) = (1, (1, 0)) \), and \( (t_3, x_3) = (3, (4, 6)) \). For the probability density estimation steps, the training and validation sets consist of \( Q/Q_{val} = 3000/100 \) sample paths and \( M/M_{val} = 100/100 \) time steps, respectively. For the Fokker-Planck matching step, we use a training set \(\{t_i\}_{i=1}^{30} \times \{x_i\}_{i=1}^{100}\) where the sets of times and positions are drawn uniformly in well-chosen interval and two-dimensional box, respectively. We use a validation set $\{t_i\}_{i=1}^{10} \times \{x_i\}_{i=1}^{100}$ where times and positions are uniform grids in well-chosen interval and two-dimensional box. For both processes, the initial condition is \( X(0) \sim \bmN(0, 1/4 I_{\R^2}) \).

\paragraph{Recovering the true dynamics.}
We perform density estimation and Fokker--Planck (FP) matching for both processes, selecting hyperparameters on the validation sets as above. We draw and plot 100 sample paths from both the true and estimated coefficients in Figure~\ref{fig:2d_exp_dubins} and Figure~\ref{fig:2d_exp_fes}. The estimated coefficients lead to probability distributions that are visually close to those of the true dynamics, with closely matching means and variances over time.

\paragraph{Observations.}
Our experiments highlight two difficulties in estimating SDEs.
\begin{itemize}
    \item \textbf{Cumulative error.} After the coefficients are estimated, the learned SDE is simulated by free rollout, without teacher forcing, i.e., without being reset or corrected using samples from the true process at intermediate times. Therefore, local errors in the estimated coefficients can accumulate over time. Even if the coefficients are accurate near a given time \(t_0\), earlier errors on \(0 \leq t < t_0\) may already have changed the distribution of the learned process at time \(t_0\).

    \item \textbf{Error amplification.} The learned process may enter regions of the time--state space where the training density is low and the coefficients are weakly constrained by data. In such regions, coefficient errors can be amplified, sometimes leading to path divergence or numerical termination. Increasing the variance of \(X(0)\) can improve coverage of the state space and reduce this instability.
\end{itemize}

\subsection{Controlled SDE Estimation}\label{subsec:exp_controlled}

\paragraph{Python open-source library.}
Our implementation is available as an open-source Python library at \url{lmotte/controlled-sde-learn}. The library includes documentation and example scripts with light computational demands. More details are given in Appendix~\ref{sec:implementation_controlled}, including derivations of the formulas used to implement the estimator.

\paragraph{Experimental setup.}
We consider an isotropic diffusion and use the pointwise positivity constraint
described in Remark~\ref{rem:alternative_ellipticity_constraints}. All kernels are Gaussian kernels \(k(x,y) = \exp(-\gamma \|x-y\|^2)\) with different parameters \(\gamma>0\). For the 1D SDE, we select all hyperparameters by grid search on validation sets. The selection metrics are the log-likelihood for the probability density estimation step, i.e., \(\hat p \mapsto \sum_{l=1}^{M_{val}} \sum_{i=1}^{Q_{val}} \log(\hat p(t_l, x_i))\), and the mean squared error for the Fokker--Planck matching step, i.e., \((\hat b, \hat a) \mapsto \sum_{i=1}^{N} [\frac{\partial \hat p}{\partial t}(t_i, x_i) - \bmL^{(\hat b,\hat a)} \hat p(t_i,x_i)]^2\). For the 2D SDE, to avoid excessive computation, we fix the hyperparameters for both probability density estimation and Fokker--Planck matching using those previously selected for the uncontrolled Dubins process in Section~\ref{subsubsec:dubins_finite_exp}. Further details can be found in the code repository.

% \paragraph{Experimental setup.} We consider uniform diffusion and the soft-shape constraint presented in Section \ref{sec:proposed_method} for the diffusion coefficient. All kernels are Gaussian kernels \(k(x,y) = \exp(-\gamma \|x-y\|^2)\) with different parameters $\gamma>0$. For the 1D SDE, we select all hyperparameters using grid search and validation sets. The selection metrics are the log-likelihood for the probability density estimation step, i.e., \(\hat p \mapsto \sum_{l=1}^{M_{val}} \sum_{i=1}^{Q_{val}} \log(\hat p(t_l, x_i))\),  and the mean squared error for the Fokker-Planck matching step, i.e., \((\hat b, \hat a) \mapsto \sum_{i=1}^{N} [\frac{\partial p}{\partial t}(t_i, x_i) - \bmL^{(\hat b,\hat a)} p(t_i,x_i)]^2\). For the 2D SDE, to avoid excessive computation, we fix the hyperparameters for both the probability density estimation and the Fokker-Planck matching using previously selected hyperparameters for the estimation of uncontrolled Dubins process in Section \ref{subsubsec:dubins_finite_exp}. Further details can be found in the code repository.

\paragraph{Computational considerations.}
While the computational complexity of each step is given in Appendix~\ref{sec:implementation_controlled}, we also report observed execution times to give a practical sense of scale. The experiments were run on a machine equipped with an Apple M3 Pro processor and 18 GB of RAM. For the 1D case, the probability density estimation step takes 5.6 seconds for training with 1000 sample paths, 100 time steps, and 10 training controls, and 1.0 seconds for prediction over a Fokker--Planck training set of size 10000. The Fokker--Planck matching step takes 190 seconds for training with 10000 data points and 150 seconds to generate 100 sample paths with 100 time steps for 10 different controls. For the 2D problem, probability density estimation training takes around 0.032 seconds for 3000 sample paths, 100 time steps, and 20 training controls, with prediction taking approximately 160 seconds over a Fokker--Planck training set of size 10000. The Fokker--Planck matching step takes about 490 seconds for training and 140 seconds to generate 100 sample paths with 100 time steps for 5 different controls. These timings show that controlled SDE estimation is substantially more demanding than the uncontrolled case, but remains feasible at the moderate scales considered here. As expected, the cost increases with the number of training controls.

% \paragraph{Computational considerations.} While the computational complexity of each step is provided in Section \ref{sec:implementation_controlled}, we offer here a practical idea of the computational requirements based on the observed execution times from our experiments on the considered problem and datasets, using a machine equipped with an Apple M3 Pro processor and 18 GB of RAM. For the 1D case, the probability density estimation step takes 5.6 seconds for training with 1000 sample paths, 100 time steps, and 10 training controls, and 1.0 seconds for prediction over a Fokker-Planck training set of size 10000. The Fokker-Planck matching step requires 190 seconds for training with 10000 data points and 150 seconds to generate 100 sample paths with 100 time steps for 10 different controls. For the 2D problem, probability density estimation training takes around 0.032 seconds for 3000 sample paths, 100 time steps, and 20 training controls, with prediction taking approximately 160 seconds over a Fokker-Planck training set of size 10000. The Fokker-Planck matching step requires about 490 seconds for training and 140 seconds to generate 100 sample paths with 100 time steps for 5 different controls. These timings show that controlled SDE estimation is substantially more demanding than
% the uncontrolled case, but remains feasible at the moderate scales considered here. As expected, the computational requirements increase as the sizes of the datasets are multiplied by the number of training controls.

\subsubsection{Linear Scalar SDE} 

We first consider controlled Ornstein--Uhlenbeck processes.

\paragraph{Controlled Ornstein--Uhlenbeck process.}
Let \(X(t)\) be an Ornstein--Uhlenbeck process with controlled mean, defined by the SDE
\begin{align}
    dX(t) = \theta(u(t) - X(t)) dt + \sigma dW(t),
\end{align}
where \(\theta\in \R\), \(\sigma \in \R_+\), and \(u\in\bmH, u:[0,T]\to\mathbb R\). For simplicity, we take \(\bmH\) to be the class of two-step control functions from \([0,T]\) to \(\R\), namely
\begin{align}
    u(t) = 
    \begin{cases} 
    u_0, & \text{if } t \in [0, t_1) \\
    u_1, & \text{if } t \in [t_1, T]
    \end{cases}
    \quad \text{ with } \quad u_0, u_1 \in \R, t_1 \in [0, T].
\end{align}
In the experiments, the control parameter belongs to
\[
\Theta_{\mathrm{OU}}=[-2,2]^2\times[3,7],
\]
with parameter \((u_0,u_1,t_1)\), and the three components are sampled
independently and uniformly over their respective intervals.
These controls fall outside Assumption~\ref{as:smooth_controls}, but they provide a simple numerical illustration.

% \paragraph{Controlled Ornstein–Uhlenbeck process.} Let \(X(t)\) be an Ornstein–Uhlenbeck process with controlled mean, defined by the SDE
% \begin{align}
%     dX(t) = \theta(u(t) - X(t)) dt + \sigma dW(t),
% \end{align}
% where $\theta\in \R, \sigma \in \R_+$, and $u\in\bmH, u:[0,T]\to\mathbb R$. For simplicity, we take \(\bmH\) to be the class of two-step control functions from \([0,T]\) to \(\R\), namely
% \begin{align}
%     u(t) = 
%     \begin{cases} 
%     u_0, & \text{if } t \in [0, t_1) \\
%     u_1, & \text{if } t \in [t_1, T]
%     \end{cases}
%     \quad \text{ with } \quad u_0, u_1 \in \R, t_1 \in [0, T].
% \end{align}
% These controls fall outside Assumption A2, but still provide a simple numerical illustration.

\paragraph{Datasets.}
We consider a controlled OU process and set \(\theta = 0.5, \sigma^2 = \theta/4, \mu_0 = 0.5, \sigma_0^2 = \sigma^2 / (2\theta)\), \(T=10\), and \( X(0) \sim \bmN(\mu_0, \sigma_0^2) \). In Figure~\ref{fig:controlled_ornstein_uhlenbeck_samples}, we plot 100 sample paths generated from this OU process for three different controls. We draw a dataset of \(K=10\) controls, shown in Figure~\ref{fig:521_control_tr_all}, by sampling
\((u_0,u_1,t_1)\) independently according to the distribution defined above. For the probability density estimation steps, we draw training and validation sets with \(Q/Q_{val}=1000/100\) sample paths and \(M/M_{val}=100/100\) time steps, respectively. For the Fokker--Planck matching step, we draw a training set \((t_i, x_i)_{i=1}^N = \{t_i\}_{i=1}^{20} \times \{x_i\}_{i=1}^{50}\) by drawing times and positions uniformly within well-chosen intervals, and we draw a validation set with the same size and distribution.

\begin{figure}[t]
    \centering
    \begin{minipage}[t]{\textwidth}
        \centering
        \includegraphics[height=3.5cm]{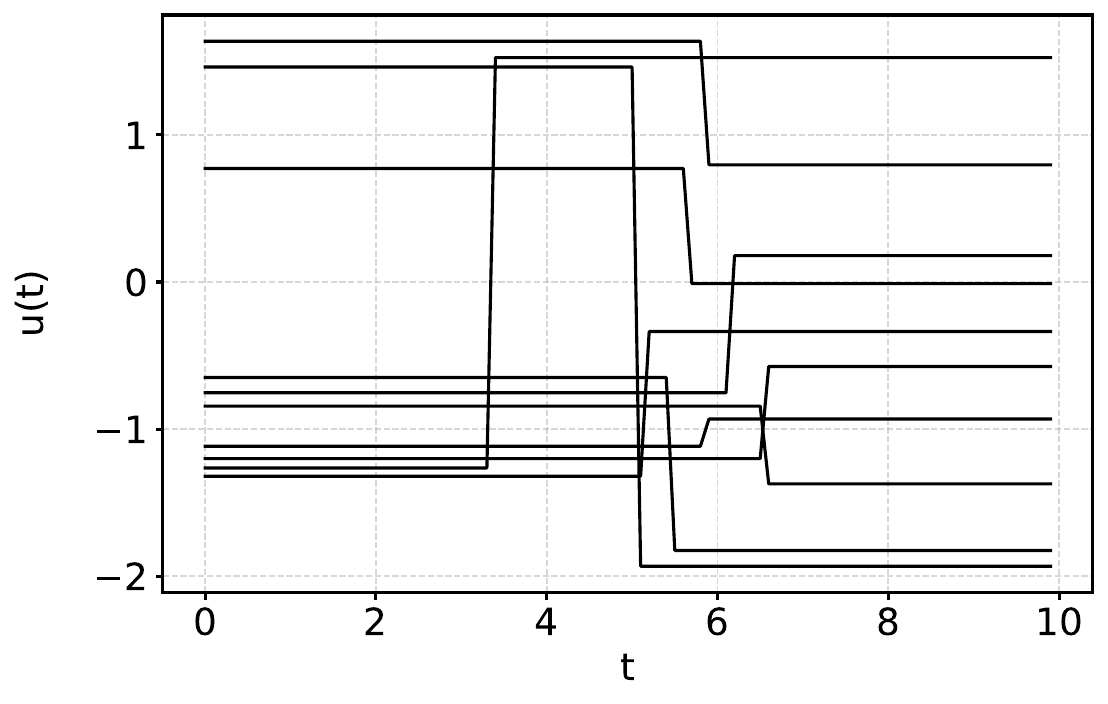} 
    \end{minipage}
    \caption{Training set of 10 i.i.d. piecewise-constant controls.}
    \label{fig:521_control_tr_all}
\end{figure}

\paragraph{Recovering the true controlled dynamics.}
For \(K_{te}=3\) randomly drawn controls, we generate and plot 100 sample paths using both the true and estimated coefficients; see Figure~\ref{fig:controlled_ornstein_uhlenbeck_samples}. The estimated coefficients produce probability distributions that are visually close to the true ones, with matching means and variances over time. The three controls used for evaluation are not part of the training data.

\begin{figure}[t]
    \centering

        \begin{minipage}[t]{0.33\textwidth}
        \centering
        \caption*{\(u_1\)}
        \includegraphics[height=3.2cm]{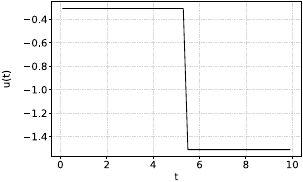} 
    \end{minipage}\hfill
    \begin{minipage}[t]{0.33\textwidth}
        \centering
        \caption*{\(u_2\)}
        \includegraphics[height=3.2cm]{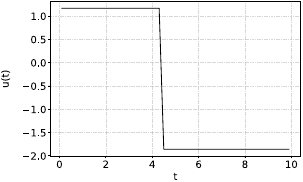}
    \end{minipage}\hfill
    \begin{minipage}[t]{0.33\textwidth}
        \centering
        \caption*{\(u_3\)}
        \includegraphics[height=3.2cm]{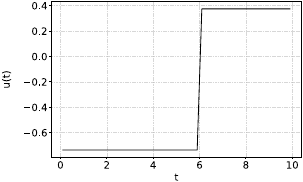} 
    \end{minipage}

    \begin{minipage}[t]{0.33\textwidth}
        \centering
        \includegraphics[height=3.2cm]{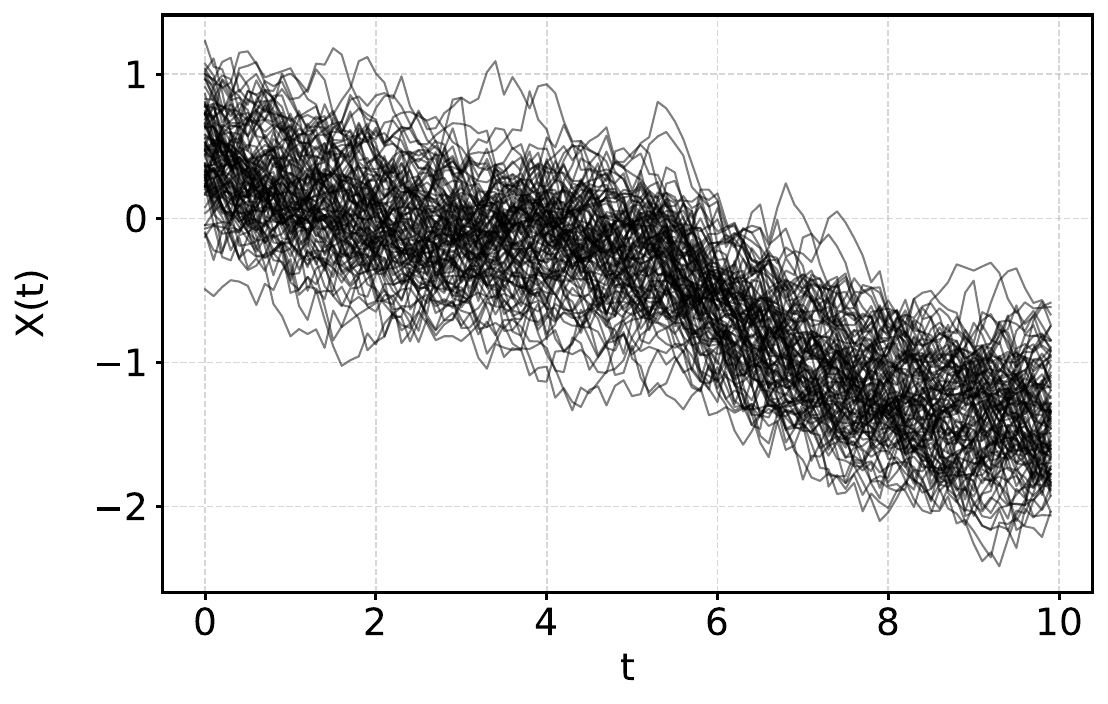} 
    \end{minipage}\hfill
    \begin{minipage}[t]{0.33\textwidth}
        \centering
        \includegraphics[height=3.2cm]{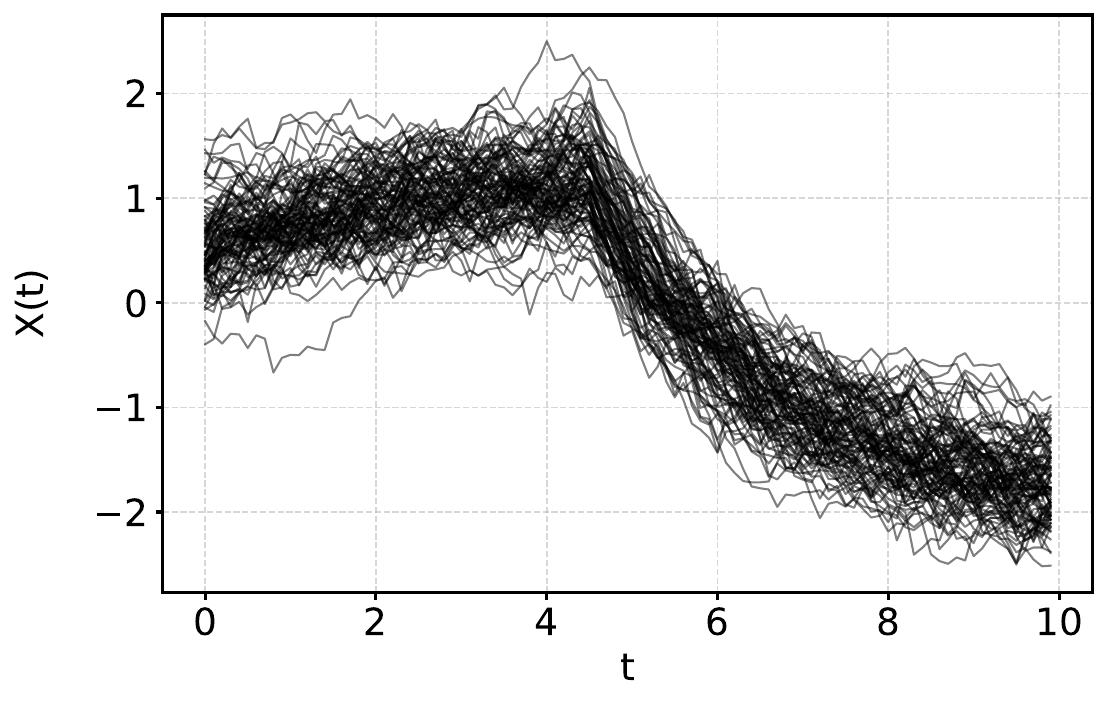}
    \end{minipage}\hfill
    \begin{minipage}[t]{0.33\textwidth}
        \centering
        \includegraphics[height=3.2cm]{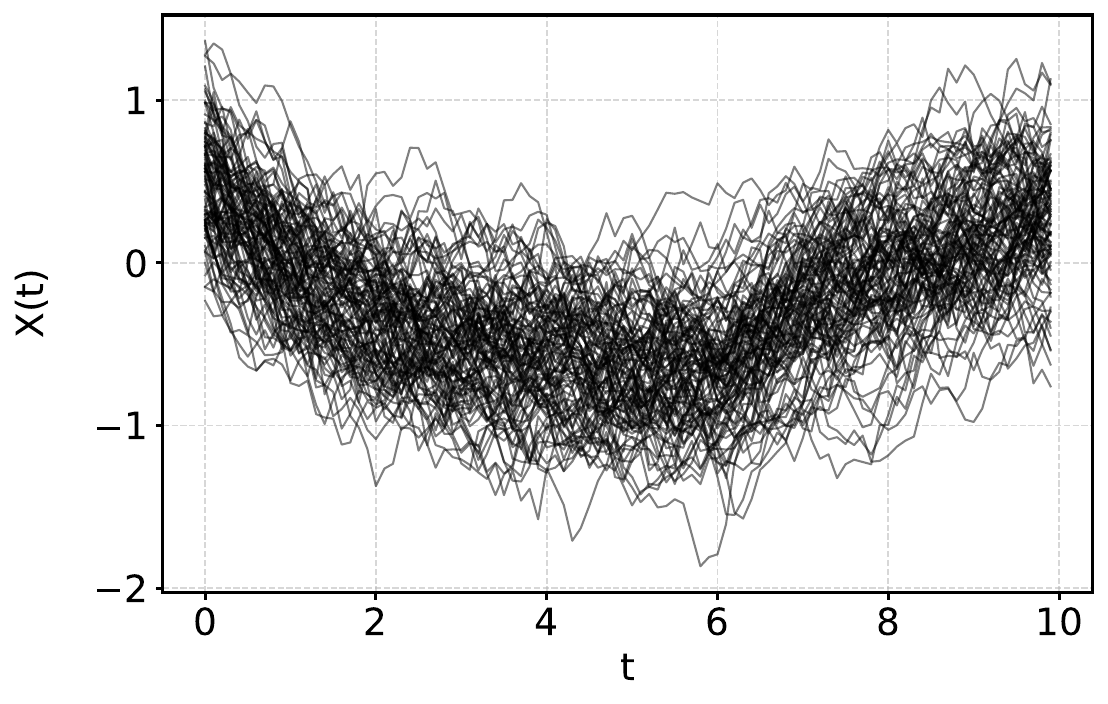} 
    \end{minipage}

     \begin{minipage}[t]{0.33\textwidth}
        \centering
        \includegraphics[height=3.2cm]{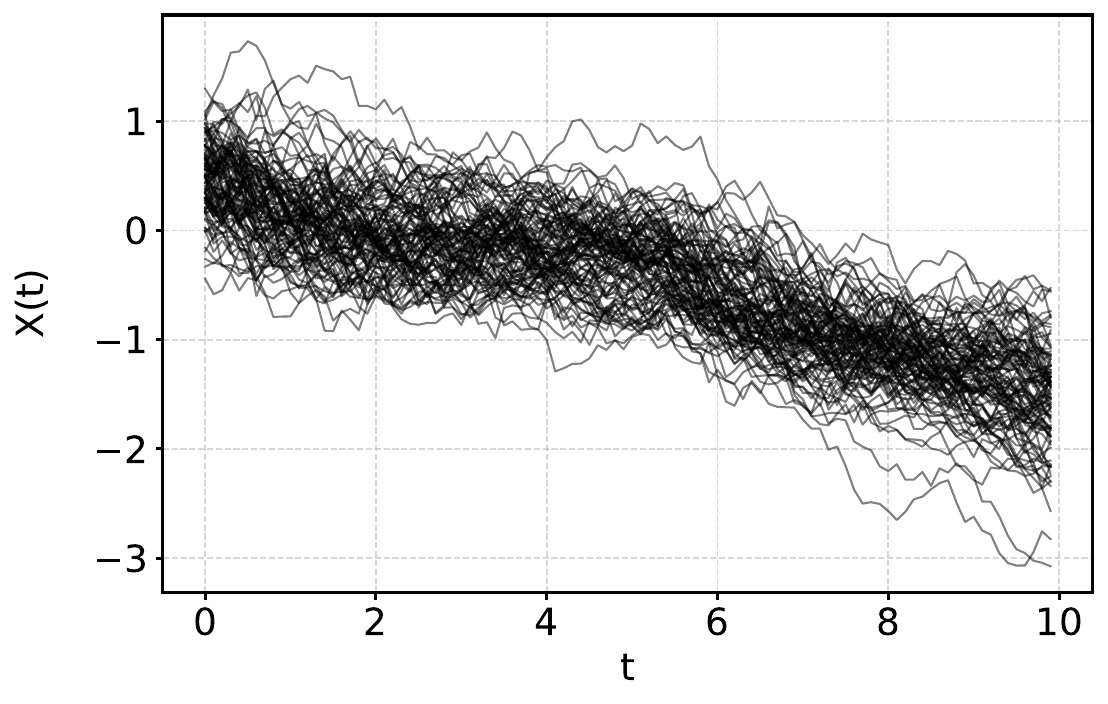} 
    \end{minipage}\hfill
    \begin{minipage}[t]{0.33\textwidth}
        \centering
        \includegraphics[height=3.2cm]{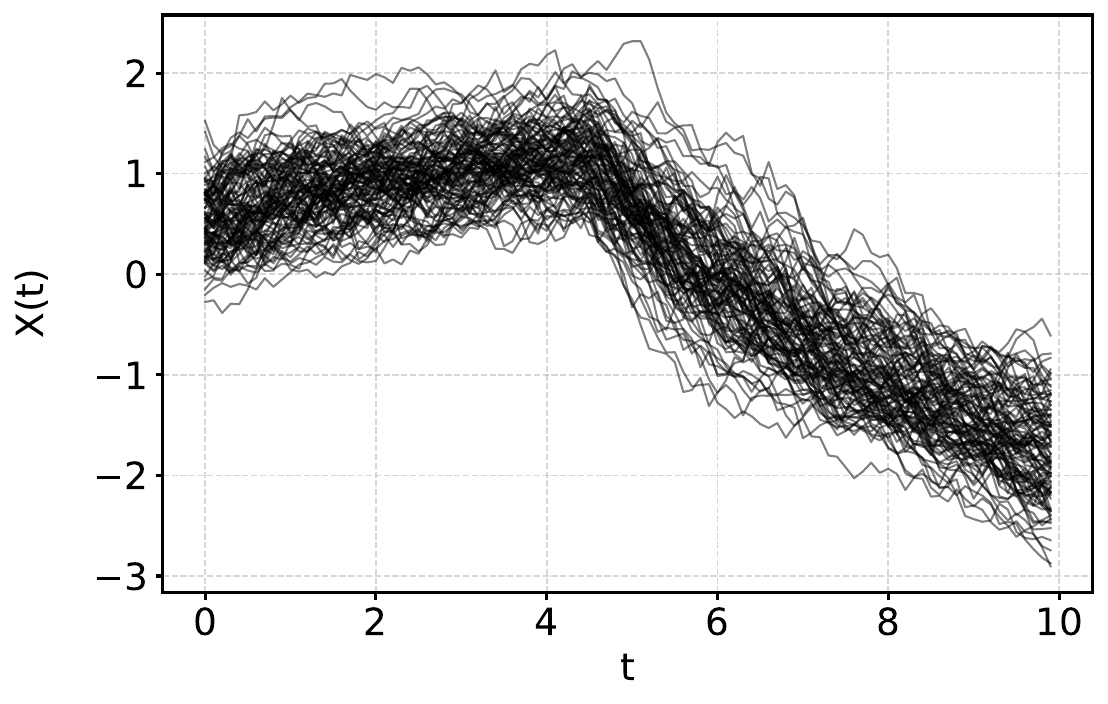}
    \end{minipage}\hfill
    \begin{minipage}[t]{0.33\textwidth}
        \centering
        \includegraphics[height=3.2cm]{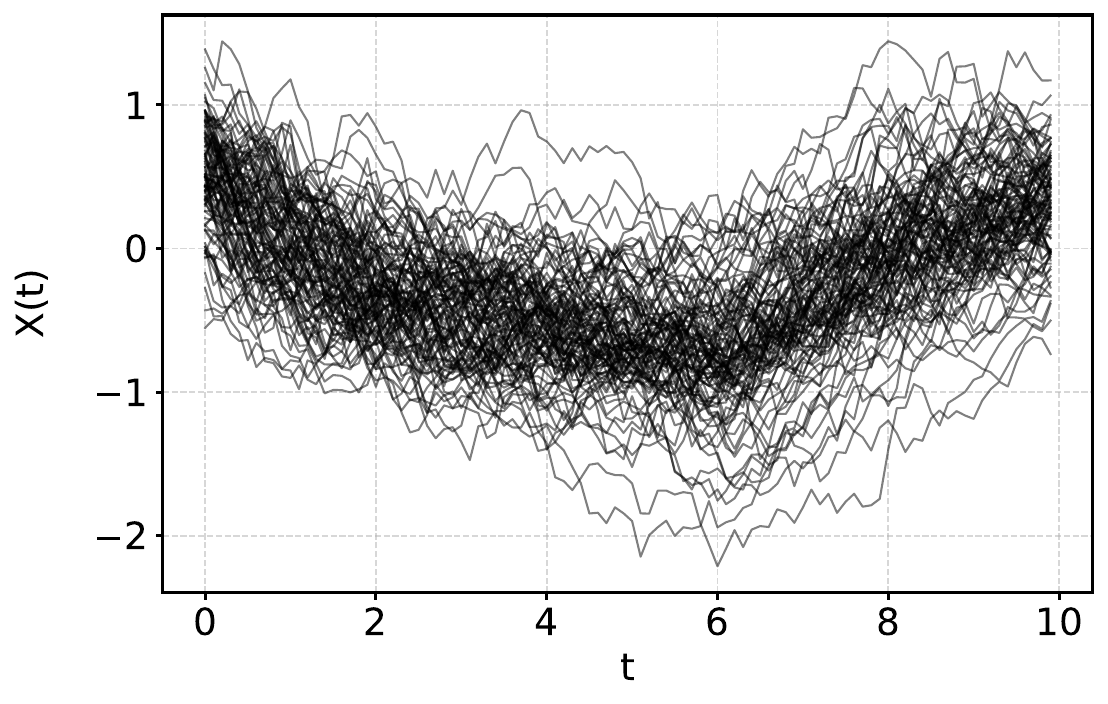} 
    \end{minipage}

    \caption{Controlled Ornstein–Uhlenbeck process. 100 sample paths from the SDEs with true \textit{(center)} and estimated \textit{(bottom)}  coefficients for three random controls \(u_1, u_2, u_3\) \textit{(top)}.}
    \label{fig:controlled_ornstein_uhlenbeck_samples}
\end{figure}

\subsubsection{Nonlinear Multivariate SDE}

We next consider controlled Dubins processes.

\paragraph{Controlled Dubins process.}
Let \(X(t)\) be a Dubins process with controlled angle, defined by the SDE
\begin{align}\label{eq:controlled_mult_sde}
    dX(t) = v \times (\cos(u(t)), \sin(u(t))) dt + \sigma dW(t) \in \R^2,
\end{align}
where \(v \in \R\), \(\sigma > 0\), and \(u\in\bmH\), where \(u:[0,T]\to\R\). We consider the parametrized control family
\[
\Theta=[-1.2,1.2],
\qquad
u_\theta(t)=\theta\sin(\pi t/10),
\qquad
\mathcal H=\{u_\theta:\theta\in\Theta\}.
\]

% We define sinusoidal controls by setting
% \begin{align}
%     \bmH = \{u : t \mapsto a \sin(t\pi/ 10) \;|\; a \in \R \}.
% \end{align}

\begin{figure}[H]
    \centering
    \begin{minipage}[t]{\textwidth}
        \centering
        \includegraphics[height=3.5cm]{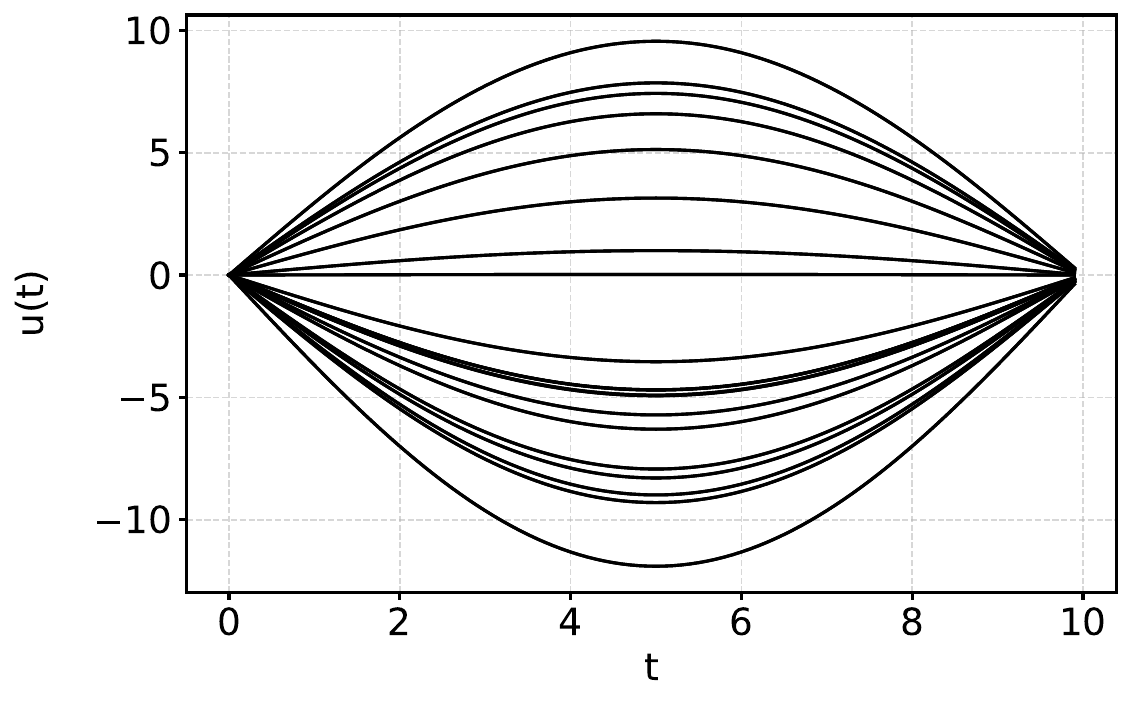} 
    \end{minipage}
\caption{Training set of 20 i.i.d. controls.}    \label{fig:522_controls_tr_all_2d_2}
\end{figure}

\paragraph{Datasets.}
We consider a controlled Dubins process and set \(v = 2\), \(\sigma = 0.3\), \(\mu_0 = 0\), \(\sigma_0 = 0.5\), \(T = 10\), and \( X(0) \sim \bmN(\mu_0, \sigma_0^2 I_{\R^2}) \). In Figure~\ref{fig:controlled_dubins_samples}, we plot 100 sample paths generated from this process for five different controls. We build a dataset of \(K = 20\) controls, shown in Figure~\ref{fig:522_controls_tr_all_2d_2} (with amplitudes scaled by a factor of \(10\)), by drawing the amplitudes independently according to \(\theta_k\sim\operatorname{Unif}(\Theta)\). For the probability density estimation steps, we generate a training set with \(Q = 3000\) sample paths and \(M = 100\) time steps. For the Fokker--Planck matching step, we draw a training set \((\theta_k, t_i, x_i)_{k \in \llbracket 1, 20 \rrbracket, i \in \llbracket 1, 500 \rrbracket }\) of size \(10^4\). To avoid sampling regions where \(p(\theta_k, t, x)\) is negligible, for each \(\theta_k\), we draw 500 pairs \((t_i, x_i)_{i=1}^{500}\) from 5 sample paths with 100 time steps, generated with the same parameters as the controlled Dubins process but with initial variance \(\sigma_0 = 2.5\).

\begin{figure}[H]
    \centering
    \hspace{0.05em}
    \begin{minipage}[t]{0.19\textwidth}
        \centering
        \caption*{\hspace{1.5em}\(u_1\)}
        \includegraphics[height=1.8cm]{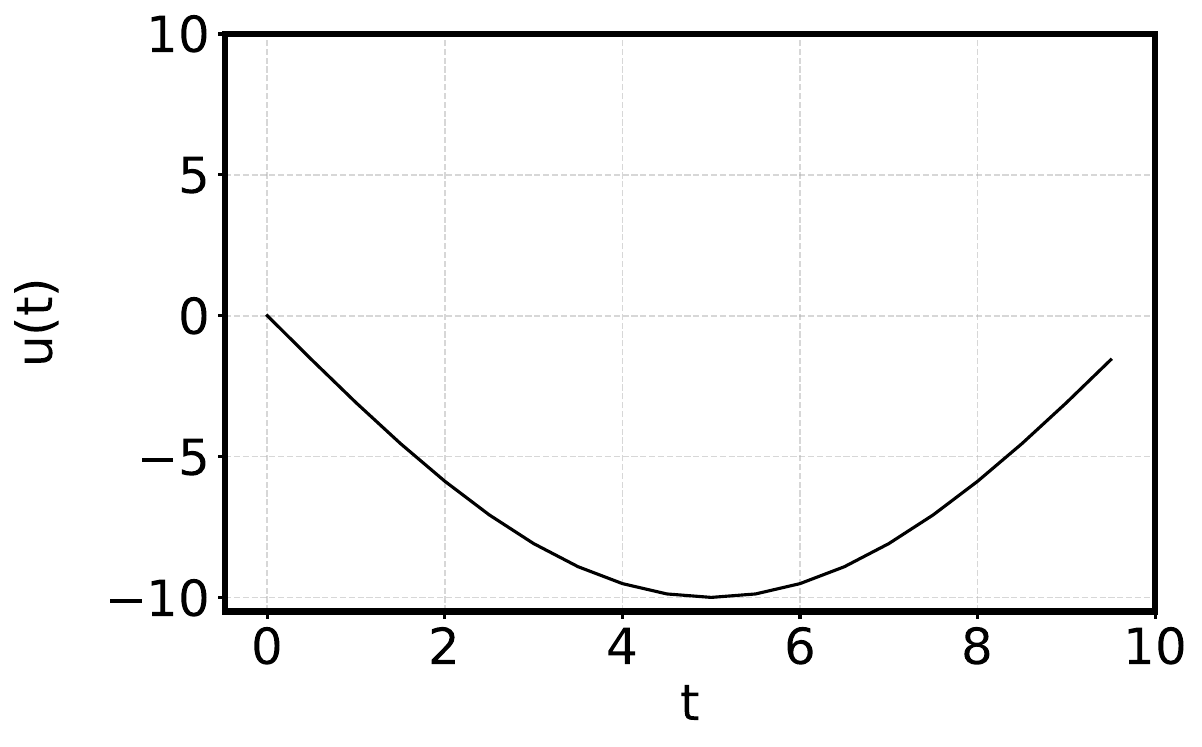} 
    \end{minipage}\hfill
    \begin{minipage}[t]{0.2\textwidth}
        \centering
        \caption*{\hspace{1.5em}\(u_2\)}
        \includegraphics[height=1.8cm]{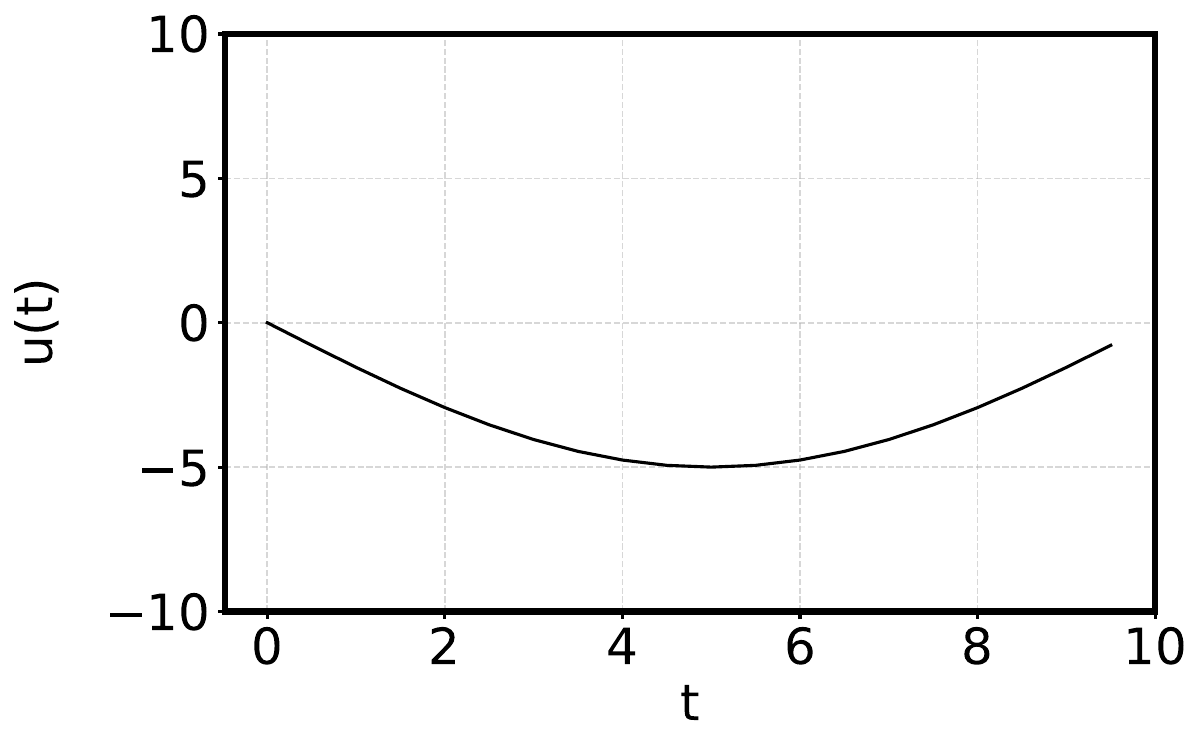} 
    \end{minipage}\hfill
    \begin{minipage}[t]{0.2\textwidth}
        \centering
        \caption*{\hspace{1.5em}\(u_3\)}
        \includegraphics[height=1.8cm]{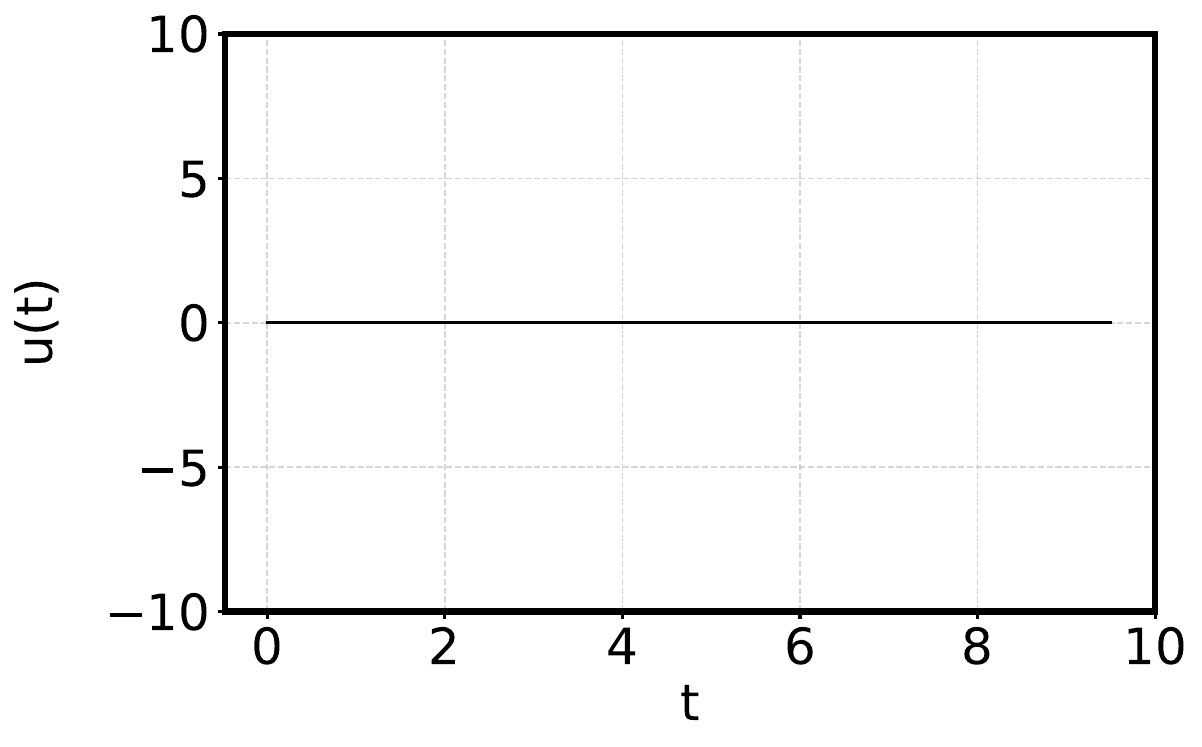} 
    \end{minipage}\hfill
    \begin{minipage}[t]{0.19\textwidth}
        \centering
        \caption*{\hspace{1.5em}\(u_4\)}
        \includegraphics[height=1.8cm]{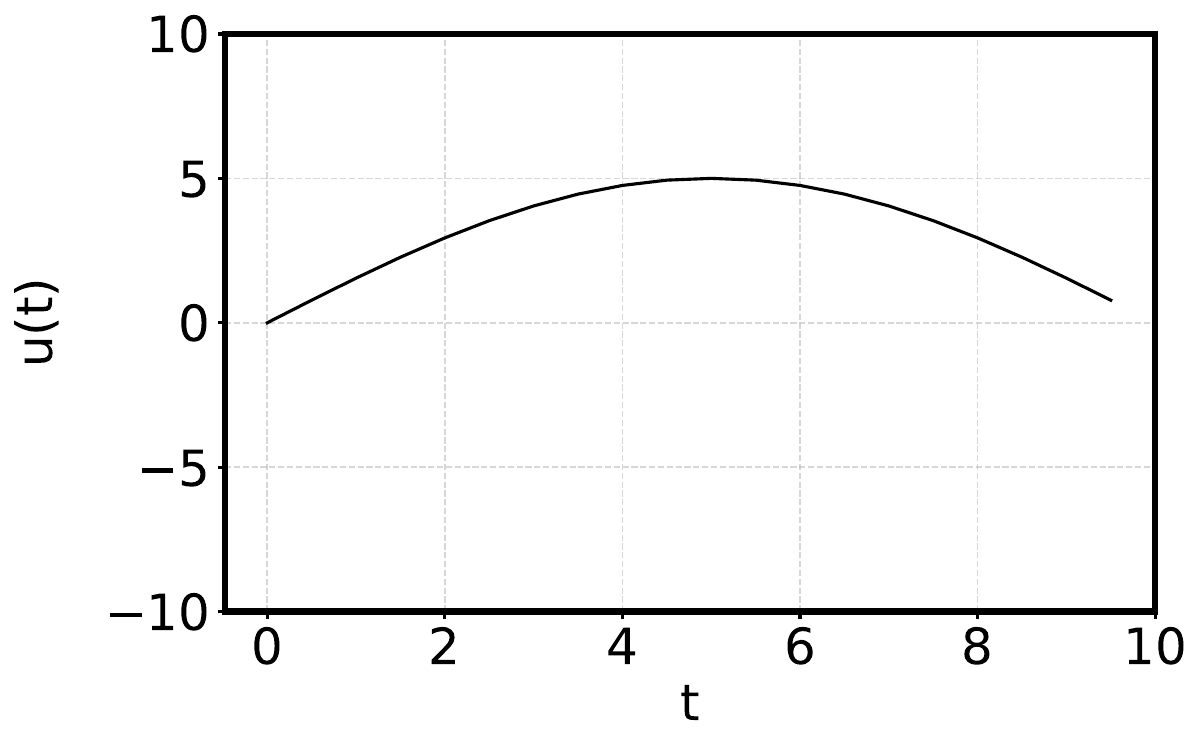} 
    \end{minipage}\hfill
    \begin{minipage}[t]{0.19\textwidth}
        \centering
        \caption*{\hspace{1.5em}\(u_5\)}
        \includegraphics[height=1.8cm]{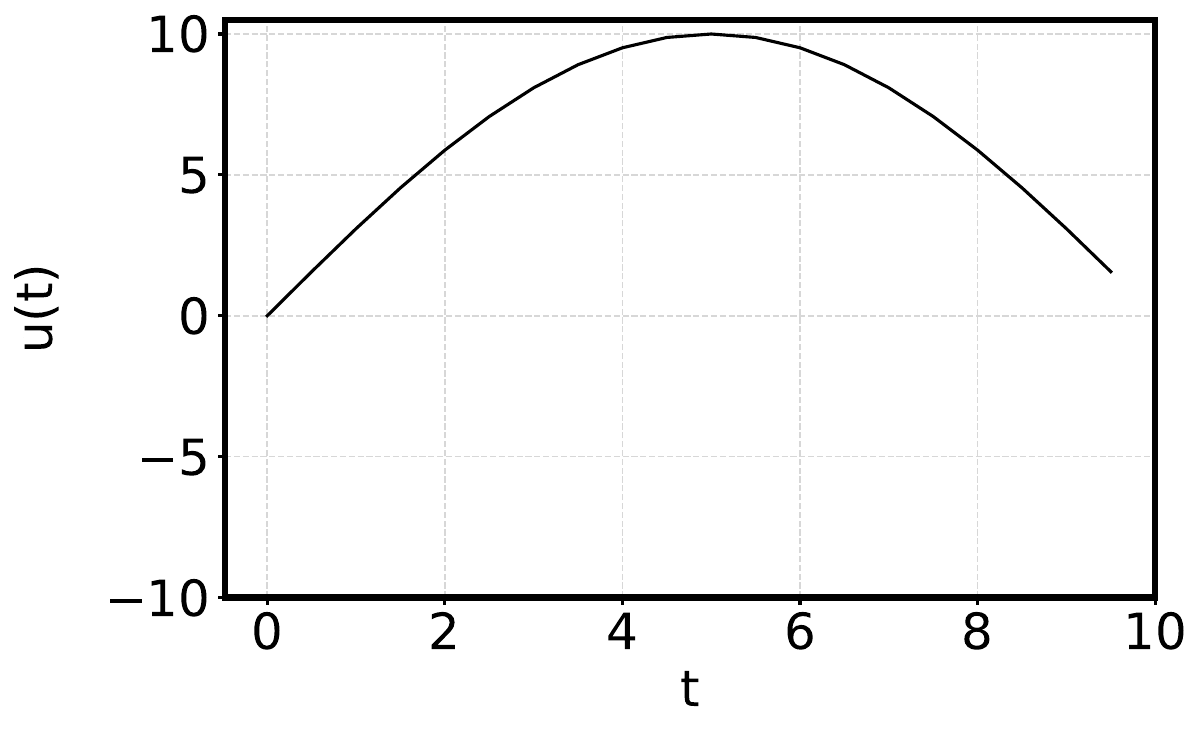} 
    \end{minipage}

    \begin{minipage}[t]{0.2\textwidth}
        \centering
        \includegraphics[height=2.8cm]{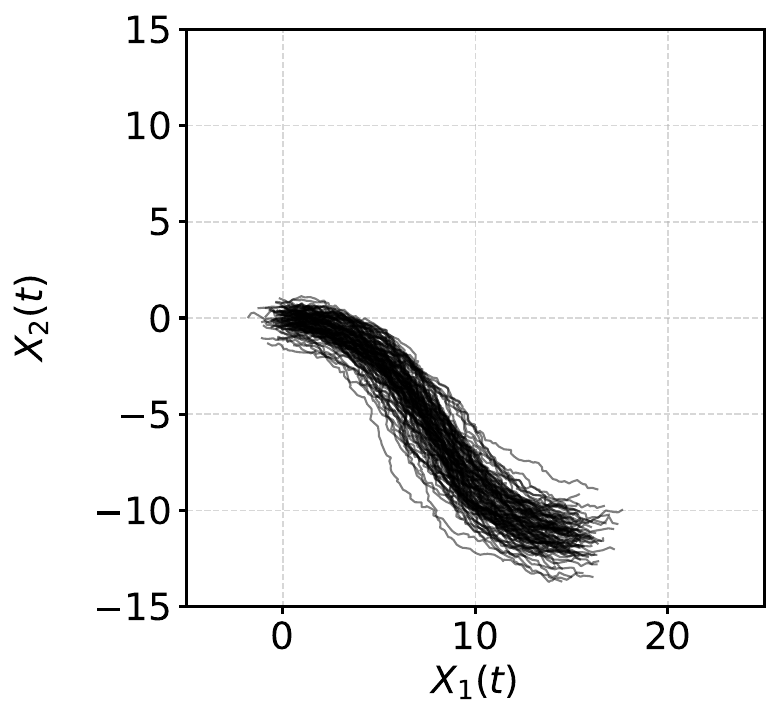} 
    \end{minipage}\hfill
    \begin{minipage}[t]{0.2\textwidth}
        \centering
        \includegraphics[height=2.8cm]{figures/522_samples_true_2d_2_0.pdf} 
    \end{minipage}\hfill
    \begin{minipage}[t]{0.2\textwidth}
        \centering
        \includegraphics[height=2.8cm]{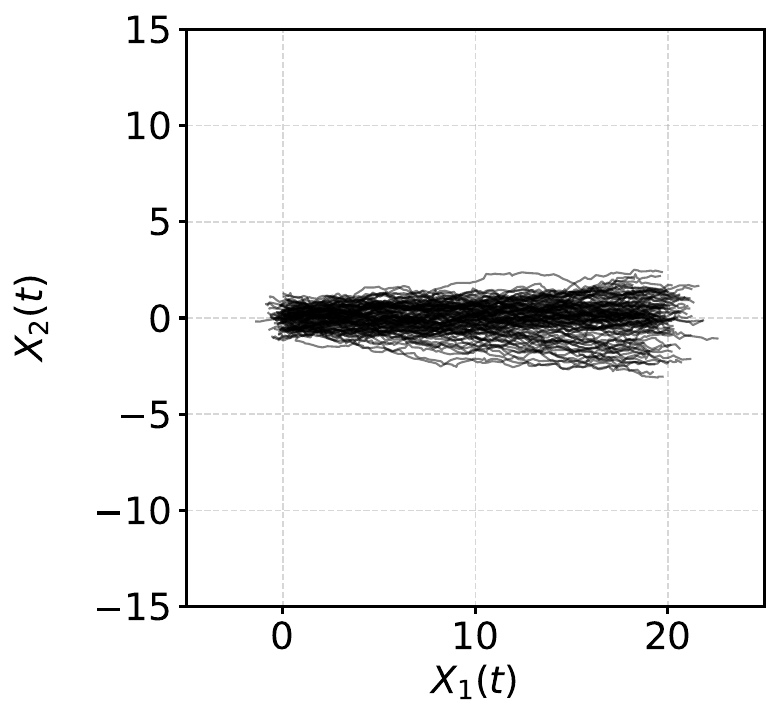} 
    \end{minipage}\hfill
    \begin{minipage}[t]{0.2\textwidth}
        \centering
        \includegraphics[height=2.8cm]{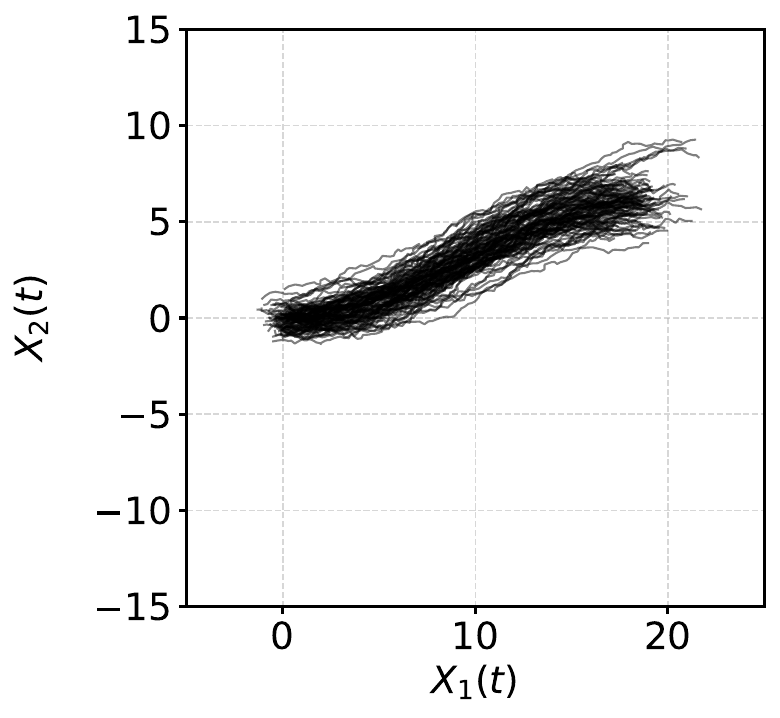} 
    \end{minipage}\hfill
    \begin{minipage}[t]{0.2\textwidth}
        \centering
        \includegraphics[height=2.8cm]{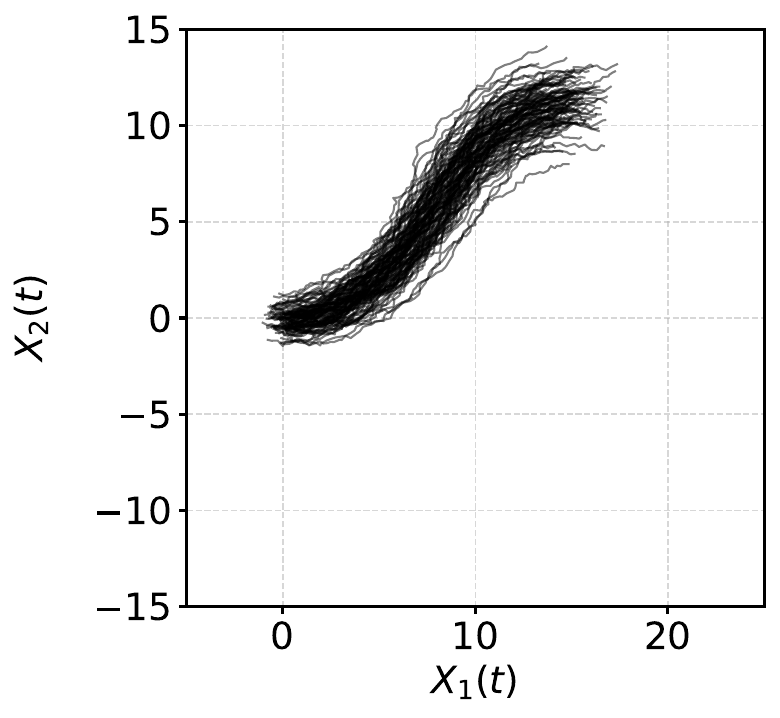} 
    \end{minipage}

    \begin{minipage}[t]{0.2\textwidth}
        \centering
        \includegraphics[height=2.8cm]{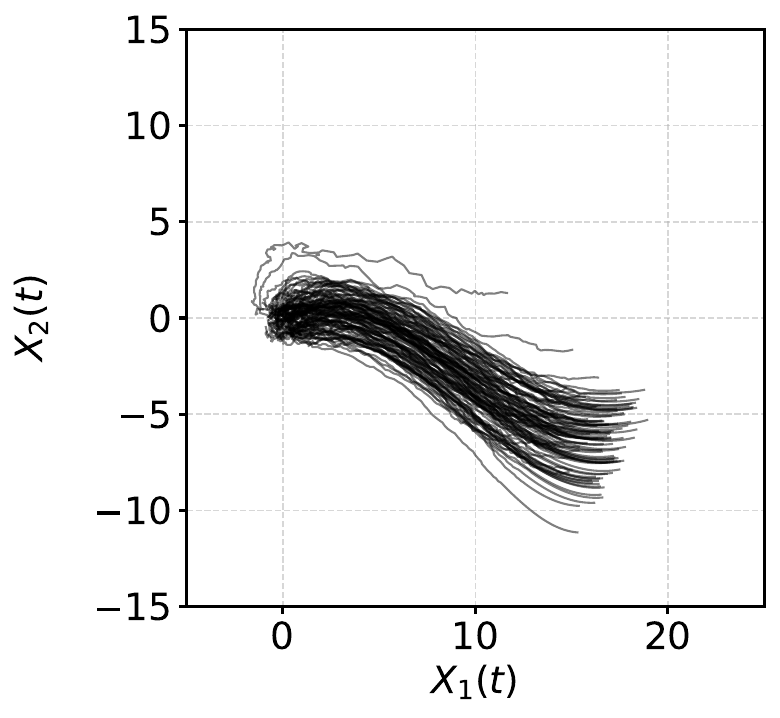} 
    \end{minipage}\hfill
    \begin{minipage}[t]{0.2\textwidth}
        \centering
        \includegraphics[height=2.8cm]{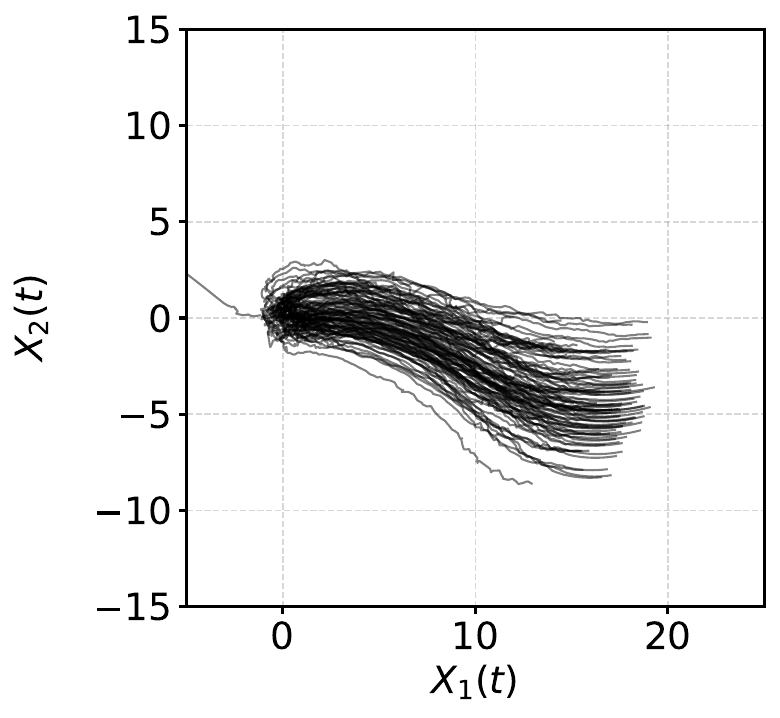} 
    \end{minipage}\hfill
    \begin{minipage}[t]{0.2\textwidth}
        \centering
        \includegraphics[height=2.8cm]{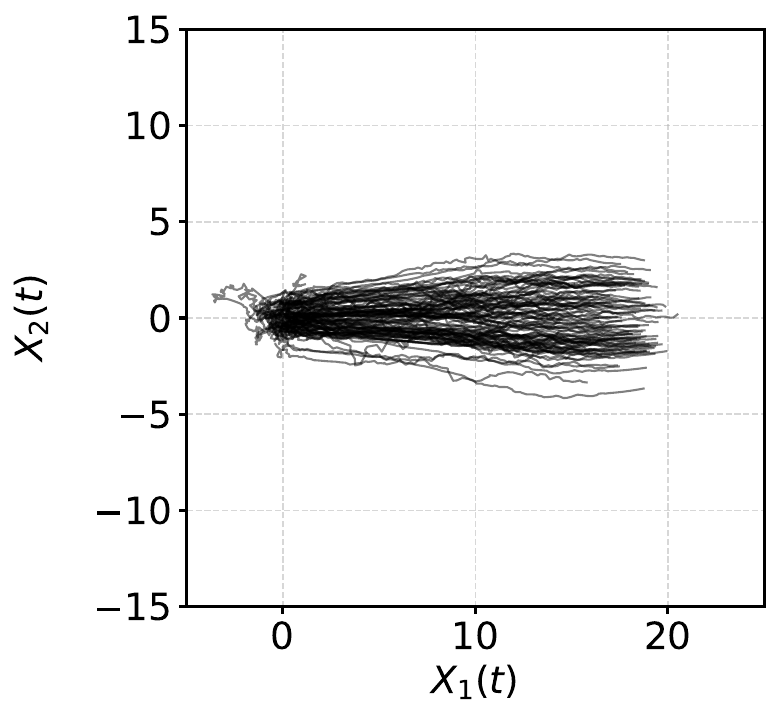} 
    \end{minipage}\hfill
    \begin{minipage}[t]{0.2\textwidth}
        \centering
        \includegraphics[height=2.8cm]{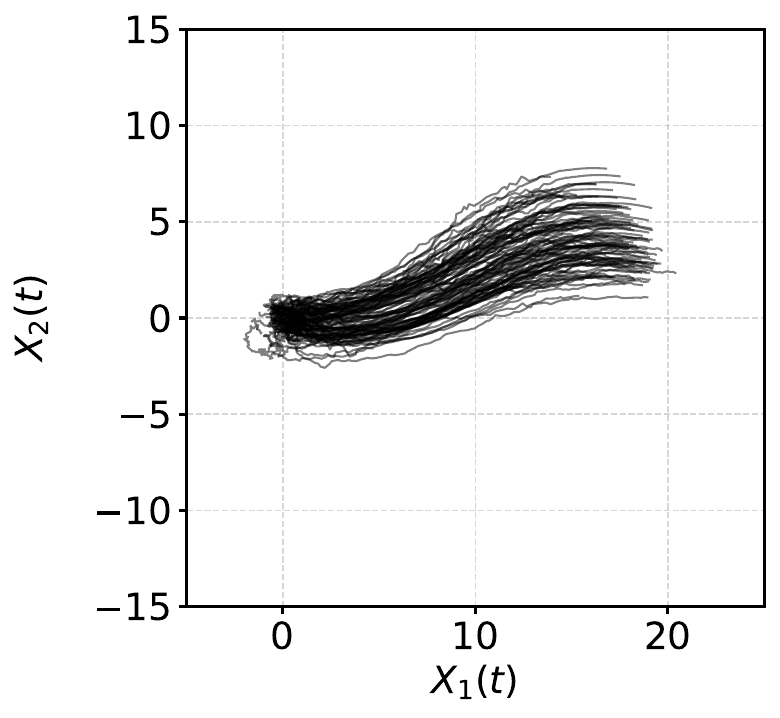} 
    \end{minipage}\hfill
    \begin{minipage}[t]{0.2\textwidth}
        \centering
        \includegraphics[height=2.8cm]{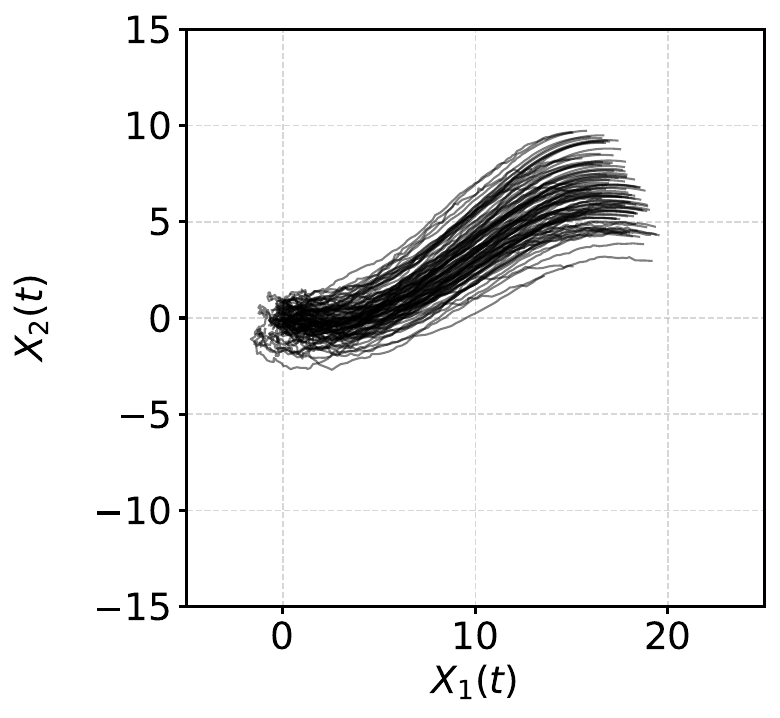} 
    \end{minipage}
    \caption{Controlled Dubins process. 100 sample paths from the SDEs with true \textit{(center)} and estimated \textit{(bottom)} coefficients for five controls \textit{(top)} with \(\theta = -1, -1/2, 0, 1/2, 1\) \textit{(from left to right)}.}
    \label{fig:controlled_dubins_samples}
\end{figure}

\paragraph{Recovering the true controlled dynamics.}
For evaluation, we consider \(K_{\mathrm{te}}=5\) held-out controls
with amplitudes \(\theta=-1,-1/2,0,1/2,1\). These amplitudes lie in the
interior of the training range \([-1.2,1.2]\). We generate and plot
100 sample paths using both the true and estimated coefficients; see
Figure~\ref{fig:controlled_dubins_samples}. The estimated coefficients
yield probability distributions that are visually close to the true
ones, with similar means and variances over time.

\section{Conclusion}

We studied the problem of learning multidimensional controlled SDEs whose drift and diffusion may depend nonlinearly on time,
state, and control, from trajectory data collected under several controls. The
proposed method follows a two-step strategy: it first estimates the state-density evolution
associated with each observed control, and then learns drift and diffusion coefficients by
least-squares Fokker--Planck matching.

The main contribution is a finite-sample analysis of the density flows induced by the
learned coefficients. The resulting bounds show how the error decreases with the number
of sampled controls, and how this rate depends on the dimension and regularity of the
control parametrization. Beyond average-case accuracy, we also obtain guarantees that hold
uniformly over controls, together with CVaR-type bounds for tail-sensitive
quantities under the learned dynamics.

Several limitations should be kept in mind. The results are stated at the level of
density-flow reconstruction, which is the natural object controlled by the Fokker--Planck
equation, but does not in general imply identification of a unique drift--diffusion pair:
distinct coefficients may induce the same laws. The setting is also passive and open-loop:
the controls are observed inputs sampled from a fixed finite-dimensional family, rather
than actions chosen adaptively by the learner. Finally, the analysis relies on smoothness,
ellipticity, localization, and first-stage density-estimation assumptions, which make the
finite-sample guarantees explicit but also delimit the regime covered by the theory.

These limitations point to several directions for future work. A complete end-to-end
analysis should track the density-estimation step more explicitly, including the number of
trajectories per control and the time discretization. Beyond the Sobolev instantiation
considered here, it would be interesting to verify the general source and embedding
conditions for other kernels and structural assumptions, such as finite-dimensional,
sparse, or low-rank representations, which may yield different dimension dependence and
improved statistical guarantees. For larger problems, the kernel least-squares structure of
the Fokker--Planck matching step makes Nystr\"om methods, random features, sketching,
stochastic optimization, and iterative solvers natural scaling tools. Further extensions
include feedback controls, adaptive or safe exploration of the control space, weaker
localization assumptions, and tests on real-world systems such as autonomous driving or
space rendezvous.

% Acknowledgements and Disclosure of Funding should go at the end, before appendices and references
\acks{This work was carried out while L.B.M. was a postdoctoral researcher at Laboratoire des Signaux et Systèmes, Université Paris-Saclay, CentraleSupélec. The Agence Nationale de la Recherche (grant ANR-22-CE48-0006, PI: R.B.) provided funds to support this research. A.R. acknowledges support from the European Research Council (grant REAL 947908).}

\newpage
\appendix
\section{Proofs}

\subsection{Probabilistic Setting}\label{app:prob_setting}

We fix integers \(K,Q,M \in \mathbb N_+\) and a time horizon \(T>0\).

\paragraph{Control parameters.}
The controls are parametrized by
\[
\bmH = \{u_\theta : \theta \in \Theta\},
\]
where \(\Theta \subset \mathbb R^m\) is a is a bounded Lipschitz domaine, equipped with its Borel \(\sigma\)-algebra
\(\mathscr T\). Let \(\mathbb P_c\) be a probability measure on
\((\Theta,\mathscr T)\). The parameters \((\theta_k)_{k=1}^K\) are sampled
independently from \(\mathbb P_c\).

\paragraph{Base probability space.}
Let \((\Omega_0,\mathcal F_0,\mathbb P_0)\) be a probability space equipped with
a filtration \((\mathcal F_t)_{t \in [0,T]}\). Let \(W\) be a standard Brownian
motion adapted to this filtration, and let \(X_0\) be an
\(\mathcal F_0\)-measurable random variable with law \(p_0\). The measure
\(\mathbb P_0\) is fixed and does not depend on the control parameter.

\paragraph{Controlled stochastic processes.}
For \(\theta \in \Theta\), let
\[
X^{\theta} : \Omega_0 \times [0,T] \to \mathbb R^n
\]
denote the strong solution of the controlled SDE~\eqref{eq:sde} associated with
\(u_\theta\), whenever it exists. For \(\omega \in \Omega_0\) and \(t \in [0,T]\),
\(X^{\theta}(\omega,t)\) denotes the state at time \(t\).

\paragraph{Auxiliary probability space for control parameters.}
Let
\[
(\Omega_c,\mathcal F_c,\mathbb P_c^{\otimes K})
\]
be a probability space supporting the independent parameters
\((\theta_k)_{k=1}^K\).

\paragraph{Joint probability space.}
We define the product probability space
\[
(\Omega,\mathcal F,\mathbb P)
\;\triangleq\;
(\Omega_c \times \Omega_0^{QK},\;
\mathcal F_c \otimes \mathcal F_0^{\otimes QK},\;
\mathbb P_c^{\otimes K} \otimes \mathbb P_0^{\otimes QK}).
\]

\paragraph{Trajectory indices.}
For each \((k,i) \in \llbracket 1,K \rrbracket \times
\llbracket 1,Q \rrbracket\), let
\[
\omega_i^k : \Omega \to \Omega_0
\]
denote the coordinate projection onto the \((k,i)\)-th copy of
\(\Omega_0\). For a fixed parameter \(\theta_k\), the process
\(X^{\theta_k}(\omega_i^k,\cdot)\) is the solution of the controlled SDE driven
by the Brownian motion and initial condition corresponding to \(\omega_i^k\).

\paragraph{Observed data.}
Let \((t_\ell)_{\ell=1}^M \subset [0,T]\) be fixed sampling times. The observed
dataset is
\[
(\theta_k, X^{\theta_k}(\omega_i^k,t_\ell))_{k \in \llbracket 1,K \rrbracket,\;
i \in \llbracket 1,Q \rrbracket,\;
\ell \in \llbracket 1,M \rrbracket}.
\]

\paragraph{Independence structure.}
The parameters \((\theta_k)_{k=1}^K\) are independent. Conditionally on these
parameters, the trajectories
\((X^{\theta_k}(\omega_i^k,\cdot))_{i=1}^Q\) are independent for each fixed
\(k\), and independent across different values of \(k\).

\subsection{Notation}\label{app:notation}

We collect the notation used throughout the paper. Let \(q,r,s,d,n \in \mathbb N_+\).

\paragraph{Function spaces.}
For sets \(A,B\), \(\mathcal F(A,B)\) denotes the set of all functions from \(A\) to \(B\). If \((A,\mathcal A)\) and \((B,\mathcal B)\) are measurable spaces, \(\mathcal M(A,B)\) denotes the set of measurable functions from \((A,\mathcal A)\) to \((B,\mathcal B)\). For \(q \in \mathbb N_+\), \(H^q(A;B)\) denotes the Sobolev space
\(W^{q,2}(A;B)\) of functions from \(A\) to \(B\) with weak derivatives up to order \(q\). Unless otherwise specified, \(A\) is either \(\mathbb R^r\) or \([0,T]\times\mathbb R^r\), and \(B=\mathbb R^s\).

\paragraph{Control parametrization.}
When \(\bmH=\{u_\theta:\theta\in\Theta\}\) is parametrized by \(\Theta\), we sometimes write \(f(\theta,\cdot)\) for \(f(u_\theta,\cdot)\), whenever no ambiguity arises.

\paragraph{Derivatives.}
For a function \(f : \mathbb R^n \to \mathbb R\) and indices \(i,j \in \llbracket 1,n \rrbracket\), we write
\[
f_i \triangleq \frac{\partial f}{\partial x_i},
\qquad
f_{ij} \triangleq \frac{\partial^2 f}{\partial x_i \partial x_j}.
\]
When convenient, we also use
\[
f_{x_i} \triangleq \frac{\partial f}{\partial x_i},
\qquad
f_{x_i x_j} \triangleq \frac{\partial^2 f}{\partial x_i \partial x_j}.
\]

\paragraph{Uniform bounds.}
For sets \(\bmX,\bmY\), a norm \(\|\cdot\|_{\bmY}\) on \(\bmY\), and a bounded map \(f : \bmX \to \bmY\), we define
\[
\kappa_f \triangleq \sup_{x \in \bmX} \|f(x)\|_{\bmY}.
\]

\paragraph{Mixed \(L^p\) norms.}
For a measurable function
\(f:\Theta\times[0,T]\times\mathbb R^n\to\mathbb R\), we define
\begin{align}
    \|f\|^2_{L^2(\mathbb P_c)\,L^2([0,T]\times\mathbb R^n)}
    &\triangleq
    \mathbb E_{\theta\sim\mathbb P_c}
    \int_0^T\int_{\mathbb R^n}
    f(\theta,t,x)^2\,\mathrm dx\,\mathrm dt, \\
    \|f\|^2_{L^\infty(\mathbb P_c)\,L^2([0,T]\times\mathbb R^n)}
    &\triangleq
    \operatorname*{ess\,sup}_{\theta\sim\mathbb P_c}
    \int_0^T\int_{\mathbb R^n}
    f(\theta,t,x)^2\,\mathrm dx\,\mathrm dt .
\end{align}

\paragraph{Matrix and tensor notation.}
For bounded linear operators \(A,B\), we write \(A\preceq B\) if \(B-A\) is positive semidefinite. For vectors or Hilbert-space elements \(u,v\), \(u\otimes v\) denotes the tensor product. We write \(a\wedge b=\min(a,b)\) and \(a\vee b=\max(a,b)\).

\paragraph{Diffusion matrix.}
Throughout the paper, we write
\[
a \triangleq \sigma \sigma^\top
\]
for the diffusion matrix.

\subsection{Organization of the Proofs}

The proofs are organized as follows.
\begin{enumerate}
    \item \textbf{Proof of the Fokker-Planck matching inequality.} The existence of densities, strong Fokker-Planck equation, and FP matching inequality are proven in Section \ref{subsec:fp_proof}.
    \item \textbf{Proof of \(L^2\) learning rates.} We present necessary preliminary results in Section \ref{subsec:preli}. The \(L^2\) learning rates are then established in Section \ref{subsec:proof_theorem}, based on four main lemmas detailed in Sections \ref{subsec:cone}, \ref{subsec:error_1}, \ref{subsec:error_2}, and \ref{subsec:error_3}, respectively. These main lemmas are proven using auxiliary lemmas, which are established in Section \ref{subsec:lemmas}. The auxiliary lemmas rely on concentration inequalities adapted to our needs, stated in Section \ref{subsec:concentration}.
    \item \textbf{Proof of refined \(L^2\) learning rates.} Refined \(L^2\) learning rates are derived in Section \ref{subsec:proof_refined_L2} using a similar approach to the unrefined \(L^2\) rates, but employing refined lemmas, which are proven in Section \ref{subsec:proof_ref_lem}.
    \item \textbf{Proof of \(L^\infty\) learning rates.} The proofs for the \(L^\infty\) learning rates are provided in Section \ref{subsec:proof_linfty}.
    \item \textbf{Proof of CVaR learning rates.} The proofs for deriving CVaR learning rates are detailed in Section \ref{subsec:proof_cvar}.
    \item \textbf{Proofs for Sobolev coefficients.} The embedding property of Fokker-Planck matching, when the coefficients belong to a Sobolev space, is derived in Section \ref{proof_emb_sobolev}.
\end{enumerate}

\subsection{Existence of Densities, Strong Fokker-Planck Equation, and FP Matching Inequality}
\label{subsec:fp_proof}

We prove the FP matching inequality.

\begin{lemma}[Detailed version of Lemma~\ref{lem:FP_ineq_main}]
\label{lem:FP_ineq}
Let
\(
s_0 \triangleq 5+n+\left\lceil \frac{d+1}{2}\right\rceil .
\)
Assume that Assumptions~\ref{as:uniform_ellipticity}-\ref{as:smooth_coeffs} hold for both coefficient pairs
\((b,a)\) and \((\hat b,\hat a)\), with the same localization domain \(D\),
the same cutoff \(\xi\), and the same ellipticity parameter \(\kappa>0\). In
particular, assume that
\[
b=\xi b_0,\qquad a=\kappa I_n+\xi a_0,\qquad
\hat b=\xi \hat b_0,\qquad \hat a=\kappa I_n+\xi \hat a_0 .
\]
Then, for every \(\theta\in\Theta\), the controlled SDEs associated with
\((b,a)\) and \((\hat b,\hat a)\) admit unique strong solutions whose laws have
densities
\[
p(\theta,t,\cdot)\triangleq p_{b,a}(\theta,t,\cdot),
\qquad
p_{\hat b,\hat a}(\theta,t,\cdot).
\]
Moreover, these densities satisfy, for a.e. \((t,x)\in[0,T]\times\mathbb R^n\),
\[
\partial_t p(\theta,t,x)=\bmL^{(b,a), \theta}p(\theta,t,x),
\qquad
\partial_t p_{\hat b,\hat a}(\theta,t,x)
=
\bmL^{(\hat b,\hat a), \theta}
p_{\hat b,\hat a}(\theta,t,x).
\]

For every \(R>0\), there exists \(C_{\mathrm{FP}}(R)>0\) such that, if
\[
\|\hat b_0\|_{H^{s_0}([0,T]\times D\times V)}
+
\|\hat a_0\|_{H^{s_0}([0,T]\times D\times V)}
\le R,
\]
then
\[
\|p_{\hat b,\hat a}-p\|_{L^2(\mathbb P_c)L^2([0,T]\times\mathbb R^n)}
\le
C_{\mathrm{FP}}(R)\,
\mathrm{FP}(\hat b,\hat a),
\]
where
\[
\mathrm{FP}(\hat b,\hat a)
\triangleq
\left(
\E_{\theta\sim\mathbb P_c}
\left\|
\partial_t p(\theta,\cdot,\cdot)
-
\bmL^{(\hat b,\hat a), \theta}
p(\theta,\cdot,\cdot)
\right\|_{L^2([0,T]\times D)}^2
\right)^{1/2}.
\]
\end{lemma}
\begin{proof}
Fix \(\theta\in\Theta\). By Assumptions~\ref{as:uniform_ellipticity},
\ref{as:smooth_controls}, \ref{as:initial_density}, and
\ref{as:smooth_coeffs}, together with
\citet[Theorems~3.1--3.4]{bonalli2023non}, both SDEs admit unique strong
solutions with densities satisfying the stated Fokker--Planck equations. Set
\[
\rho_\theta \triangleq p_{\hat b,\hat a}(\theta)-p(\theta),
\qquad
r_\theta \triangleq
\partial_t p(\theta)-\bmL^{(\hat b,\hat a), \theta}p(\theta).
\]
Since both equations start from \(p_0\),
\[
\rho_\theta(0,\cdot)=0,
\qquad
\partial_t\rho_\theta
=
\bmL^{(\hat b,\hat a), \theta}\rho_\theta-r_\theta .
\]
Moreover, using \(\partial_t p(\theta)=\bmL^{(b,a), \theta}p(\theta)\),
\[
r_\theta
=
\bigl[
\bmL^{(b,a), \theta}
-
\bmL^{(\hat b,\hat a), \theta}
\bigr]p(\theta).
\]
Since the two diffusion matrices share the same fixed background term
\(\kappa I_n\), this term cancels in the operator difference, and therefore
\[
r_\theta(t,x)
=
\frac12
\sum_{i,j=1}^n
\partial_{x_i x_j}
\left(
\xi(x)
\bigl(a_{0,ij}-\hat a_{0,ij}\bigr)
(t,x,u_\theta(t))
p(\theta,t,x)
\right)
-
\sum_{i=1}^n
\partial_{x_i}
\left(
\xi(x)
\bigl(b_{0,i}-\hat b_{0,i}\bigr)
(t,x,u_\theta(t))
p(\theta,t,x)
\right).
\]
Since \(\xi\in C_c^\infty(\mathbb R^n)\) and
\(\operatorname{supp}(\xi)\subset D\), we have
\[
\|r_\theta\|_{L^2([0,T]\times\mathbb R^n)}
=
\|r_\theta\|_{L^2([0,T]\times D)}.
\]

Let \(r:=4+\lfloor n/2\rfloor\). The \(L^2\)-stability estimate of
\citet{bonalli2023non} gives
\[
\|\rho_\theta\|_{L^2([0,T]\times\mathbb R^n)}
\le
C\!\left(
\|(\hat b,\hat a-\kappa I_n)(\cdot,\cdot,u_\theta(\cdot))\|_
{H^r([0,T]\times D)}
\right)
\|r_\theta\|_{L^2([0,T]\times D)}.
\]

It remains to make the constant uniform in \(\theta\). Since
\[
s_0=5+n+\left\lceil \frac{d+1}{2}\right\rceil
>
r+\frac{n+d+1}{2},
\]
Lemma~\ref{lem:composition-controls}, applied with \(\theta\) fixed, \(q=r\),
and \(\tau=s_0\), gives for every component \(g\) of \(\hat b_0\) or \(\hat a_0\),
\[
\|g(\cdot,\cdot,u_\theta(\cdot))\|_{H^r([0,T]\times D)}
\le
C_\theta
\|g\|_{H^{s_0}([0,T]\times D\times V)} .
\]
By Assumption~\ref{as:smooth_controls}, the constants \(C_\theta\) are bounded
uniformly over \(\theta\in\Theta\). Since
\(\hat b=\xi\hat b_0\), \(\hat a-\kappa I_n=\xi\hat a_0\), and multiplication by
the fixed smooth cutoff \(\xi\) is continuous on Sobolev spaces, it follows that
\[
\sup_{\theta\in\Theta}
\|(\hat b,\hat a-\kappa I_n)(\cdot,\cdot,u_\theta(\cdot))\|_{H^r([0,T]\times D)}
\le C_R
\]
whenever
\[
\|\hat b_0\|_{H^{s_0}([0,T]\times D\times V)}
+
\|\hat a_0\|_{H^{s_0}([0,T]\times D\times V)}
\le R .
\]
Hence the parabolic stability constant is bounded uniformly over
\(\theta\in\Theta\). Denote this uniform bound by \(C_{\mathrm{FP}}(R)\). Then
\[
\|\rho_\theta\|_{L^2([0,T]\times\mathbb R^n)}
\le
C_{\mathrm{FP}}(R)
\|r_\theta\|_{L^2([0,T]\times D)}.
\]
Squaring, integrating with respect to \(\theta\sim\mathbb P_c\), and taking
square roots gives the result.
\end{proof}% \begin{proof}
\begin{lemma}[Fokker--Planck matching inequality for \(L^\infty\)-in-control learning rates]
\label{lem:fp_ineq_linfty}
Assume the setting and assumptions of Lemma~\ref{lem:FP_ineq}. Let
\(D\subset\mathbb R^n\) be the set from Assumption~\ref{as:p_s}.
For every \(R>0\), there exists \(C_{\mathrm{FP}}(R)>0\) such that, if
\[
\|\hat b_0\|_{H^{s_0}([0,T]\times D\times V)}
+
\|\hat a_0\|_{H^{s_0}([0,T]\times D\times V)}
\le R,
\]
then

\[
\|p_{\hat b,\hat a}-p\|_{L^\infty(\mathbb P_c)L^2([0,T]\times\mathbb R^n)}
\le
C_{\mathrm{FP}}(R)\,
\mathrm{FP}^\infty(\hat b,\hat a).
\]

where
\[
\mathrm{FP}^\infty(\hat b,\hat a)
\triangleq
\operatorname*{ess\,sup}_{\theta\sim\mathbb P_c}
\left\|
\partial_t p(\theta,\cdot,\cdot)
-
\bmL^{(\hat b,\hat a), \theta}
p(\theta,\cdot,\cdot)
\right\|_{L^2([0,T]\times D)} .
\]
\end{lemma}
\begin{proof}
The proof of Lemma~\ref{lem:FP_ineq} establishes, for every
\(\theta\in\Theta\), the fixed-control estimate
\[
\|p_{\hat b,\hat a}(\theta)-p(\theta)\|_{L^2([0,T]\times\mathbb R^n)}
\le
C_{\mathrm{FP}}(R)
\left\|
\partial_t p(\theta)
-
\bmL^{(\hat b,\hat a), \theta}p(\theta)
\right\|_{L^2([0,T]\times D)} ,
\]
where \(C_{\mathrm{FP}}(R)\) is independent of \(\theta\). Taking the essential
supremum over \(\theta\sim\mathbb P_c\) proves the claim.
\end{proof}

\subsection{Useful Preliminary Results for Proving FP Matching Learning rates}\label{subsec:preli}

We state the auxiliary results used to prove the FP matching learning rates for the estimator of Section~\ref{sec:proposed_method}, with hard shape-constrained PSD diffusion. Since the true and candidate diffusion matrices share the term $\kappa I_n$, it cancels in the Fokker-Planck residual, so the matching error depends only on the localized components.

\paragraph{FP matching as least-squares regression.}
Let \(k\) be the scalar kernel used in the coefficient hypothesis space, let
\(\bmG\) be its RKHS, and write, with a slight abuse of notation,
% \[
% \phi(t,x,v)\triangleq k((t,x,v),\cdot)\in\bmG .
% \]
\[
\phi(t,x,v)\triangleq \xi(x)\,k((t,x,v),\cdot)\in\bmG,
\qquad
(\phi\otimes\phi)(t,x,v)
\triangleq
\xi(x)\,
\bigl(k((t,x,v),\cdot)\otimes k((t,x,v),\cdot)\bigr)
\in\bmG\otimes\bmG .
\]
Define also
\[
\bmG_{\mathrm{FP}}
\triangleq
\bmG^n \times (\bmG\otimes\bmG)^{n^2}.
\]
The induced Fokker--Planck feature map
\(
\tilde\phi : \Theta\times[0,T]\times\R^n \to \bmG_{\mathrm{FP}}
\)
is defined by
\begin{align}
\tilde \phi(\theta,t,x)
\triangleq
\left(
(-\tilde\phi_i)_{i\in\llbracket1,n\rrbracket}
\;\middle|\;
\left(\tfrac12\tilde\phi_{ij}\right)_{i,j\in\llbracket1,n\rrbracket}
\right)(\theta,t,x)
\in \bmG_{\mathrm{FP}},
\end{align}
where
\begin{align}
\tilde\phi_i(\theta,t,x)
&\triangleq
\left(
\phi\,\partial_{x_i}p
+
p\,\phi_i
\right)\!\bigl(t,x,u_\theta(t)\bigr)
\in\bmG,
\\
\tilde\phi_{ij}(\theta,t,x)
&\triangleq
\left(
\phi\otimes\phi\,\partial_{x_ix_j}p
+
(\phi\otimes\phi)_i\,\partial_{x_j}p
+
(\phi\otimes\phi)_j\,\partial_{x_i}p
+
p(\phi\otimes\phi)_{ij}
\right)\!\bigl(t,x,u_\theta(t)\bigr)
\in\bmG\otimes\bmG .
\end{align}

The FP matching problem can then be written as least-squares regression:
\begin{align}
\mathrm{FP}^2(b,a)
=
\E_{\theta,t,x}
\Big[
\big(
\partial_t p(\theta,t,x)
-
\langle w,\tilde\phi(\theta,t,x)\rangle_{\bmG_{\mathrm{FP}}}
\big)^2
\Big],
\end{align}
where, throughout this proof, \((t,x)\) is uniform on \([0,T]\times D\) and \(\theta\sim\mathbb P_c\). Since \(T|D|\) is fixed, the resulting factor \(\sqrt{T|D|}\) can be absorbed into the constant \(C_{\mathrm{FP}}\) in the Fokker--Planck matching inequality.

\paragraph{Empirical FP matching as ridge regression.}

Let \(w=((w_i)_{i=1}^n,(w_{ij})_{i,j=1}^n)\in\bmG_{\mathrm{FP}}\), and define
\begin{align}
    \bmS_{\mathrm{FP}} \triangleq \bmG^{n} \times \bmS \subset \bmG_{\mathrm{FP}}.
\end{align}
Regularized empirical FP matching, Eq.~\eqref{eq:proposed_estimator}, can be written as ridge regression:
\begin{align}\label{eq:w_s}
    \hat w_S 
    \triangleq 
    \argmin_{w \in \bmS_{\mathrm{FP}}}
    \frac{1}{K N}
    \sum_{k=1}^K  
    \sum_{i=1}^{N}
    \left(
    \frac{\partial \hat p_k}{\partial t}(t_i, x_i)
    -
    \langle w,\, \hat \phi^k(t_i, x_i) \rangle_{\bmG_{\mathrm{FP}}}
    \right)^2
    +
    \lambda  \|w\|^2_{\bmG_{\mathrm{FP}}}.
\end{align}
The empirical feature map \(\hat\phi^k(t,x)\in\bmG_{\mathrm{FP}}\) is
\begin{align}
\hat \phi^k(t,x)
\triangleq
\left(
(-\hat \phi_i^k)_{i\in\llbracket1,n\rrbracket}
\;\middle|\;
\left(\tfrac12\hat \phi_{ij}^k\right)_{i,j\in\llbracket1,n\rrbracket}
\right)(t,x,u_{\theta_k}(t))
\in\bmG_{\mathrm{FP}},
\end{align}
where
\begin{align}
\hat \phi_i^k(t,x,v)
&\triangleq
\left(
\phi \frac{\partial \hat p_k}{\partial x_i}
+
\hat p_k \phi_i
\right)(t,x,v),
\\
\hat \phi_{ij}^k(t,x,v)
&\triangleq
\left(
(\phi\otimes\phi)\frac{\partial^2 \hat p_k}{\partial x_i\partial x_j}
+
(\phi\otimes\phi)_i\frac{\partial \hat p_k}{\partial x_j}
+
(\phi\otimes\phi)_j\frac{\partial \hat p_k}{\partial x_i}
+
\hat p_k(\phi\otimes\phi)_{ij}
\right)(t,x,v).
\end{align}

\paragraph{Auxiliary ridge estimators for the proof of Theorem~\ref{thm:lr}.}

We define the unconstrained ridge estimator \(\hat w\), obtained from Eq.~\eqref{eq:w_s} by removing the cone constraint:
\begin{align}
\hat w
\triangleq
\argmin_{w\in\bmG_{\mathrm{FP}}}
\frac{1}{KN}
\sum_{k=1}^K
\sum_{i=1}^N
\left(
\frac{\partial \hat p_k}{\partial t}(t_i,x_i)
-
\langle w,\hat\phi^k(t_i,x_i)\rangle_{\bmG_{\mathrm{FP}}}
\right)^2
+
\lambda\|w\|^2_{\bmG_{\mathrm{FP}}}.
\end{align}
We also define the ridge estimator \(\tilde w\), whose error comes only from finite sampling of \(\Theta\times[0,T]\times D\), excluding density-estimation errors from \((\hat p_k)_{k=1}^K \approx (p(\theta_k))_{k=1}^K\):
\begin{align}
\tilde w
\triangleq
\argmin_{w\in\bmG_{\mathrm{FP}}}
\frac{1}{KN}
\sum_{k=1}^K
\sum_{i=1}^N
\left(
\frac{\partial p}{\partial t}(\theta_k,t_i,x_i)
-
\langle w,\tilde\phi(\theta_k,t_i,x_i)\rangle_{\bmG_{\mathrm{FP}}}
\right)^2
+
\lambda\|w\|^2_{\bmG_{\mathrm{FP}}}.
\end{align}
Finally, we define \(w_K\), whose error comes only from finite sampling of the control parameters:
\begin{align}
w_K
\triangleq
\argmin_{w\in\bmG_{\mathrm{FP}}}
\frac{1}{K}
\sum_{k=1}^K
\E_{t,x}
\Big[
\big(
\partial_t p(\theta_k,t,x)
-
\langle w,\tilde\phi(\theta_k,t,x)\rangle_{\bmG_{\mathrm{FP}}}
\big)^2
\Big]
+
\lambda\|w\|^2_{\bmG_{\mathrm{FP}}}.
\end{align}

The next proposition gives closed-form expressions for these ridge estimators.

\begin{proposition}[\(\hat w\), \(\tilde w\), and \(w_K\) expressions] 
For any operator \(A\), define \(A_{\lambda} \triangleq A + \lambda I\), where \(\otimes\) denotes the tensor product. Then \(\hat w\) has the standard ridge regression form
\begin{align}
    \hat w = \hat D \hat C_{\lambda}^{-1},
\end{align}
with \(\hat D \triangleq \frac{1}{K N} \sum_{k=1}^K\sum_{i=1}^N \frac{\partial \hat p_k}{\partial t}(t_i, x_i) \hat \phi^k(t_i,x_i)\), and 
\(\hat C \triangleq \frac{1}{K N} \sum_{k=1}^K\sum_{i=1}^N \hat \phi^k(t_i,x_i) \otimes \hat \phi^k(t_i,x_i)\).

Similarly, \(\tilde w = \tilde D \tilde C_{\lambda}^{-1}\), with
\begin{align}
    \tilde D \triangleq \frac{1}{NK}\sum_{k=1}^K \sum_{i=1}^N \frac{\partial p}{\partial t}(\theta_k, t_i, x_i) \tilde\phi(\theta_k,t_i,x_i),\quad &\text{and} \quad 
    \tilde C \triangleq \frac{1}{NK} \sum_{k=1}^K \sum_{i=1}^N \tilde \phi(\theta_k,t_i,x_i) \otimes \tilde \phi(\theta_k,t_i,x_i).
\end{align}
We also have \(w_K = D_K C_{K, \lambda}^{-1}\), with
\begin{align}
    D_K \triangleq \frac{1}{K}\sum_{k=1}^K \E_{t,x}\left[\frac{\partial p}{\partial t}(\theta_k, t, x) \tilde \phi(\theta_k,t,x)\right],\quad &\text{and} \quad 
    C_K \triangleq \frac{1}{K} \sum_{k=1}^K \E_{t,x}\left[\tilde \phi(\theta_k,t,x) \otimes \tilde \phi(\theta_k,t,x)\right].
\end{align}
\end{proposition}

\begin{proposition}[Attainability of \(\partial_t p\)]
Assumption~\ref{as:attain} is equivalent to the existence of 
\(W_b \in \R^n \otimes \bmG\) and 
\(W_a \in \R^{n^2} \otimes (\bmG \otimes \bmG)\) such that
\begin{align}
(b,a-\kappa I_n)
=
\left(
W_b \phi(\cdot),
W_a (\phi(\cdot)\otimes\phi(\cdot))
\right)
\quad
\text{with}
\quad
\|W_b\|_{\hs} < +\infty,
\quad
\|W_a\|_{\hs} < +\infty,
\end{align}
and such that the representation \([W_a]_{\bmG^n \otimes \bmG^n}\) of \(W_a\) in \(\bmG^n \otimes \bmG^n\) satisfies
\begin{align}
[W_a]_{\bmG^n \otimes \bmG^n} \in \bmS .
\end{align}

The assumptions of Section~\ref{subsec:fp_ineq} further ensure that there exists
\(w \in \bmS_{\mathrm{FP}}\) such that
\begin{align}
\frac{\partial p}{\partial t}(\theta,t,x)
=
\left(
\bmL^{(b,a), \theta} p
\right)(\theta,t,x)
=
\langle w ,\, \tilde \phi(\theta,t,x)\rangle_{\bmG_{\mathrm{FP}}}.
\end{align}

Furthermore, the Fokker--Planck matching \(L^2\)-risk can be written as
\begin{align}
\mathrm{FP}(\hat b,\hat a)
=
\|(\hat w_S - w)C^{1/2}\|_{\bmG_{\mathrm{FP}}},
\end{align}
with
\begin{align}
C
\triangleq
\E_{\theta,t,x}
\left[
\tilde \phi(\theta,t,x) \otimes \tilde \phi(\theta,t,x)
\right].
\end{align}
See also \citet{ciliberto2020general}.
\end{proposition}

\begin{remark}[Alternative assumptions to Assumptions~\ref{as:emb}]
We denote \(\tilde \phi \triangleq \tilde \phi(\theta, t, x)\). The following three assumptions are standard alternatives for measuring the regularity of the feature space in least-squares regression:
\begin{itemize}
    \item \textbf{(1)} There exists \(r_1 \in [0,1]\) and \(c > 0\) such that 
    \begin{align}
        \Tr(C^{1-r_1}) < c.
    \end{align}
    \item \textbf{(2)} There exists \(r_2 \in [0,1]\) and \(c > 0\) such that 
    \begin{align}
        \mathbb{E}\!\left[\langle \tilde \phi, C^{-r_2} \tilde \phi \rangle_{\bmG_{\mathrm{FP}}} \tilde \phi \otimes \tilde \phi\right] \preccurlyeq c\, C.
    \end{align}
    \item \textbf{(3)} There exists \(r_3 \in [0,1]\) and \(c > 0\) such that 
    \begin{align}
        \|C^{-r_3/2} \tilde \phi\|_{\bmG_{\mathrm{FP}}} < c 
        \quad \text{almost surely}.
    \end{align}
\end{itemize}

Assumptions (1) and (3) are known as the capacity condition and embedding property, respectively \citep{pillaud2018statistical, fischer2020sobolev}. Assumptions (1), (2), and (3) give progressively finer regularity conditions on the features \(\tilde \phi(\theta,t,x)\). Indeed, Assumption (2) implies Assumption (1) with \( r_1=r_2\), and Assumption (3) implies Assumption (2) with \(r_2=r_3 \); see Remark 3 in \citet{berthier2020tight}. In noisy kernel ridge regression, Assumption (1) refines rates from \(N^{-1/4}\) when \(r_1=0\) to \(N^{-1/2}\) when \(r_1=1\). In noiseless kernel ridge regression, Assumption (2) refines rates from \(N^{-1/2}\) when \(r_2=0\) to \(N^{-1}\) when \(r_2=1\), while Assumption (3) refines them from \(N^{-1}\) when \(r_3=0\) to arbitrarily fast polynomial decay as \(r_3\to1\).
For completeness, we present all three assumptions, but for clarity we base our proofs on the stronger Assumption (3).
\end{remark}

\subsection{Proof of Theorem \ref{thm:lr}}\label{subsec:proof_theorem}

\begin{theorem}[\(L^2\) Learning rates]
For any \(N,K\in\mathbb N^*\), let \((\hat b,\hat a)\) be the estimator
defined in Section~\ref{sec:proposed_method}, with hard shape-constrained
PSD diffusion, trained using \(K\) control parameters
\((\theta_k)_{k=1}^K\overset{\mathrm{i.i.d.}}{\sim}\mathbb P_c\) and
\(N\) state-time points \(z_i=(s_i,y_i)\overset{\mathrm{i.i.d.}}{\sim}
\mathrm{Unif}([0,T]\times D)\), where \(D\subset\mathbb R^n\) is the
localization domain of Assumption~\ref{as:p_s}. 
Assume that Assumptions~\ref{as:uniform_ellipticity}--\ref{as:attain} hold. Then there exist constants $c_1, c_2 > 0$, independent of $N$, $K$, and $\delta$,
such that for any $\delta \in (0,1]$, defining \(\varepsilon \triangleq \sup_{k \in \llbracket 1,\, K\rrbracket} \mathcal{E}(\hat p_k, p(\theta_k))\) where \( \mathcal{E}\) is a Sobolev-type error defined as
\begin{align}
    \mathcal{E}(p_1, p_2)^2 \triangleq \int_{0}^T \left\| \frac{\partial p_1}{\partial t}(t, \cdot) - \frac{\partial p_2}{\partial t}(t, \cdot)\right\|_{L^2(\R^n)}^2 +  \left\|p_1(t, \cdot) -p_2(t, \cdot)\right\|_{H^2(\R^n)}^2 dt,
\end{align}
and setting
\begin{align}
    \lambda = c_2 \log \frac{2}{\delta} \left(\frac{1}{N}\log^2\frac{N}{\delta} + \frac{1}{K}\log^2\frac{K}{\delta}\right),
\end{align}
if 
\begin{align}
    \varepsilon \leq \frac{1}{N}\log^2 \frac{N}{\delta} + \frac{1}{K}\log^2 \frac{K}{\delta},
\end{align}
then with probability at least $1-\delta$,
\begin{align}
   \|p_{\hat b, \hat a} - p\|_{L^2(\mathbb P_c)L^2([0,T]\times \R^n)} \leq c_1 \log \frac{2}{\delta} \left(\frac{\log \frac{N}{\delta}}{\sqrt{N}} + \frac{\log \frac{K}{\delta}}{\sqrt{K}}\right).
\end{align}
\end{theorem}
\begin{sproof}
The proof involves decomposing the error of the proposed estimator into three components:
\begin{enumerate}
    \item Error stemming from the estimation of the probability densities \((\hat{p}_k)_{k=1}^K \approx (p(\theta_k))_{k=1}^K\),
    \item Error due to the finite-sample approximation over the localized state-time domain \([0,T]\times D\),
    \item Error due to the finite sampling approximation of the space of control parameters \(\Theta\),
\end{enumerate}
and then bounding each error component.

Denoting $z=(t,x)$, $\tilde{\phi} = \tilde{\phi}(\theta, t, x)$, and $\hat{\mathbb{E}}[\cdot]$ as the empirical expectation over the training points, this corresponds to the following successive approximations.
\begin{align}
    \hat{C} \triangleq \hat{\mathbb{E}}_\theta \hat{\mathbb{E}}_z [\hat{\phi} \otimes \hat{\phi}] \quad\rightarrow\quad \tilde{C} \triangleq \hat{\mathbb{E}}_\theta \hat{\mathbb{E}}_z [\tilde{\phi} \otimes \tilde{\phi}] \quad\rightarrow\quad C_K \triangleq \hat{\mathbb{E}}_\theta \mathbb{E}_z [\tilde{\phi} \otimes \tilde{\phi}] \quad\rightarrow\quad C \triangleq \mathbb{E}_\theta \mathbb{E}_z [\tilde{\phi} \otimes \tilde{\phi}].
\end{align}
\end{sproof}

\begin{proof}
The definitions of Section \ref{subsec:preli} allow for the following decomposition
\begin{align}
    \|(\hat w - w) C^{1/2}\|_{\bmG_{\mathrm{FP}}} \leq \|(\hat w - \tilde w) C^{1/2}\|_{\bmG_{\mathrm{FP}}}+  \|(\tilde w - w_K) C^{1/2}\|_{\bmG_{\mathrm{FP}}} + \|(w_K - w) C^{1/2}\|_{\bmG_{\mathrm{FP}}}.
\end{align}

Using Lemmas \ref{lem:cone}, \ref{lem:error_1}, \ref{lem:error_2}, and \ref{lem:error_3}, and Lemma \ref{lem:kappa}, we determine that for any \(\delta \in (0,1]\), if \(\lambda \geq c_2 \log \frac{2}{\delta} (N^{-1} + \varepsilon)\), \(\lambda \geq \frac{9\kappa_{\tilde \phi}^2}{N}\log \frac{N}{\delta}\), and \(\lambda \geq \frac{9\kappa_{\tilde \phi}^2}{K}\log \frac{K}{\delta}\), with probability \(1-\delta\),
\begin{align}
    \text{FP}(\hat b, \hat  a) \leq c_1 \left((N^{-1} + \varepsilon) \lambda^{-1/2} \log \frac{2}{\delta}  + \frac{\log \frac{N}{\delta}}{\sqrt{N}} + \frac{\log \frac{K}{\delta}}{\sqrt{K}} + \lambda^{1/2}\right),
\end{align}
where constants \(c_1, c_2 > 0\) are independent of \(N\), \(K\), and \(\delta\). Therefore, setting
\begin{align}
    \lambda = c_2 \log \frac{2}{\delta} \left(\frac{1}{N}\log^2 \frac{N}{\delta} + \frac{1}{K}\log^2 \frac{K}{\delta}\right),
\end{align}
and assuming
\begin{align}
    \varepsilon \leq \frac{1}{N}\log^2 \frac{N}{\delta} + \frac{1}{K}\log^2 \frac{K}{\delta},
\end{align}
the three lower bounds on \(\lambda\) hold, up to increasing \(c_2\). Hence, after overloading \(c_1>0\),
\begin{align}
    \text{FP}(\hat b, \hat  a) \leq c_1 \log \frac{2}{\delta} \left(\frac{\log \frac{N}{\delta}}{\sqrt{N}} + \frac{\log \frac{K}{\delta}}{\sqrt{K}}\right).
\end{align}
Finally, Lemma \ref{lem:FP_ineq} completes the proof.
\end{proof}

\subsection{Proof of Lemma \ref{lem:cone}}\label{subsec:cone}

This lemma shows that the error of \(\hat{w}_S\) can be bounded by that of the unconstrained ridge estimator \(\hat{w}\). This stems from the fact that using the same empirical learning objective with a subset \(\bmS_{\mathrm{FP}} \subset \bmG_{\mathrm{FP}}\), rather than the full set \(\bmG_{\mathrm{FP}}\), cannot increase the error, provided that \(w \in \bmS_{\mathrm{FP}}\).

\begin{lemma}\label{lem:cone}
Under Assumption~\ref{as:attain}, the following inequality holds:
\begin{align}
    \|(\hat w_S - w)C^{1/2}\|_{\bmG_{\mathrm{FP}}}
    \leq
    \|(\hat w - w)C_{\lambda}^{1/2}\|_{\bmG_{\mathrm{FP}}}
    \left(
    1+
    \|C_{\lambda}^{-1/2}\hat C_{\lambda}^{1/2}\|_{\infty}
    \|\hat C_{\lambda}^{-1/2}C^{1/2}\|_{\infty}
    \right).
\end{align}
\end{lemma}

\begin{proof}
By the triangle inequality,
\begin{align}
    \|(\hat w_S - w)C^{1/2}\|_{\bmG_{\mathrm{FP}}}
    \leq
    \|(\hat w_S - \hat w)C^{1/2}\|_{\bmG_{\mathrm{FP}}}
    +
    \|(\hat w - w)C^{1/2}\|_{\bmG_{\mathrm{FP}}}.
\end{align}
Since \(C\preceq C_\lambda\), the second term is bounded by
\begin{align}
    \|(\hat w - w)C^{1/2}\|_{\bmG_{\mathrm{FP}}}
    \leq
    \|(\hat w - w)C_{\lambda}^{1/2}\|_{\bmG_{\mathrm{FP}}}.
\end{align}
For the first term, recalling that \(\hat w=\hat D\hat C_\lambda^{-1}\), the constrained problem can be rewritten as
\begin{align}
    \hat w_S
    =
    \argmin_{v\in\bmS_{\mathrm{FP}}}
    \|(v-\hat w)\hat C_{\lambda}^{1/2}\|_{\bmG_{\mathrm{FP}}}.
\end{align}
Since \(w\in\bmS_{\mathrm{FP}}\), we get
\begin{align}
    \|(\hat w_S-\hat w)\hat C_{\lambda}^{1/2}\|_{\bmG_{\mathrm{FP}}}
    \leq
    \|(w-\hat w)\hat C_{\lambda}^{1/2}\|_{\bmG_{\mathrm{FP}}}.
\end{align}
Therefore,
\begin{align}
    \|(\hat w_S-\hat w)C^{1/2}\|_{\bmG_{\mathrm{FP}}}
    &\leq
    \|(\hat w_S-\hat w)\hat C_{\lambda}^{1/2}\|_{\bmG_{\mathrm{FP}}}
    \|\hat C_{\lambda}^{-1/2}C^{1/2}\|_{\infty} \\
    &\leq
    \|(\hat w-w)C_{\lambda}^{1/2}\|_{\bmG_{\mathrm{FP}}}
    \|C_{\lambda}^{-1/2}\hat C_{\lambda}^{1/2}\|_{\infty}
    \|\hat C_{\lambda}^{-1/2}C^{1/2}\|_{\infty}.
\end{align}
Combining the previous bounds gives the result.
\end{proof}

\subsection{Proof of Lemma \ref{lem:error_1}}\label{subsec:error_1}

This lemma bounds the error stemming from the estimation of the probability densities \((\hat{p}_k)_{k=1}^K \approx (p(\theta_k))_{k=1}^K\).
\begin{lemma}[Bound $\|(\hat w - \tilde w) C^{1/2}\|_{\bmG_{\mathrm{FP}}}$]\label{lem:error_1} There exist constants $c_1, c_2>0$ that do not depend on $N, K, \delta$, such that, for any $\delta \in (0, 1]$, with probability at least $1-\delta$
\begin{align}
    \|(\hat w - \tilde w) C^{1/2}\|_{\bmG_{\mathrm{FP}}} \leq c_1\log \frac{2}{\delta} (N^{-1} + \varepsilon) \lambda^{-1/2},
\end{align}
provided that $\lambda \geq c_2 \log \frac{2}{\delta} (N^{-1} + \varepsilon)$, $\lambda \geq \frac{18c}{N}\log \frac{N}{\delta}$ and $\lambda \geq \frac{18c}{K}\log \frac{K}{\delta}$.
\end{lemma}

\begin{proof} We have
\begin{align}
    \|(\hat w - \tilde w) C^{1/2}\|_{\bmG_{\mathrm{FP}}} &= \|(\hat D \hat C_{\lambda}^{-1} - \tilde D \tilde C_{\lambda}^{-1}) C^{1/2}\|_{\bmG_{\mathrm{FP}}}\\
    &\leq \|(\hat D  - \tilde D)\hat C_{\lambda} ^{-1} C^{1/2}\|_{\bmG_{\mathrm{FP}}} + \|\tilde D (\hat C_{\lambda}^{-1} - \tilde C_{\lambda}^{-1}) C^{1/2}\|_{\bmG_{\mathrm{FP}}}.
\end{align}
Moreover,
\begin{align}
    \|(\hat D  - \tilde D)\hat C_{\lambda} ^{-1} C^{1/2}\|_{\bmG_{\mathrm{FP}}} \leq \|\hat D  - \tilde D \|_{\bmG_{\mathrm{FP}}} \|\hat C_{\lambda} ^{-1} C^{1/2}\|_{\infty}.
\end{align}
Furthermore, using $A^{-1} - B^{-1} = B^{-1}(B - A)A^{-1}$, we have
\begin{align}
    \|\tilde D (\hat C_{\lambda}^{-1} - \tilde C_{\lambda}^{-1}) C^{1/2}\|_{\bmG_{\mathrm{FP}}} &= \|\tilde D \tilde C_{ \lambda}^{-1}(\tilde C - \hat C)\hat C_{\lambda}^{-1} C^{1/2}\|_{\bmG_{\mathrm{FP}}}.
\end{align}

\paragraph{1. Bound $\|\hat D  - \tilde D\|_{\bmG_{\mathrm{FP}}}$.} From Lemma \ref{lem:hat_D}, for any $\delta \in (0, 1]$, with probability at least $1-\delta$
\begin{align}
    \|\hat D  - \tilde D\|_{\bmG_{\mathrm{FP}}} &\leq c \log \frac{2}{\delta} (N^{-1} + \varepsilon),
\end{align}
where $c>0$ is a constant that does not depend on $N, K, \delta$.

\paragraph{2. Bound $\|\hat C_{\lambda} ^{-1} C^{1/2}\|_{\infty}$.} We have
\begin{align}
    \|\hat C_{\lambda} ^{-1} C^{1/2}\|_{\infty} &\leq \lambda^{-1/2}  \|\hat C_{\lambda}^{-1/2} \tilde C_{\lambda}^{1/2}\|_{\infty} \|\tilde C_{\lambda}^{-1/2}C^{1/2}\|_{\infty}.
\end{align}
By Lemmas \ref{lem:hat_C} and \ref{lem:sampling}, for any $\delta \in (0, 1]$, if $\lambda \geq c \log \frac{2}{\delta} (N^{-1} + \varepsilon)$, $\lambda \geq \frac{18c}{N}\log \frac{N}{\delta}$ and $\lambda \geq \frac{18c}{K}\log \frac{K}{\delta}$, then with probability at least $1-\delta$,
\begin{align}
    \|\hat C_{\lambda} ^{-1} C^{1/2}\|_{\infty} &\leq  2\sqrt{2} \lambda^{-1/2},
\end{align}
where $c>0$ is a constant that does not depend on $N, K, \delta$.

\paragraph{3. Bound $\|\tilde D \tilde C_{ \lambda}^{-1}(\tilde C - \hat C)\hat C_{\lambda}^{-1} C^{1/2}\|_{\bmG_{\mathrm{FP}}}$.} From Assumption \ref{as:attain}, we have
\begin{align}
    \tilde D = w \tilde C.
\end{align}
Therefore,
\begin{align}
    \|\tilde D \tilde C_{ \lambda}^{-1}(\tilde C - \hat C)\hat C_{\lambda}^{-1} C^{1/2}\|_{\bmG_{\mathrm{FP}}} &\leq  \|w\|_{\bmG_{\mathrm{FP}}} \|\tilde C \tilde C_{ \lambda}^{-1}(\tilde C - \hat C)\hat C_{\lambda}^{-1} C^{1/2}\|_{\infty}\\
    &\leq \|w\|_{\bmG_{\mathrm{FP}}} \|\tilde C - \hat C\|_{\infty} \|\hat C_{\lambda}^{-1} C^{1/2}\|_{\infty}.
\end{align}
Then, as for bound 2., from Lemma \ref{lem:hat_C} and Lemma \ref{lem:sampling}, we have
\begin{align}
    \|\tilde D \tilde C_{ \lambda}^{-1}(\tilde C - \hat C)\hat C_{\lambda}^{-1} C^{1/2}\|_{\bmG_{\mathrm{FP}}} &\leq c \lambda^{-1/2} \log \frac{2}{\delta} (N^{-1} + \varepsilon),
\end{align}
where $c>0$ is a constant that does not depend on $N, K, \delta$.

\paragraph{Conclusion.} Combining all bounds, we conclude that there exist constants $c_1, c_2>0$ that do not depend on $N, K, \delta$, such that, for any $\delta \in (0, 1]$, with probability at least $1-\delta$
\begin{align}
    \|(\hat w - \tilde w) C^{1/2}\|_{\bmG_{\mathrm{FP}}} \leq c_1\log \frac{2}{\delta} (N^{-1} + \varepsilon) \lambda^{-1/2},
\end{align}
if $\lambda \geq c_2 \log \frac{2}{\delta} (N^{-1} + \varepsilon)$, $\lambda \geq \frac{18c}{N}\log \frac{N}{\delta}$ and $\lambda \geq \frac{18c}{K}\log \frac{K}{\delta}$.
\end{proof}

\subsection{Proof of Lemma \ref{lem:error_2}}\label{subsec:error_2}

This lemma bounds the error resulting from the finite-sample approximation over the localized state-time domain \([0,T]\times D\).
\begin{lemma}[Bound $\|(\tilde w - w_K) C^{1/2}\|_{\bmG_{\mathrm{FP}}}$]\label{lem:error_2}  Under Assumption \ref{as:attain}, there exists a constant $c>0$ that does not depend on $N, K, \delta$, such that, for any $\delta \in (0, 1]$, with probability at least $1-\delta$
\begin{align}
    \|(\tilde w - w_K) C^{1/2}\|_{\bmG_{\mathrm{FP}}}  \leq c \frac{\log \frac{N}{\delta}}{\sqrt{N}},
\end{align}
provided that $\lambda \geq \frac{9\kappa_{\tilde \phi}^2}{N}\log \frac{N}{\delta}$ and $\lambda \geq \frac{9\kappa_{\tilde \phi}^2}{K}\log \frac{K}{\delta}$.
    
\end{lemma}

\begin{proof} From Assumption \ref{as:attain}, we have
\begin{align}
    \tilde D = w \tilde C \quad &\text{and} \quad D_K = w C_K.
\end{align}
The error norm can be decomposed as follows
\begin{align}
    \|(\tilde w - w_K) C^{1/2}\|_{\bmG_{\mathrm{FP}}} &= \|(\tilde D \tilde C_{\lambda}^{-1} - D_K C_{K, \lambda}^{-1}) C^{1/2}\|_{\bmG_{\mathrm{FP}}}\\
    &\leq \|w\|_{\bmG_{\mathrm{FP}}} \|(\tilde C \tilde C_{\lambda}^{-1} - C_K C_{K, \lambda}^{-1}) C^{1/2}\|_{\infty}.
\end{align}
Moreover, using $A^{-1} - B^{-1} = B^{-1}(B - A)A^{-1}$, we have
\begin{align}
    \|(\tilde C \tilde C_{\lambda}^{-1} - C_K C_{K, \lambda}^{-1}) C^{1/2}\|_{\infty} &= \lambda \|(\tilde C_{\lambda}^{-1} - C_{K, \lambda}^{-1}) C^{1/2}\|_{\infty}\\
    &= \lambda \|\tilde C_{\lambda}^{-1} (C_{K} - \tilde C)C_{K, \lambda}^{-1} C^{1/2}\|_{\infty}\\
    &\leq  \lambda^{1/2}  \|\tilde C_{\lambda}^{-1/2} C_{K, \lambda}^{1/2}\|_{\infty} \|C_{K, \lambda}^{-1/2}(\tilde C - C_{K})C_{K, \lambda}^{-1/2}\|_{\infty} \|C_{K, \lambda}^{1/2}C^{1/2}\|_{\infty}.
\end{align}

\paragraph{1. Bound $\|C_{K, \lambda}^{-1/2}(\tilde C - C_{K})C_{K, \lambda}^{-1/2}\|_{\infty}$.} From Proposition \ref{prop:mixed_cross_cov}, for any $\delta \in (0,1]$, with probability at least $1-\delta$,
\begin{align}
    \|C_{K, \lambda}^{-1/2}(\tilde C - C_K)C_{K, \lambda}^{-1/2}\|_{\infty} \leq \frac{2\beta \kappa_{\tilde \phi}^2}{\lambda N} + \sqrt{\frac{2\kappa_{\tilde \phi}^2\beta}{3\lambda N}},
\end{align}
with $\beta = \log \frac{4 (\|C_K\|_{\infty} + \lambda) \Tr(C_K)}{\delta \lambda \|C_K\|_{\infty}}$.
Moreover,
\begin{align}
    \|C_K\|_{\infty} \leq \Tr(C_K)
    = \frac{1}{K} \sum_{k=1}^K \E[\|\tilde \phi(\theta_k, t,x)\|_{\bmG_{\mathrm{FP}}}^2]
    \leq \kappa_{\tilde \phi}^2.
\end{align}

Hence, if $N^{-1} \leq \lambda \leq \|C_K\|_{\infty}$, it is straightforward to verify that $\beta \leq \log (4\kappa_{\tilde \phi}^2 \lambda^{-1} \delta^{-1})$, and that there exists a constant $c>0$ that does not depend on $\delta, N, K$ such that
\begin{align}
    \|C_{K, \lambda}^{-1/2}(\tilde C - C_K)C_{K, \lambda}^{-1/2}\|_{\infty} \leq c \frac{\log \frac{N}{\delta}}{\sqrt{\lambda N}}.
\end{align}
Moreover,
\begin{align}
    \|C_K\|_{\infty} \geq \|C\|_{\infty} - \|C_K - C\|_{\infty},
\end{align}
and same proof as Proposition \ref{prop:mixed_cross_cov} gives $\|C_K - C\|_{\infty} \leq c K^{-1/2} \log(K\delta^{-1})$ for a constant $c>0$ that does not depend on $K, \delta$, if $K^{-1} \leq \lambda \leq \|C\|_{\infty}$, such that
\begin{align}
    \lambda \leq \|C\|_{\infty} - c K^{-1/2} \log(K\delta^{-1}),
\end{align}
ensures $\lambda \leq \|C_K\|_{\infty}$.

\paragraph{2. Bound $\|\tilde C_{\lambda}^{-1/2} C_{K, \lambda}^{1/2}\|_{\infty} \|C_{K, \lambda}^{1/2}C^{1/2}\|_{\infty}$.} From Lemma \ref{lem:sampling}, for any $\delta \in (0,1]$, if $\lambda \geq \frac{9\kappa_{\tilde \phi}^2}{N}\log \frac{N}{\delta}$ and $\lambda \geq \frac{9\kappa_{\tilde \phi}^2}{K}\log \frac{K}{\delta}$, with probability at least $1-\delta$,
\begin{align}
     \|\tilde C_{\lambda}^{-1/2} C_{K, \lambda}^{1/2}\|_{\infty} \|C_{K, \lambda}^{1/2}C^{1/2}\|_{\infty} &\leq 2.
\end{align}

\paragraph{Conclusion.} We conclude by combining all bounds that there exists a constant $c>0$ that does not depend on $N, K, \delta$, such that, for any $\delta \in (0, 1]$, with probability at least $1-\delta$
\begin{align}
    \|(\tilde w - w_K) C^{1/2}\|_{\bmG_{\mathrm{FP}}}  \leq c \frac{\log \frac{N}{\delta}}{\sqrt{N}},
\end{align}
if $\lambda \geq \frac{9\kappa_{\tilde \phi}^2}{N}\log \frac{N}{\delta}$ and $\lambda \geq \frac{9\kappa_{\tilde \phi}^2}{K}\log \frac{K}{\delta}$.
\end{proof}

\subsection{Proof of Lemma \ref{lem:error_3}}\label{subsec:error_3}

This lemma bounds the error stemming from the finite sampling approximation of $\Theta$, and the RKHS norm regularization.
\begin{lemma}[Bound $\|(w_K - w) C^{1/2}\|_{\bmG_{\mathrm{FP}}}$]\label{lem:error_3} Under assumption \ref{as:attain}, There exists a constant $c>0$ that does not depend on $K, \delta$, such that, for any $\delta \in (0, 1]$, with probability at least $1-\delta$
\begin{align}
    \|(w_K - w) C^{1/2}\|_{\bmG_{\mathrm{FP}}} \leq c \left(\frac{\log \frac{K}{\delta}}{\sqrt{K}} + \lambda^{1/2}\right),
\end{align}
if $\lambda \geq \frac{9\kappa_{\tilde \phi}^2}{K}\log \frac{K}{\delta}$.

\end{lemma}
\begin{proof} We have
\begin{align}
    \|(w_K - w) C^{1/2}\|_{\bmG_{\mathrm{FP}}} \leq \|(w_K - w_{\lambda}) C^{1/2}\|_{\bmG_{\mathrm{FP}}} + \|(w_{\lambda} - w) C^{1/2}\|_{\bmG_{\mathrm{FP}}},
\end{align}
defining $ w_\lambda = D C_{\lambda}^{-1}$.

\paragraph{1. Bound $\|(w_K - w_{\lambda}) C^{1/2}\|_{\bmG_{\mathrm{FP}}}$} Employing the same reasoning as in Lemma \ref{lem:error_2},  there exists a constant $c>0$ that does not depend on $K, \delta$, such that, for any $\delta \in (0, 1]$, with probability at least $1-\delta$
\begin{align}
    \|(w_K - w_\lambda) C^{1/2}\|_{\bmG_{\mathrm{FP}}}  \leq c \frac{\log \frac{K}{\delta}}{\sqrt{K}},
\end{align}
if $\lambda \geq \frac{9\kappa_{\tilde \phi}^2}{K}\log \frac{K}{\delta}$.

\paragraph{2. Bound $\|(w_{\lambda} - w) C^{1/2}\|_{\bmG_{\mathrm{FP}}}$.} From Assumption \ref{as:attain}, we have
\begin{align}
    D = wC.
\end{align}
Therefore,
\begin{align}
    \|(w_{\lambda} - w) C^{1/2}\|_{\bmG_{\mathrm{FP}}} = \|w(CC_{\lambda}^{-1} - I)C^{1/2}\|_{\bmG_{\mathrm{FP}}}
    = \lambda \|w C_{\lambda}^{-1}C^{1/2}\|_{\bmG_{\mathrm{FP}}}
    \leq \lambda^{1/2} \|w\|_{\bmG_{\mathrm{FP}}}.
\end{align}
     
\paragraph{Conclusion.} Combining the bounds, there exists a constant $c>0$ that does not depend on $K, \delta$, such that, for any $\delta \in (0, 1]$, with probability at least $1-\delta$
\begin{align}
    \|(w_K - w) C^{1/2}\|_{\bmG_{\mathrm{FP}}} \leq c \left(\frac{\log \frac{K}{\delta}}{\sqrt{K}} + \lambda^{1/2}\right),
\end{align}
if $\lambda \geq \frac{9\kappa_{\tilde \phi}^2}{K}\log \frac{K}{\delta}$.

\end{proof}

\subsection{Auxiliary Lemmas}\label{subsec:lemmas}

This section presents auxiliary lemmas used in the proofs of the main lemmas, along with their proofs.

This lemma bounds the $L^2$ discrepancy between each $\hat \phi_k(t, x)$ and $\tilde \phi(\theta_k, \cdot, \cdot)$ because of the use of $\hat p_k$ instead of $p(\theta_k)$ in the expression of $\hat \phi_k(t, x)$.
\begin{lemma}[Bound {$\E_{t,x}[\|\tilde \phi(\theta_k, t, x) - \hat \phi_k(t, x)\|_{\bmG_{\mathrm{FP}}}^2]$}]\label{lem:hat_phi} For any $k \in \llbracket 1, K\rrbracket$, the following inequality holds
\begin{align}
    \E_{t,x}\left[\left\|\tilde \phi(\theta_k, t, x) - \hat \phi_k(t, x)\right\|_{\bmG_{\mathrm{FP}}}^2\right]  \leq \kappa_{tot} \mathcal{E}(\hat p_k, p(\theta_k))^2.
\end{align}
with \(\mathcal{E}(p_1, p_2)^2 \triangleq \int_{t=0}^T \left(\left\|\frac{\partial p_1}{\partial t}(t, .) - \frac{\partial p_2}{\partial t}(t, .)\right\|^2_{L^2} + \left\|p_1(t, .) - p_2(t, .)\right\|^2_{H^2}\right)dt\),\\ and $\kappa_{tot} \triangleq \sum_{i,j=1}^n \left(\kappa_{\phi}^2 + \kappa_{\phi_i}^2 + \kappa_{\phi}^4 + \kappa_{\phi_{i}}^4+ \kappa_{\phi_{j}}^4 + \kappa_{\phi_{ij}}^4\right)$.
\end{lemma}
\begin{proof} We have
\begin{align}
    \|\tilde \phi(\theta_k, t, x) - \hat \phi_k(t, x)\|_{\bmG_{\mathrm{FP}}}^2 =& \sum_{i=1}^n \|\tilde \phi_i(\theta_k, t, x) - \hat \phi_i^k(t, x, u_{\theta_k}(t))\|_{\bmG_{\mathrm{FP}}}^2 \notag\\
    &+ \frac{1}{4}\sum_{i,j=1}^n \|\tilde \phi_{ij}(\theta_k, t, x) - \hat \phi_{ij}^k(t, x, u_{\theta_k}(t))\|_{\bmG_{\mathrm{FP}}}^2,
\end{align}
and
\begin{equation}
    \begin{split}
        \|\tilde \phi_i(\theta_k, t, x) - \hat \phi_i^k(t, x, u_{\theta_k}(t))\|_{\bmG_{\mathrm{FP}}}^2 &\leq 2\left\|\phi(t, x, u_{\theta_k}(t))\right\|_{\bmG_{\mathrm{FP}}}^2 \left|\frac{\partial p(\theta_k)}{\partial x_i}(t,x)- \frac{\partial \hat p_k(t,x)}{\partial x_i}\right|^2\\
        &\quad+  2\left\|\phi_i(t, x, u_{\theta_k}(t))\right\|_{\bmG_{\mathrm{FP}}}^2 \left|p(\theta_k)(t,x) - \hat p_k(t,x) \right|^2\\
        &\leq 2(\kappa_{\phi}^2 +  \kappa_{\phi_i}^2) \bigg( \left|\frac{\partial p(\theta_k)}{\partial x_i}(t,x)- \frac{\partial \hat p_k(t,x)}{\partial x_i}\right|^2\\
        &\quad+  \left|p(\theta_k)(t,x) - \hat p_k(t,x) \right|^2\bigg).
    \end{split}
\end{equation}

Then, one can obtain a similar bound for $\|\tilde \phi_{ij}(\theta_k, t, x) - \hat \phi_{ij}^k(t, x, u_{\theta_k}(t))\|_{\bmG_{\mathrm{FP}}}^2$ involving a sum of the difference of the second order derivatives with respect to $x$. Using these two bounds, we obtain
\begin{align}
    \E_{t,x}\left[\left\|\tilde \phi(\theta_k, t, x) - \hat \phi_k(t, x)\right\|_{\bmG_{\mathrm{FP}}}^2\right]  \leq \mathcal{E}(\hat p_k, p(\theta_k))^2 \times \sum_{i,j=1}^n \left(\kappa_{\phi}^2 + \kappa_{\phi_i}^2 + \kappa_{\phi}^4 + \kappa_{\phi_{i}}^4+ \kappa_{\phi_{j}}^4 + \kappa_{\phi_{ij}}^4\right).
\end{align}
\end{proof}

Lemma \ref{lem:hat_phi} allows us to derive the following lemmas, which bounds the two error terms appearing in the main lemmas' proofs, which stem from the estimation $\hat p_k$ of $p(\theta_k)$.
\begin{lemma}[Bound $\|\hat D  - \tilde D\|_{\bmG_{\mathrm{FP}}}$]\label{lem:hat_D} For any $\delta \in (0, 1]$, with probability at least $1-\delta$
\begin{align}
    \|\hat D  - \tilde D\|_{\bmG_{\mathrm{FP}}} &\leq c \log \frac{2}{\delta} (N^{-1} + \varepsilon N^{-1/2} + \varepsilon),
\end{align}
where $c>0$ is a constant that does not depend on $N, K, \delta$.
\end{lemma}
\begin{proof} We have
\begin{align}
    \hat D  - \tilde D &=  \frac{1}{KN}\sum_{i=1}^N \sum_{k=1}^K \left( \frac{\partial \hat p_k}{\partial t}(t_i, x_i) \hat \phi_k(t_i, x_i) - \frac{\partial p}{\partial t}(\theta_k, t_i, x_i) \tilde \phi(\theta_k, t_i, x_i)\right).
\end{align}
Hence, defining
\begin{align}
    &\hat E_1 \triangleq  \frac{1}{KN}\sum_{i=1}^N \sum_{k=1}^K \left( \frac{\partial \hat p_k}{\partial t}(t_i, x_i) -\frac{\partial p}{\partial t}(\theta_k, t_i, x_i)\right)\hat \phi_k(t_i, x_i),\\
    &\hat E_2 \triangleq  \frac{1}{KN}\sum_{i=1}^N \sum_{k=1}^K \frac{\partial p}{\partial t}(\theta_k, t_i, x_i)\left(\hat \phi_k(t_i, x_i) - \tilde \phi(\theta_k, t_i, x_i)\right),\\
    &E_1 \triangleq  \frac{1}{K}\sum_{k=1}^K \E_{t,x}\left[\left( \frac{\partial \hat p_k}{\partial t}(t, x) -\frac{\partial p}{\partial t}(\theta_k, t, x)\right)\hat \phi_k(t, x)\right],\\
    &E_2 \triangleq  \frac{1}{K}\sum_{k=1}^K\E_{t,x}\left[\frac{\partial p}{\partial t}(\theta_k, t, x)\left(\hat \phi_k(t, x) - \tilde \phi(\theta_k, t, x)\right)\right],
\end{align}
we have
\begin{align}
    \|\hat D  - \tilde D\|_{\bmG_{\mathrm{FP}}} &\leq \|\hat E_1 - E_1\|_{\bmG_{\mathrm{FP}}} + \|\hat E_2 - E_2\|_{\bmG_{\mathrm{FP}}} + \|E_1 \|_{\bmG_{\mathrm{FP}}} + \|E_2 \|_{\bmG_{\mathrm{FP}}}.
\end{align}

\paragraph{1. Bound $\|\hat E_1 - E_1\|_{\bmG_{\mathrm{FP}}}$.} To apply Proposition \ref{prop:mixed_cross_cov}, we define
\begin{align}
    &\psi(t, x, k) \triangleq \left( \frac{\partial \hat p_k}{\partial t}(t, x) -\frac{\partial p}{\partial t}(\theta_k, t, x)\right),\\
    &\chi(t,x, k) \triangleq \hat \phi_k(t, x),
\end{align}
and
\begin{align}
    &\kappa_{\psi} \triangleq \sup_{t,x, k} |\psi(t, x, k)| \leq \kappa_{\frac{\partial \hat p_k}{\partial t}} +  \kappa_{\frac{\partial p}{\partial t}},\\
    &\kappa_{\chi} \triangleq \sup_{t,x, k} \|\hat \phi_k(t, x)\|_{\bmG_{\mathrm{FP}}},
\end{align}
and
\begin{align}
     a &\triangleq \frac{1}{K}\sum_{k=1}^K \E_{t,x}[\|\psi(t,x, k) \otimes \chi(t,x, k)\|_{\hs}^2] \\
    &= \frac{1}{K}\sum_{k=1}^K \E_{t,x}\left[\left( \frac{\partial \hat p_k}{\partial t}(t, x) -\frac{\partial p}{\partial t}(\theta_k, t, x)\right)^2 \|\hat \phi_k(t, x)\|_{\bmG_{\mathrm{FP}}}^2\right]\\
    &\leq \frac{\kappa_{\chi}^2}{K}\sum_{k=1}^K \E_{t,x}\left[\left( \frac{\partial \hat p_k}{\partial t}(t, x) -\frac{\partial p}{\partial t}(\theta_k, t, x)\right)^2\right]\\
    &\leq \kappa_{\chi}^2 \sup_k \mathcal{E}(\hat p_k, p(\theta_k))^2.
\end{align}

Then, denoting $\varepsilon \triangleq \sup_k \mathcal{E}(\hat p_k, p(\theta_k))$, from Proposition \ref{prop:mixed_cross_cov}, we have, for any $\delta \in (0,1]$, with probability at least $1-\delta$,
\begin{align}
    \|\hat E_1 - E_1\|_{\bmG_{\mathrm{FP}}} &\leq \frac{16 (\kappa_{\frac{\partial \hat p_k}{\partial t}} +  \kappa_{\frac{\partial p}{\partial t}})\kappa_{\hat \phi}\log \frac{2}{\delta}}{N} + \sqrt{\frac{2 \varepsilon^2 \kappa_{\hat \phi}^2 \log \frac{2}{\delta}}{N}}.
\end{align}

\paragraph{2. Bound $\|\hat E_2 - E_2\|_{\bmG_{\mathrm{FP}}}$.} To apply Proposition \ref{prop:mixed_cross_cov}, we define
\begin{align}
    &\psi(t, x, k) \triangleq \frac{\partial p}{\partial t}(\theta_k, t, x),\\
    &\chi(t,x, k) \triangleq \left(\hat \phi_k(t, x) - \tilde \phi(\theta_k, t, x)\right),
\end{align}
and
\begin{align}
    &\kappa_{\psi} \triangleq \sup_{t,x, k} |\psi(t, x, k)| \leq \kappa_{\frac{\partial p}{\partial t}},\\
    &\kappa_{\chi} \triangleq \sup_{t,x, k} \|\hat \phi_k(t, x) - \tilde \phi(\theta_k, t, x)\|_{\bmG_{\mathrm{FP}}} \leq \kappa_{\hat \phi} + \kappa_{\tilde \phi},
\end{align}
and, from Lemma \ref{lem:hat_phi}, we obtain
\begin{align}
     a &\triangleq \frac{1}{K}\sum_{k=1}^K \E_{t,x}[\|\psi(t,x, k) \otimes \chi(t,x, k)\|_{\hs}^2] \\
    &= \frac{1}{K}\sum_{k=1}^K \E_{t,x}\left[\left(\frac{\partial p}{\partial t}(\theta_k, t, x)\right)^2 \|\hat \phi_k(t, x) - \tilde \phi(\theta_k, t, x)\|_{\bmG_{\mathrm{FP}}}^2\right]\\
    &\leq \frac{\kappa_{\frac{\partial p}{\partial t}}^2}{K}\sum_{k=1}^K \E_{t,x}\left[\|\hat \phi_k(t, x) - \tilde \phi(\theta_k, t, x)\|_{\bmG_{\mathrm{FP}}}^2\right]\\
    &\leq \kappa_{\frac{\partial p}{\partial t}}^2 \kappa_{tot} \sup_k \mathcal{E}(\hat p_k, p(\theta_k))^2.
\end{align}

Then, from Proposition \ref{prop:mixed_cross_cov}, we have, for any $\delta \in (0,1]$, with probability at least $1-\delta$,
\begin{align}
    \|\hat E_2 - E_2\|_{\bmG_{\mathrm{FP}}} &\leq \frac{16  \kappa_{\frac{\partial p}{\partial t}}(\kappa_{\hat \phi} + \kappa_{\tilde \phi}) \log \frac{2}{\delta}}{N} + \sqrt{\frac{2 \varepsilon^2 \kappa_{\frac{\partial p}{\partial t}}^2 \kappa_{tot}  \log \frac{2}{\delta}}{N}}.
\end{align}

\paragraph{3. Bound $\|E_1\|_{\bmG_{\mathrm{FP}}}$.} We have
\begin{align}
    \|E_1\|_{\bmG_{\mathrm{FP}}} &\leq  \frac{1}{K}\sum_{k=1}^K \E_{t,x}\left[\left\| \left(\frac{\partial \hat p_k}{\partial t}(t, x) -\frac{\partial p}{\partial t}(\theta_k, t, x)\right)\hat \phi_k(t, x)\right\|_{\bmG_{\mathrm{FP}}}\right]\\
    &\leq \kappa_{\hat \phi} \varepsilon.
\end{align}

\paragraph{4. Bound $\|E_2\|_{\bmG_{\mathrm{FP}}}$.} We have
\begin{align}
    \|E_2\|_{\bmG_{\mathrm{FP}}} &\leq \frac{1}{K}\sum_{k=1}^K\E_{t,x}\left[\left\|\frac{\partial p}{\partial t}(\theta_k, t, x)\left(\hat \phi_k(t, x) - \tilde \phi(\theta_k, t, x)\right)\right\|_{\bmG_{\mathrm{FP}}}\right]\\
    &\leq \kappa_{\frac{\partial p}{\partial t}} \kappa_{tot} \varepsilon,
\end{align}
where $c>0$
\end{proof}

\paragraph{Conclusion.} We conclude by combining all bounds that there exists a constant $c>0$ that does not depend on $N, K, \delta$, such that, for any $\delta \in (0, 1]$, with probability at least $1-\delta$
\begin{align}
    \|\hat D  - \tilde D\|_{\bmG_{\mathrm{FP}}} &\leq c \log \frac{2}{\delta} (N^{-1} + \varepsilon).
\end{align}

\begin{lemma}[Bound $\|\hat C  - \tilde C\|_{\bmG_{\mathrm{FP}}}$]\label{lem:hat_C} For any $\delta \in (0, 1]$, with probability at least $1-\delta$,
\begin{align}
    \|\hat C  - \tilde C\|_{\bmG_{\mathrm{FP}}} &\leq c \log \frac{2}{\delta} (N^{-1} + \varepsilon).
\end{align}
Moreover, if $\lambda \geq 2c \log \frac{2}{\delta} (N^{-1} + \varepsilon)$, 
\begin{align}
    \|\hat C_{\lambda}^{-1/2}\tilde C^{1/2}\|_{\infty} \leq \sqrt{2},
\end{align}
where $c>0$ is a constant that does not depend on $N, K, \delta$.
\end{lemma}
\begin{proof} Defining the operator $B_n = \tilde C_{\lambda}^{-1/2}(\tilde C - \hat C)\tilde C_{\lambda}^{-1/2}$, if $\|B_n\|_{\infty}<1$, we have
\begin{align}
     \|\hat C_{\lambda}^{-1/2}\tilde C^{1/2}\|_{\infty} &= \|\tilde C^{1/2}\hat C_{\lambda}^{-1}\tilde C^{1/2}\|_{\infty}^{1/2}\\
     &\leq \|\tilde C_{\lambda}^{1/2}\hat  C_{\lambda}^{-1}\tilde C_{\lambda}^{1/2}\|_{\infty}^{1/2}\\
     &=\|(I-B_n)^{-1}\|_{\infty}^{1/2}\\
     &\leq (1- \|B_n\|_{\infty})^{-1/2}.
\end{align}

Moreover, we have
\begin{align}
    \hat C  - \tilde C &=  \frac{1}{KN}\sum_{i=1}^N \sum_{k=1}^K \left( \hat \phi_k(t_i, x_i) \otimes  \hat \phi_k(t_i, x_i) - \tilde \phi(\theta_k, t_i, x_i) \otimes \tilde \phi(\theta_k, t_i, x_i)\right).
\end{align}
Hence, defining
\begin{align}
    &\hat E_1 \triangleq  \frac{1}{KN}\sum_{i=1}^N \sum_{k=1}^K \left( \hat \phi_k(t_i, x_i) - \tilde \phi(\theta_k, t_i, x_i)\right)\hat \phi_k(t_i, x_i),\\
    &\hat E_2 \triangleq  \frac{1}{KN}\sum_{i=1}^N \sum_{k=1}^K \tilde \phi(\theta_k, t_i, x_i)\left(\hat \phi_k(t_i, x_i) - \tilde \phi(\theta_k, t_i, x_i)\right),\\
    &E_1 \triangleq  \frac{1}{K}\sum_{k=1}^K \E_{t,x}\left[\left(\hat \phi_k(t, x) - \tilde \phi(\theta_k, t, x)\right)\hat \phi_k(t, x)\right],\\
    &E_2 \triangleq  \frac{1}{K}\sum_{k=1}^K\E_{t,x}\left[\tilde \phi(\theta_k, t, x)\left(\hat \phi_k(t, x) - \tilde \phi(\theta_k, t, x)\right)\right],
\end{align}
we have
\begin{align}
    \|\hat C  - \tilde C\|_{\bmG_{\mathrm{FP}}} &\leq \|\hat E_1 - E_1\|_{\bmG_{\mathrm{FP}}} + \|\hat E_2 - E_2\|_{\bmG_{\mathrm{FP}}} + \|E_1 \|_{\bmG_{\mathrm{FP}}} + \|E_2 \|_{\bmG_{\mathrm{FP}}}.
\end{align}

Then, following the same proof as for Lemma \ref{lem:hat_D} leads, for any $\delta \in (0, 1]$, with probability at least $1-\delta$
\begin{align}
    \|B_n\|_{\bmG_{\mathrm{FP}}} &\leq c \log \frac{2}{\delta} (N^{-1} + \varepsilon) \lambda^{-1},
\end{align}
where $c>0$ is a constant that does not depend on $N, K, \delta$.

Hence, if $\lambda \geq 2c \log \frac{2}{\delta} (N^{-1} + \varepsilon)$, we have
\begin{align}
    \|\hat C_{\lambda}^{-1/2}\tilde C^{1/2}\|_{\infty} \leq \sqrt{2}.
\end{align}
    
\end{proof}

\begin{lemma}[Bound $\|\tilde C_{\lambda}^{-1/2}C^{1/2}\|_{\infty}$]\label{lem:sampling} For any $\delta \in (0,1]$, if $\lambda \geq \frac{9\kappa_{\tilde \phi}^2}{N}\log \frac{N}{\delta}$ and $\lambda \geq \frac{9\kappa_{\tilde \phi}^2}{K}\log \frac{K}{\delta}$, with probability at least $1-\delta$,
\begin{align}
     \|\tilde C_{\lambda}^{-1/2}C^{1/2}\|_{\infty} &\leq 2.
\end{align}
\end{lemma}
\begin{proof} We define
\begin{align}
    C_K = \frac{1}{K}\sum_{k=1}^K \E_{t,x}\left[\tilde \phi(\theta_k, t, x) \otimes \tilde \phi(\theta_k, t, x)\right].
\end{align}
We have
\begin{align}
         \|\tilde C_{\lambda}^{-1/2}C^{1/2}\|_{\infty} &\leq \|\tilde C_{\lambda}^{-1/2}C_{K, \lambda}^{1/2}\|_{\infty} \|C_{K, \lambda}^{-1/2}C^{1/2}\|_{\infty}.
\end{align}
Therefore, applying Lemma \ref{lem:mixed_cov} two times gives us, for any $\delta \in (0,1]$, if $\lambda \geq \frac{9\kappa_{\tilde \phi}^2}{N}\log \frac{N}{\delta}$ and $\lambda \geq \frac{9\kappa_{\tilde \phi}^2}{K}\log \frac{K}{\delta}$, then with probability at least $1-\delta$,
\begin{align}
     \|\tilde C_{\lambda}^{-1/2}C^{1/2}\|_{\infty} &\leq 2,
\end{align}
with $\kappa_{\tilde \phi} \triangleq \sup_{(\theta , t, x ) \in \Theta \times [0, T] \times \R^n}\|\tilde \phi(\theta, t, x)\|$.
\end{proof}

The following lemma provides bounds for all supremum norms involved in our proofs.

\begin{lemma}[Bound all \(\kappa\)'s]\label{lem:kappa}
Recall that, for any set \(\bmX\), any Hilbert space \(\bmH\), and any
bounded map \(f:\bmX\to\bmH\), we write
\[
\kappa_f \triangleq \sup_{x\in\bmX}\|f(x)\|_{\bmH}.
\]
The following constants are bounded by a constant that does not depend on
\(N\) or \(K\):
\[
\kappa_{\frac{\partial p}{\partial t}},
\qquad
(\kappa_{\frac{\partial \widehat p_k}{\partial t}})_{k=1}^K,
\qquad
\kappa_{\widetilde \phi},
\qquad
(\kappa_{\widehat \phi^k})_{k=1}^K,
\qquad
\kappa_{\phi},
\qquad
(\kappa_{\phi_i})_{i=1}^n,
\qquad
(\kappa_{\phi_{ij}})_{i,j=1}^n .
\]
\end{lemma}

\begin{proof}
We first bound \(\kappa_{\widetilde\phi}\). By definition of
\(\widetilde\phi\),
\[
\begin{aligned}
\kappa_{\widetilde \phi}
&\le
\sum_{i=1}^n
\left(
\kappa_{\phi}\kappa_{\partial_{x_i}p}
+
\kappa_{\phi_i}\kappa_p
\right)
\\
&\quad+
\frac12
\sum_{i,j=1}^n
\left(
\kappa_{\phi\otimes\phi}\kappa_{\partial_{x_i x_j}p}
+
\kappa_{(\phi\otimes\phi)_i}\kappa_{\partial_{x_j}p}
+
\kappa_{(\phi\otimes\phi)_j}\kappa_{\partial_{x_i}p}
+
\kappa_{(\phi\otimes\phi)_{ij}}\kappa_p
\right).
\end{aligned}
\]
The same estimate holds for \(\widehat\phi^k\), replacing
\(p(\theta_k)\) by \(\widehat p_k\).

The quantities involving only the feature map,
\[
\kappa_\phi,\qquad
(\kappa_{\phi_i})_{i=1}^n,\qquad
(\kappa_{\phi_{ij}})_{i,j=1}^n,
\]
are finite by boundedness of the kernel and of the derivatives entering the
Fokker--Planck features on \([0,T]\times D\times V\).

The quantities involving $p$ are finite by the density-flow regularity
assumptions. By construction, the estimators $\widehat p_k$ belong uniformly
to the same bounded regularity class, so their associated $\kappa$-constants
are bounded independently of $k$, $K$, and $N$. Taking the maximum over $k$
therefore preserves a uniform bound, which proves the claim.
\end{proof}

\begin{lemma}[Uniform boundedness of the estimated coefficients]
\label{lem:uniform_estimator_bound}
Let \(\widehat w_S\) be the regularized empirical FP estimator defined in
\eqref{eq:w_s}, and let \((\widehat b,\widehat a)\), with
\(\widehat a=\kappa I_n+\widehat a_0\), be the corresponding estimated
coefficients. Assume that the assumptions, choice of \(\lambda\), and
density-estimation condition of either Theorem~\ref{thm:lr} or
Theorem~\ref{th:refined} hold. Then there exists a constant
\(R_\delta>0\), independent of \(N\) and \(K\), such that, with probability
at least \(1-\delta\),
\[
\|\widehat b\|_{H^{s_0}([0,T]\times D\times V)}
+
\|\widehat a_0\|_{H^{s_0}([0,T]\times D\times V)}
\le R_\delta .
\]
Consequently, the constant \(C_{\mathrm{FP}}\) in
Lemma~\ref{lem:FP_ineq} is bounded independently of \(N\) and \(K\), and may
be absorbed into the constants of the learning-rate bounds.
\end{lemma}
\begin{proof}
By Assumption~\ref{as:attain}, there exists
\(w\in\bmS_{\mathrm{FP}}\) representing the true coefficients. Set
\[
e\triangleq \widehat w_S-w.
\]
Since \(\bmS_{\mathrm{FP}}\) is convex, the first-order optimality condition
for \(\widehat w_S\) in \eqref{eq:w_s}, applied with the feasible point \(w\),
gives
\[
\left\langle
\widehat w_S\widehat C+\lambda\widehat w_S-\widehat D,
e
\right\rangle_{\bmG_{\mathrm{FP}}}
\le 0.
\]
Using \(\widehat w_S=w+e\), we obtain
\[
\langle e\widehat C,e\rangle_{\bmG_{\mathrm{FP}}}
+
\lambda\|e\|_{\bmG_{\mathrm{FP}}}^2
\le
\left\langle
\widehat D-w\widehat C-\lambda w,
e
\right\rangle_{\bmG_{\mathrm{FP}}}.
\]
Since \(\widehat C\) is positive semidefinite, it follows that
\[
\lambda\|e\|_{\bmG_{\mathrm{FP}}}
\le
\|\widehat D-w\widehat C\|_{\bmG_{\mathrm{FP}}}
+
\lambda\|w\|_{\bmG_{\mathrm{FP}}}.
\]

Because \(w\) represents the true coefficients,
\[
\partial_t p(\theta_k,t,x)
=
\langle w,\widetilde\phi(\theta_k,t,x)\rangle_{\bmG_{\mathrm{FP}}},
\]
and hence
\[
\widetilde D=w\widetilde C.
\]
Therefore,
\[
\widehat D-w\widehat C
=
(\widehat D-\widetilde D)
+
w(\widetilde C-\widehat C),
\]
so that
\[
\|\widehat D-w\widehat C\|_{\bmG_{\mathrm{FP}}}
\le
\|\widehat D-\widetilde D\|_{\bmG_{\mathrm{FP}}}
+
\|w\|_{\bmG_{\mathrm{FP}}}
\|\widehat C-\widetilde C\|_{\infty}.
\]
Consequently,
\[
\|\widehat w_S\|_{\bmG_{\mathrm{FP}}}
\le
2\|w\|_{\bmG_{\mathrm{FP}}}
+
\frac{
\|\widehat D-\widetilde D\|_{\bmG_{\mathrm{FP}}}
+
\|w\|_{\bmG_{\mathrm{FP}}}
\|\widehat C-\widetilde C\|_{\infty}
}{\lambda}.
\]

Apply Lemmas~\ref{lem:hat_D} and~\ref{lem:hat_C} with confidence parameter
\(\delta/2\). Under the assumptions of Theorem~\ref{thm:lr}, their bounds and
the choice of \(\lambda\) imply
\[
\|\widehat D-\widetilde D\|_{\bmG_{\mathrm{FP}}}
+
\|w\|_{\bmG_{\mathrm{FP}}}
\|\widehat C-\widetilde C\|_{\infty}
\le c_\delta\lambda.
\]
Under the assumptions of Theorem~\ref{th:refined}, the same conclusion follows
from Lemmas~\ref{lem:hat_D_refined} and~\ref{lem:hat_C_refined}. Thus, with
probability at least \(1-\delta\),
\[
\|\widehat w_S\|_{\bmG_{\mathrm{FP}}}
\le
2\|w\|_{\bmG_{\mathrm{FP}}}+c_\delta
\triangleq B_\delta,
\]
where \(B_\delta\) is independent of \(N\) and \(K\).

By the RKHS-to-Sobolev embedding used in the coefficient parametrization,
there exists \(C_{s_0}>0\) such that
\[
\|\widehat b\|_{H^{s_0}([0,T]\times D\times V)}
+
\|\widehat a_0\|_{H^{s_0}([0,T]\times D\times V)}
\le
C_{s_0}\|\widehat w_S\|_{\bmG_{\mathrm{FP}}}.
\]
Therefore,
\[
\|\widehat b\|_{H^{s_0}([0,T]\times D\times V)}
+
\|\widehat a_0\|_{H^{s_0}([0,T]\times D\times V)}
\le
R_\delta,
\qquad
R_\delta\triangleq C_{s_0} B_\delta,
\]
where \(R_\delta\) is independent of \(N\) and \(K\).

Finally, Lemma~\ref{lem:FP_ineq}, applied on the Sobolev ball of radius
\(R_\delta\), yields
\[
C_{\mathrm{FP}}(\widehat b,\widehat a)
\le
C_{\mathrm{FP}}(R_\delta).
\]
Hence the FP stability constant is bounded independently of \(N\) and \(K\)
and may be absorbed into the constants of the learning-rate bounds.
\end{proof}

\subsection{Concentration Inequalities}\label{subsec:concentration}

In this section, we provide the concentration inequalities used in the lemmas' proofs.

The following inequality is essentially a restatement of Proposition 2 of \citet{rudi2017generalization}.
\begin{proposition}[Bernstein's inequality for sum of random vectors]\label{prop:bernstein_hs} Let $\bmZ$ be a Polish space, $M: \bmZ \to \bmH$ be bounded maps with values in a separable Hilbert space $\bmH$, such that $\kappa_{M} \triangleq \sup_{z \in \bmZ}\|M(z)\|_{\bmH}$. Let $z_1, \dots, z_N$ be a sequence of independent and identically distributed random vectors. Then, for any $\delta \in (0,1]$, with probability at least $1-\delta$,
\begin{align}
    \left\|\frac{1}{N}\sum_{i=1}^N M(z_i)- \E[M(z)]\right\|_{\hs} &\leq \frac{16 \kappa_{M}\log \frac{2}{\delta}}{N} + \sqrt{\frac{2 a\log \frac{2}{\delta}}{N}},
\end{align}
with $ a = \E[\|M(z)\|_{\hs}^2]$.
\end{proposition}

\begin{proof} For $p \in \N^*$, we have
\begin{align}
    \E[\|M_i(z)- \E[M_i(z)]\|^p_{\hs}] &\leq \E_{z,z'}[\|M_i(z)- M_i(z')\|^p_{\hs}]\\
    &\leq  \E_{z,z'}[(\|M_i(z)\|_{\hs} + \| M_i(z')\|)^p_{\hs}]\\
    &\leq 2^{p}\E_{z}[\|M_i(z)\|_{\hs}^p]\\
    &\leq  2^{p}\kappa_{M}^{p-2}\E[\|M(z)\|_{\bmH}^2],
\end{align}
using $(a + b)^p \leq 2^{p-1}(a^p + b^p)$ for $a,b>0, p \in \N^*$.

One can conclude that, for $p \geq 2$,
\begin{align}
    \E[\|M_i(z)- \E[M_i(z)]\|^p_{\hs}] \leq \frac{p!}{2}   a M^{p-2},
\end{align}
with $ a = \E[\|\psi(z) \otimes \phi(z)\|_{\hs}^2]$, and $M=8\kappa_{M}$.

Then, application of Proposition 2 of \citet{rudi2017generalization} gives us 
\begin{align}
    \left\|\frac{1}{N}\sum_{i=1}^N M(z_i) - \E[M(z)]\right\|_{\hs} &\leq \frac{2M\log \frac{2}{\delta}}{n} + \sqrt{\frac{2 a\log \frac{2}{\delta}}{n}},
\end{align}
with probability at least $1-\delta$.
\end{proof}

From Proposition \ref{prop:bernstein_hs}, we deduce the following inequality.
\begin{proposition}[Covariance estimation, $\|.\|_{\hs}$]\label{prop:cov} Let $\bmZ$ be a Polish space, $\psi, \phi: \bmZ \to \bmH$ be bounded maps with values in a separable Hilbert space $\bmH$, such that $\kappa_{\psi} \triangleq \sup_{z \in \bmZ}\|\psi(z)\|,\kappa_{\phi} \triangleq \sup_{z \in \bmZ}\|\phi(z)\| $. Let $z_1, \dots, z_N$ be a sequence of independent and identically distributed random vectors. Then, for any $\delta \in (0,1]$, with probability at least $1-\delta$,
\begin{align}
    \left\|\frac{1}{N}\sum_{i=1}^N \psi(z_i) \otimes \phi(z_i) - \E[\psi(z) \otimes \phi(z)]\right\|_{\hs} &\leq \frac{2M\log \frac{2}{\delta}}{N} + \sqrt{\frac{2 a\log \frac{2}{\delta}}{N}},
\end{align}
with $ a = \E[\|\psi(z) \otimes \phi(z)\|_{\hs}^2]$, and $M=8\kappa_{\psi}\kappa_{\phi}$.
\end{proposition}

\begin{proof} For $i\in\llbracket 1, N\rrbracket$, we consider the random Hilbert-Schmidt operators
\begin{align}
    M_i(z) = \psi(z_i) \otimes \phi(z_i),
\end{align}
with mean
\begin{align}
    \E[M_i(z)] = \E[\psi(z) \otimes \phi(z)],
\end{align}
and such that
\begin{align}
    \|M_i(z)\|_{\hs} = \|\psi(z_i) \otimes \phi(z_i)\|_{\hs} =  \|\psi(z_i)\|_{\bmH} \|\phi(z_i)\|_{\bmH} \leq \kappa_{\psi}\kappa_{\phi},
\end{align}
and also
\begin{align}
    \E_{z}[\|M_i(z)\|_{\hs}^2] = \E[\|\psi(z) \otimes \phi(z)\|_{\hs}^2].
\end{align}

Then, application of Proposition \ref{prop:bernstein_hs} gives us the desired bound.
\end{proof}

The following proposition adapts Proposition \ref{prop:cov} for our purposes.
\begin{proposition}[Mixed covariance estimation, $\|.\|_{\hs}$]\label{prop:mixed_cross_cov} Let $\bmX, \bmY$ be Polish spaces, $\psi, \phi: \bmX \times \bmY \to \bmH$ be bounded maps with values in a separable Hilbert space $\bmH$, such that $\kappa_{\psi} \triangleq \sup_{(x,y) \in \bmX \times \bmY} \|\psi(x, y)\|,\kappa_{\phi} \triangleq \sup_{(x,y) \in \bmX \times \bmY}\|\phi(x,y)\| $. Let $x_1, \dots, y_N$ be a sequence of independent and identically distributed random vectors. Then, for any $\delta \in (0,1]$, with probability at least $1-\delta$,
\begin{align}
    \left\|\frac{1}{N}\sum_{i=1}^N \E_{y|x_i}[\psi(x_i, y) \otimes \phi(x_i, y)] - \E_{x,y}[\psi(x, y) \otimes \phi(x, y)]\right\|_{\hs} &\leq \frac{2M\log \frac{2}{\delta}}{N} + \sqrt{\frac{2 a\log \frac{2}{\delta}}{N}},
\end{align}
with $ a = \E_{x,y}[\|\psi(x,y) \otimes \phi(x,y)\|_{\hs}^2]$, and $M=8\kappa_{\psi}\kappa_{\phi}$.
\end{proposition}
\begin{proof} Same proof as Proposition \ref{prop:cov} but defining $M_i = \E_y[\psi(x,y) \otimes \phi(x,y)]$ instead of $M_i = \psi(x) \otimes \phi(x)$.
\end{proof}

The following proposition adapts Lemma 3.6 in \citet{rudi2013sample} to handle "mixed" covariance estimation.
\begin{proposition}[Mixed covariance estimation, $\|.\|_{\infty}$]\label{lem:mixed_cov} Let $\bmX, \bmY$ be Polish spaces, $\phi: \bmX \times \bmY \to \bmH$ a bounded maps with values in a separable Hilbert space $\bmH$ such that $\kappa_{\phi} \triangleq \sup_{(x,y) \in \bmX \times \bmY}\|\phi(x,y)\|$. Let $x_1, \dots, x_N$ be a sequence of independent and identically distributed random vectors. Let
\begin{align}
    &C_N = \frac{1}{N}\sum_{i=1}^N \E_{y|x_i}[\phi(x_i, y) \otimes \phi(x_i, y)],\\
    &C = \E_{x,y}[\phi(x, y) \otimes \phi(x, y)],
\end{align}
Then, for any $\delta \in (0,1]$, with probability at least $1-\delta$,
\begin{align}
    \|C_{\lambda}^{-1/2}(C - C_N)C_{\lambda}^{-1/2}\|_{\infty} \leq \frac{2\beta \kappa_{\phi}^2}{\lambda N} + \sqrt{\frac{2\kappa_{\phi}^2\beta}{3\lambda N}},
\end{align}
with $\beta = \log \frac{4 (\|C\|_{\infty} + \lambda) \Tr(C)}{\delta \lambda \|C\|_{\infty}}$.

In particular, if $\lambda \geq \frac{18 \kappa^2_{\phi}}{N}\log \frac{N}{\delta}$,
\begin{align}
    \|C_{N, \lambda}^{-1/2} C_{\lambda}^{1/2}\|_{\infty} \leq \sqrt{2},
\end{align}
and 
\begin{align}
    \|C_{\lambda}^{-1/2} C_{N, \lambda}^{1/2}\|_{\infty} \leq \sqrt{2}.
\end{align}

\end{proposition}
\begin{proof} Same proof as Lemma 3.6 of \citet{rudi2013sample}, but defining 
\begin{align}
    Z_i \triangleq C_{\lambda}^{-1/2}\E_{y|x_i}[\phi(x_i, y) \otimes \phi(x_i, y)]C_{\lambda}^{-1/2},
\end{align}
instead of $Z_i \triangleq C_{\lambda}^{-1/2}\phi(x_i) \otimes C_{\lambda}^{-1/2}\phi(x_i)$. 

More precisely, defining $U_i \triangleq  C_{\lambda}^{-1/2}\phi(x_i, y)$, we have $Z_i = \E_{y|x_i}[U_i \otimes U_i]$ instead of $Z_i = U_i \otimes U_i$. 

Moreover, notice that
\begin{align}
    \|Z\|_{\infty} &\leq \lambda^{-1} \|\E_{y|x_i}[\phi(x_i, y) \otimes \phi(x_i, y)]\|_{\infty}\\
    &\leq \frac{\kappa_{\phi}^2}{\lambda},
\end{align}
and
\begin{align}
    T \triangleq \E_{x_i}[Z] = CC_{\lambda}^{-1},
\end{align}
and also
\begin{align}
    \E[(Z - T)^2] = \E[Z^2] - T^2 \preceq \E[Z^2] \preceq \E_{x,y}[\|U\|^2_{\bmH} U \otimes U] \preceq \kappa_{\phi}^2\lambda^{-1}T,
\end{align}
using the inequality $\E(M)^2 \preceq \E(M^2)$ for a random variable $M$ with values in the space of bounded self-adjoint operators.
    
\end{proof}

\subsection{Proof of Theorem \ref{th:refined}}\label{subsec:proof_refined_L2}

The following theorem establishes refined learning rates for the proposed estimator, showing that for regular problems, the rates of FP matching can improve at arbitrarily fast polynomial rates. However, this gain is constrained by the accuracy of the $\hat p_k$. In particular, the error remains greater than the $\mathcal{E}_{\infty}$ error of any $\hat p_k$.
\begin{theorem}[Refined \(L^2\) learning rates] Under the assumptions of Theorem \ref{thm:lr}, Assumption \ref{as:source}, and Assumption \ref{as:emb}, there exist constants \(c_1, c_2 > 0\) that do not depend on $N, K, \delta$, such that for any \(\delta \in (0,1]\), defining \(\varepsilon_{\infty} \triangleq \sup_{k \in \llbracket 1,\, K\rrbracket} \mathcal{E}_{\infty}(\hat p_k, p(\theta_k))\) where \(\mathcal{E}_{\infty}\) is a Sobolev-type error defined as
\begin{align}
   \mathcal{E}_{\infty}(p_1, p_2)^2 \triangleq \int_{t=0}^T \left\|\frac{\partial p_1}{\partial t}(t, .) - \frac{\partial p_2}{\partial t}(t, .)\right\|^2_{L^\infty(\R^n)} + \left\|p_1(t, .) - p_2(t, .)\right\|^2_{W^{2,\infty}(\R^n)}dt,
\end{align}
taking \(\lambda = c_2 \left(\left(\frac{18c}{N}\log^2 \frac{N}{\delta}\right)^{\frac{1}{1-r}} + \left(\frac{18c}{K}\log^2 \frac{K}{\delta}\right)^{\frac{1}{1-s}}\right)\), if 
\begin{align}
    \varepsilon +  N^{-1}\varepsilon_\infty \leq \left(\frac{\log \frac{N}{\delta}}{\sqrt{N}}\right)^{\frac{2+\alpha}{1-r}} + \left(\frac{\log \frac{K}{\delta}}{\sqrt{K}}\right)^{\frac{2+\alpha}{1-s}},
\end{align}
then with probability \(1-\delta\),
\begin{align}
    \|p_{\hat b, \hat a} - p\|_{L^2(\mathbb P_c)L^2([0, T] \times \R^{n})} &\leq c_1 \log \frac{2}{\delta} \left(\frac{\log \frac{K \wedge N}{\delta}}{\sqrt{K \wedge N}}\right)^{\frac{1+\alpha}{1-r \wedge s}}.
\end{align}
\end{theorem}
\begin{proof} The definitions of Section \ref{subsec:preli} allow for the following decomposition
\begin{align}
    \|(\hat w - w) C^{1/2}\|_{\bmG_{\mathrm{FP}}} \leq \|(\hat w - \tilde w) C^{1/2}\|_{\bmG_{\mathrm{FP}}}+  \|(\tilde w - w_K) C^{1/2}\|_{\bmG_{\mathrm{FP}}} + \|(w_K - w) C^{1/2}\|_{\bmG_{\mathrm{FP}}}.
\end{align}

Using Lemma\ref{lem:cone}, refined Lemmas \ref{lem:error_1_refined}, \ref{lem:error_2_refined}, and \ref{lem:error_3_refined}, and Lemma \ref{lem:kappa}, we determine that for any \(\delta \in (0,1]\), if \(\lambda \geq c_2 \log \frac{2}{\delta} (N^{-1}\varepsilon_{\infty} + \varepsilon)\), \(\lambda^{1-r} \geq \frac{18c}{N}\log \frac{N}{\delta}\), and \(\lambda^{1-s} \geq \frac{18c}{K}\log \frac{K}{\delta}\), with probability \(1-\delta\),
\begin{align}
    \text{FP}(\hat b, \hat  a) \leq c_1 \left((\varepsilon_{\infty}N^{-1} + \varepsilon) \lambda^{-1/2} \log \frac{2}{\delta}  + \frac{\log \frac{N}{\delta}}{\sqrt{N}}\lambda^{(\alpha + r)/2} + \frac{\log \frac{K}{\delta}}{\sqrt{K}}\lambda^{(\alpha + s)/2} + \lambda^{(1+\alpha)/2}\right),
\end{align}
where constants \(c_1, c_2 > 0\) are independent of \(N\), \(K\), and \(\delta\). Therefore, setting \(\lambda = \left(\frac{18c}{N}\log^2 \frac{N}{\delta}\right)^{\frac{1}{1-r}} + \left(\frac{18c}{K}\log^2 \frac{K}{\delta}\right)^{\frac{1}{1-s}}\), we find
\begin{align}
    \text{FP}(\hat b, \hat  a) &\leq c_1 \log \frac{2}{\delta} \Bigg(\frac{\varepsilon + N^{-1}\varepsilon_\infty}{\lambda^{1/2}} + \left(\frac{\log \frac{N}{\delta}}{\sqrt{N}}\right)^{\frac{1+\alpha}{1-r}} + \left(\frac{\log \frac{K}{\delta}}{\sqrt{K}}\right)^{\frac{1+\alpha}{1-s}} \\
    &\quad+ \left(\frac{\log \frac{N}{\delta}}{\sqrt{N}}\right) \left(\frac{\log \frac{K}{\delta}}{\sqrt{K}}\right)^{\frac{\alpha + r}{1-s}} + \left(\frac{\log \frac{K}{\delta}}{\sqrt{K}}\right)\left(\frac{\log \frac{N}{\delta}}{\sqrt{N}}\right)^{\frac{\alpha + s}{1-r}}\Bigg).
\end{align}
up to overloading $c_1>0$. Assuming $\varepsilon +  N^{-1}\varepsilon_\infty \leq \lambda^{1+\alpha/2}$, the condition $\lambda \geq c_2 \log \frac{2}{\delta} (N^{-1}\varepsilon_{\infty} + \varepsilon)$ holds true up to multiplying $\lambda$ by a constant, and overloading the constants $c_1$.

Moreover, notice that defining \(\gamma = r \wedge s\), we have \(\frac{\alpha + r}{1-s} \geq \frac{\alpha + \gamma}{1-\gamma}\), \(\frac{\alpha + s}{1-r} \geq \frac{\alpha + \gamma}{1-r}\), \(\frac{\alpha + 1}{1-r} \geq \frac{\alpha + 1}{1-\gamma}\), and \(\frac{\alpha + 1}{1-s} \geq \frac{\alpha + 1}{1-\gamma}\).

Then, defining \(P = K \wedge N\), we have 
\begin{align}
    &\left(\frac{\log \frac{N}{\delta}}{\sqrt{N}}\right) \left(\frac{\log \frac{K}{\delta}}{\sqrt{K}}\right)^{\frac{\alpha + r}{1-s}} \leq \left(\frac{\log \frac{P}{\delta}}{\sqrt{P}}\right)^{\frac{1 + \alpha}{1-\gamma}},\\
    &\left(\frac{\log \frac{K}{\delta}}{\sqrt{K}}\right)\left(\frac{\log \frac{N}{\delta}}{\sqrt{N}}\right)^{\frac{\alpha + s}{1-r}} \leq \left(\frac{\log \frac{P}{\delta}}{\sqrt{P}}\right)^{\frac{1 + \alpha}{1-\gamma}}.
\end{align}

Finally, Lemma \ref{lem:FP_ineq} completes the proof.
\end{proof}

\subsection{Proof of Refined Lemmas}\label{subsec:proof_ref_lem}

In this section, we present refined versions of the lemmas previously utilized to establish the non-refined \(L^2\) learning rate. To maintain clarity in our presentation, we do not restate the assumptions for each lemma if they are unchanged from their original, non-refined versions. The main ingredient for these proofs involves utilizing Assumption \ref{as:source} and Assumption \ref{as:emb}, and performing derivations such as:
\begin{align}\label{eq:ref}
C_{\lambda}^{-1}C^{a} \preccurlyeq  \lambda^{-(1-a)} I
\end{align}
to establish refined dependencies in \(\lambda\). Furthermore, to avoid redundancy and streamline the presentation, we directly state the results without detailing the proofs when possible, focusing on clarifying the refinements made from their non-refined counterparts.

\begin{proposition}[Refined Lemma \ref{lem:mixed_cov}]\label{lem:mixed_cov_refined} Let
\begin{align}
&\forall x \in \bmX,\, C(x) = \E_{y|x}[\phi(x, y) \otimes \phi(x, y)].
\end{align}
Assume that there exist constants $r \geq 0$ and $c > 0$ such that
\begin{align}
\|C^{-r/2} C(x)^{1/2} \|_{\infty} < c \quad \text{almost surely}.
\end{align}
Then, for any $\delta \in (0,1]$, with probability at least $1-\delta$:
\begin{align}
\|C_{\lambda}^{-1/2}(C - C_N)C_{\lambda}^{-1/2}\|_{\infty} \leq \frac{2\beta c}{3\lambda^{1-r} N} + \sqrt{\frac{2c\beta}{\lambda^{1-r} N}},
\end{align}
where $\beta = \log \frac{4 (\|C\|_{\infty} + \lambda) \Tr(C)}{\delta \lambda \|C\|_{\infty}}$.

In particular, if $\lambda \geq \left(\frac{18 c}{N}\log \frac{N}{\delta}\right)^{\frac{1}{1-r}}$,
\begin{align}
    \|C_{N, \lambda}^{-1/2} C_{\lambda}^{1/2}\|_{\infty} \leq \sqrt{2},
\end{align}
and 
\begin{align}
    \|C_{\lambda}^{-1/2} C_{N, \lambda}^{1/2}\|_{\infty} \leq \sqrt{2}.
\end{align}
\end{proposition}
\begin{proof} We have
\begin{align}
    \|Z\|_{\infty} &= \|C_{\lambda}^{-1/2} C(x)^{1/2}\|_{\infty}\\
    &\leq \|C_{\lambda}^{-1/2} C_{\lambda}^{r/2}\|_{\infty}^2 \|C_{\lambda}^{-r/2} C(x)^{1/2}\|_{\infty}^2\\
    &\leq c \lambda^{-(1-r)},
\end{align}
and
\begin{align}
    \E[(Z - T)^2] = \E[Z^2] - T^2 \preceq \E[Z^2] \preceq \E_{x,y}[\|U\|^2_{\bmH} U \otimes U] \preceq c \lambda^{-(1-r)}T.
\end{align}
\end{proof}

\begin{lemma}[Refined Lemma \ref{lem:error_1}]\label{lem:error_1_refined} There exist constants $c_1, c_2>0$ that do not depend on $N, K, \delta$, such that, for any $\delta \in (0, 1]$, with probability at least $1-\delta$
\begin{align}
    \|(\hat w - \tilde w) C^{1/2}\|_{\bmG_{\mathrm{FP}}} \leq c_1\log \frac{2}{\delta} (N^{-1}\varepsilon_\infty + \varepsilon) \lambda^{-1/2},
\end{align}
provided that $\lambda \geq c_2 \log \frac{2}{\delta} (N^{-1}\varepsilon_\infty + \varepsilon)$, $\lambda \geq \frac{18c}{N}\log \frac{N}{\delta}$ and $\lambda \geq \frac{18c}{K}\log \frac{K}{\delta}$.
\end{lemma}

\begin{lemma}[Refined Lemma \ref{lem:error_2}]\label{lem:error_2_refined} There exists a constant $c>0$ that does not depend on $N, K, \delta$, such that, for any $\delta \in (0, 1]$, with probability at least $1-\delta$
\begin{align}
    \|(\tilde w - w_K) C^{1/2}\|_{\bmG_{\mathrm{FP}}}  \leq c \frac{\log \frac{N}{\delta}}{\sqrt{N}} \lambda^{(\alpha+r)/2},
\end{align}
provided that $\lambda^{1- r} \geq \frac{18c}{N}\log \frac{N}{\delta}$ and $\lambda^{1- s} \geq \frac{18c}{K}\log \frac{K}{\delta}$.
\end{lemma}

\begin{lemma}[Refined Lemma \ref{lem:error_3}]\label{lem:error_3_refined} There exists a constant $c>0$ that does not depend on $K, \delta$, such that, for any $\delta \in (0, 1]$, with probability at least $1-\delta$
\begin{align}
    \|(w_K - w) C^{1/2}\|_{\bmG_{\mathrm{FP}}} \leq c \left(\frac{\log \frac{K}{\delta}}{\sqrt{K}} \lambda^{(\alpha+s)/2} + \lambda^{(1+\alpha)/2}\right),
\end{align}
if $\lambda^{1- s} \geq \frac{18c}{K}\log \frac{K}{\delta}$.
\end{lemma}

\begin{lemma}[Refined Lemma \ref{lem:hat_D}]\label{lem:hat_D_refined} For any $\delta \in (0, 1]$, with probability at least $1-\delta$
\begin{align}
    \|\hat D  - \tilde D\|_{\bmG_{\mathrm{FP}}} &\leq c \log \frac{2}{\delta} \left(\varepsilon_{\infty}N^{-1} + \varepsilon\right),
\end{align}
where $c>0$ is a constant that does not depend on $N, K, \delta$.
\end{lemma}

\begin{lemma}[Refined Lemma \ref{lem:hat_C}]\label{lem:hat_C_refined} For any $\delta \in (0, 1]$, with probability at least $1-\delta$,
\begin{align}
    \|\hat C  - \tilde C\|_{\bmG_{\mathrm{FP}}} &\leq c \log \frac{2}{\delta} \left(\varepsilon_{\infty}N^{-1} +  \varepsilon\right).
\end{align}
Moreover, if $\lambda \geq 2c \log \frac{2}{\delta} (\varepsilon_{\infty}N^{-1} +  \varepsilon N^{-1/2} + \varepsilon)$,
\begin{align}
    \|\hat C_{\lambda}^{-1/2}\tilde C^{1/2}\|_{\infty} \leq \sqrt{2},
\end{align}
where $c>0$ is a constant that does not depend on $N, K, \delta$.
\end{lemma}

\begin{lemma}[Refined Lemma \ref{lem:sampling}]\label{lem:sampling_refined} For any $\delta \in (0,1]$, if $\lambda^{1-r} \geq \frac{18c}{N}\log \frac{N}{\delta}$ and $\lambda^{1-s} \geq \frac{18c}{K}\log \frac{K}{\delta}$, with probability at least $1-\delta$,
\begin{align}
     \|\tilde C_{\lambda}^{-1/2}C_{\lambda}^{1/2}\|_{\infty} &\leq 2.
\end{align}
\end{lemma}

\subsection{Proof of \texorpdfstring{\(L^{\infty}\)}{L-infinity} learning rates}\label{subsec:proof_linfty}

\begin{lemma}\label{lem:C_alpha}
Let \(\gamma \triangleq r \wedge s\). Under the same assumptions as Theorem~\ref{th:refined}, for any \(\delta \in (0,1]\), with probability at least \(1-\delta\), we have
\begin{align}
    \|(\hat{w} - w) C^{\gamma/2}\|_{\bmG_{\mathrm{FP}}} &\leq c \log \frac{2}{\delta} \left(\frac{\log \frac{K \wedge N}{\delta}}{\sqrt{K \wedge N}}\right)^{\frac{\alpha + \gamma}{1 - \gamma}},
\end{align}
where $c > 0$ is independent of $N, K, \delta$.
\end{lemma}
\begin{proof} We have
\begin{align}
    \|(\hat w - w) C^{\gamma/2}\|_{\bmG_{\mathrm{FP}}} &\leq \|(\hat w - w) C_{\lambda}^{1/2}\|_{\bmG_{\mathrm{FP}}} \|C_{\lambda}^{1/2}C^{\gamma/2}\|_{\bmG_{\mathrm{FP}}}\\
    &\leq \left(\|(\hat w - w) C^{1/2}\|_{\bmG_{\mathrm{FP}}} + \lambda^{1/2} \|\hat w - w\|_{\bmG_{\mathrm{FP}}} \right) \lambda^{(\gamma-1)/2}.
\end{align}

Theorem \ref{th:refined} provides the bound
\begin{align}
    \|(\hat w - w) C^{1/2}\|_{\bmG_{\mathrm{FP}}} &\leq \bmB(N, K, \delta) \triangleq c_1 \log \frac{2}{\delta} \left(\frac{\log \frac{K \wedge N}{\delta}}{\sqrt{K \wedge N}}\right)^{\frac{1+\alpha}{1-\gamma}}.
\end{align}
taking \(\lambda = c_2 \left(\left(\frac{18c}{N}\log^2 \frac{N}{\delta}\right)^{\frac{1}{1-r}} + \left(\frac{18c}{K}\log^2 \frac{K}{\delta}\right)^{\frac{1}{1-s}}\right)\), where constants \(c_1, c_2 > 0\) are independent of \(N\), \(K\), and \(\delta\).

The same proof as Theorem \ref{th:refined} leads, under the same assumptions, to the bound
\begin{align}
    \|\hat w - w\|_{\bmG_{\mathrm{FP}}} \leq \bmB(N, K, \delta) \times \lambda^{-1/2}.
\end{align}

Therefore,
\begin{align}
    \|(\hat w - w) C^{\gamma/2}\|_{\bmG_{\mathrm{FP}}} &\leq 2\bmB(N, K, \delta)\lambda^{-(1-\gamma)/2}\\
    &\leq c \log \frac{2}{\delta} \left(\frac{\log \frac{K \wedge N}{\delta}}{\sqrt{K \wedge N}}\right)^{\frac{\gamma +\alpha}{1 - \gamma}},
\end{align}
for a constant $c >0$.

\end{proof}

\begin{theorem}[\(L^{\infty}\) learning rates]
Under identical conditions to those in Theorem \ref{th:refined}, then for any
\(\delta \in (0, 1]\), with probability at least \(1-\delta\),
\begin{align}
    \|p_{\hat b, \hat a} - p\|_{L^\infty(\mathbb P_c)L^2([0,T]\times \R^n)}
    &\leq c \log \frac{2}{\delta}
    \left(\frac{\log \frac{K \wedge N}{\delta}}{\sqrt{K \wedge N}}\right)^{\frac{\alpha + r \wedge s}{1 - r \wedge s}},
\end{align}
where \(c> 0\) is a constant independent of \(N\), \(K\), or \(\delta\).
\end{theorem}

\begin{proof}
We have, defining \(\gamma = r \wedge s\), for \(\mathbb P_c\)-almost every \(\theta\),
\begin{align}
    \E_{t, x} \left[\left(\frac{\partial p}{\partial t}(\theta, t, x)
    - \bmL^{(\hat b,\hat a), \theta} p(\theta, t, x)\right)^2\right]^{1/2}
    &= c\|(\hat w - w) C(\theta)^{1/2}\|_{\bmG_{\mathrm{FP}}} \\
    &\leq c\|(\hat w - w)C^{\gamma/2}\|_{\bmG_{\mathrm{FP}}}
    \|C_{\lambda}^{-\gamma/2}C^{1/2}(\theta)\|_{\infty} \\
    &\leq c c_0 \|(\hat w - w)C^{\gamma/2}\|_{\bmG_{\mathrm{FP}}}.
\end{align}
Therefore, there exists \(c>0\) such that
\begin{align}
     \|p_{\hat b, \hat a} - p\|_{L^\infty(\mathbb P_c) L^2([0, T]\times \R^n)}
     \leq c \|(\hat w - w)C^{\gamma/2}\|_{\bmG_{\mathrm{FP}}}.
\end{align}
Then, applying Lemma \ref{lem:C_alpha} completes the proof.
\end{proof}

% \begin{corollary}[\(L^{\infty}\) learning rates]
% Under identical conditions to those in Theorem \ref{th:refined}, then for any
% \(\delta \in (0,1]\), with probability at least \(1-\delta\),
% \begin{align}
%     \|p_{\hat b, \hat a} - p\|_{L^\infty(\mathbb P_c)L^\infty([0,T])L^2(\R^n)}
%     &\leq c \log \frac{2}{\delta}
%     \left(
%     \frac{\log \frac{K \wedge N}{\delta}}{\sqrt{K \wedge N}}
%     \right)^{\frac{\alpha + r \wedge s}{1 - r \wedge s}},
% \end{align}
% where \(c>0\) is a constant independent of \(N\), \(K\), or \(\delta\).
% \end{corollary}
% \begin{proof} We have, defining \(\gamma = r \wedge s\), almost everywhere,
% \begin{align}
%     \E_{t, x} \left[\left(\frac{\partial p}{\partial t}(\theta, t, x) - \bmL^{(\hat b,\hat a), \theta} p(\theta, t, x)\right)^2\right]^{1/2} &= c\|(\hat w - w) C(\theta)^{1/2}\|_{\bmG_{\mathrm{FP}}}\\
%     &\leq c\|(\hat w - w)C^{\gamma/2}\|_{\bmG_{\mathrm{FP}}} \|C_{\lambda}^{-\gamma/2}C^{1/2}(\theta)\|_{\infty}\\
%     &\leq c c_0 \|(\hat w - w)C^{\gamma/2}\|_{\bmG_{\mathrm{FP}}}.
% \end{align}
% Therefore, there exists $c>0$, such that
% \begin{align}
%      \|p_{\hat b, \hat a} - p\|_{L^\infty(\mathbb P_c) L^\infty([0, T]) L^2(\R^n)} \leq c \|(\hat w - w)C^{\gamma/2}\|_{\bmG_{\mathrm{FP}}}.
% \end{align}
% Then, applying Lemma \ref{lem:C_alpha} allows to conclude the proof.
% \end{proof}

\subsection{Proof of CVaR Learning Rates}\label{subsec:proof_cvar}

\paragraph{Value at Risk (VaR).}
For any \(\rho\in(0,1]\), the value at risk \(VaR_{\rho}(X)\) of a random variable
\(X:\Omega \to \R\) is defined as
\begin{align}
    VaR_{\rho}(X)
    \triangleq
    \inf_{t \in \R} \{ t: \mathbb{P}(X \leq t ) \geq 1-\rho\}.
\end{align}

\paragraph{Conditional Value at Risk (CVaR).}
For any \(\rho\in(0,1]\), the conditional value at risk of a random variable
\(X:\Omega \to \R\) is defined as
\begin{align}
    CVaR_{\rho}(X)
    \triangleq
    \inf_{t \in \R} \{ t + \rho^{-1} \E[(X-t)_+]\}.
\end{align}
When the cumulative distribution function of \(X\) is continuous at \(VaR_{\rho}(X)\),
it holds that
\begin{align}
    CVaR_{\rho}(X) = \E[X \mid X \geq VaR_{\rho}(X)].
\end{align}
Thus, CVaR can be interpreted as the expected value conditional on being in the
worst \(\rho\)-fraction of outcomes.

\begin{lemma}
For any \(f \in L^{\infty}(\R^n)\), \(\rho \in (0,1]\), and two random variables
\(X_1,X_2: \Omega \to \R^n\) with probability density functions \(p_1,p_2\), assume
that the cumulative distribution functions of \(f(X_1)\) and \(f(X_2)\) are continuous
at \(VaR_{\rho}(f(X_1))\) and \(VaR_{\rho}(f(X_2))\), respectively. Then
\begin{equation}
    | CVaR_{\rho}(f(X_1)) - CVaR_{\rho}(f(X_2))|
    \leq
    \frac{2\|f\|_{\infty}}{\rho}\|p_1 - p_2\|_{L^1}.
\end{equation}
\end{lemma}

\begin{proof}
Let \(C_i = CVaR_{\rho}(f(X_i))\) and \(V_i = VaR_{\rho}(f(X_i))\), for \(i=1,2\).
Assume first that \(V_1 \leq V_2\). By the continuity assumption,
\[
\mathbb P(f(X_i)\leq V_i)=1-\rho,\qquad i=1,2.
\]
Using the conditional representation of CVaR,
\begin{align*}
|C_1 - C_2|
&=
\frac{1}{\rho}\left|
\int_{f(x) \geq V_1} f(x) p_1(x) dx
-
\int_{f(x) \geq V_2} f(x) p_2(x) dx
\right|\\
&\leq
\frac{1}{\rho}\left|
\int_{V_1 \leq f(x) \leq V_2} f(x) p_1(x) dx
\right|
+
\frac{1}{\rho}\left|
\int_{f(x) \geq V_2} f(x) (p_1(x)-p_2(x)) dx
\right|.
\end{align*}
The second term satisfies
\[
\left|
\int_{f(x) \geq V_2} f(x) (p_1(x)-p_2(x)) dx
\right|
\leq
\|f\|_\infty \|p_1-p_2\|_{L^1}.
\]
For the first term,
\begin{align*}
\left|
\int_{V_1 \leq f(x) \leq V_2} f(x) p_1(x) dx
\right|
&\leq
\|f\|_\infty
\int_{V_1 \leq f(x) \leq V_2} p_1(x) dx\\
&=
\|f\|_\infty
\left|
\int_{f(x)\leq V_2}p_1(x)dx
-
\int_{f(x)\leq V_1}p_1(x)dx
\right|\\
&=
\|f\|_\infty
\left|
\int_{f(x)\leq V_2}(p_1(x)-p_2(x))dx
\right|\\
&\leq
\|f\|_\infty\|p_1-p_2\|_{L^1}.
\end{align*}
Combining the two bounds gives
\[
|C_1-C_2|
\leq
\frac{2\|f\|_\infty}{\rho}\|p_1-p_2\|_{L^1}.
\]
The case \(V_2 \leq V_1\) is symmetric.
\end{proof}

\begin{lemma}
For any \(g\in L^2(\R^n)\), if there exists \(\beta>0\) such that
\[
    \int_{\R^n}\|x\|^\beta |g(x)|\,dx <+\infty,
\]
then
\begin{align}
    \|g\|_{L^1(\R^n)}
    \leq
    \|g\|_{L^2(\R^n)}^{\frac{\beta}{\beta+n/2}}
    \left(
        3+\int_{\R^n}\|x\|^\beta |g(x)|\,dx
    \right).
\end{align}
\end{lemma}

\begin{proof}
Let \(R>0\), and denote \(\bmB=B_R^{\R^n}(0)\). Then
\[
    \|g\|_{L^1(\R^n)}
    =
    \|g\|_{L^1(\bmB)}
    +
    \|g\|_{L^1(\R^n\setminus \bmB)}.
\]
By Cauchy--Schwarz,
\[
    \|g\|_{L^1(\bmB)}
    \leq
    \|g\|_{L^2(\bmB)}|\bmB|^{1/2}
    \leq
    3 R^{n/2}\|g\|_{L^2(\R^n)}.
\]
Moreover,
\[
    \|g\|_{L^1(\R^n\setminus \bmB)}
    \leq
    R^{-\beta}\int_{\R^n}\|x\|^\beta |g(x)|\,dx.
\]
Thus
\[
    \|g\|_{L^1(\R^n)}
    \leq
    3R^{n/2}\|g\|_{L^2(\R^n)}
    +
    R^{-\beta}\int_{\R^n}\|x\|^\beta |g(x)|\,dx.
\]
Taking
\[
    R=\|g\|_{L^2(\R^n)}^{-\frac{1}{\beta+n/2}}
\]
gives the result.
\end{proof}

% \begin{lemma} For any $f\in L^2(\R^n)$, if there exists $\beta>1$ such that $\|f x^\beta\|_{L^1(\R^n)} < +\infty$, then we have
% \begin{align}
%     \|f\|_{L^1(\R^n)} \leq  \|f\|_{L^2(\R^n)}^{\frac{\beta}{\beta + n/2}} \left(3 +  \|f x^\beta\|_{L^1(\R^n)}\right).
% \end{align}
    
% \end{lemma}
% \begin{proof} For $f\in L^2(\R^n)$, and $R>0$, denoting $\bmB = B_R^{\R^n}(0)$, we have 
% \begin{align}
%     \|f\|_{L^1(\R^n)} &= \|f\|_{L^1(\bmB)} + \|f\|_{L^1(\R^n \setminus \bmB)}.
% \end{align}

% Moreover,
% \begin{align}
%     \|f\|_{L^1(\bmB)} &\leq \|f\|_{L^2(\bmB)} 3 R^{n/2}\\
%     &\leq \|f\|_{L^2(\R^n)} 3 R^{n/2}.
% \end{align}

% Furthermore, similar proof than for the standard Markov's inequality gives
% \begin{align}
%    \|f\|_{L^1(\R^n \setminus \bmB)} \leq R^{-\beta}  \|f x^\beta\|_{L^1(\R^n)}.
% \end{align}

% Hence,
% \begin{align}
%      \|f\|_{L^1(\R^n)} \leq \|f\|_{L^2(\R^n)} 3 R^{n/2} + R^{-\beta}  \|f x^\beta\|_{L^1(\R^n)}.
% \end{align}

% Now, taking $R= \|f\|_{L^2(\R^n)}^{-\frac{n/2}{\beta + n/2}}$, gives 
% \begin{align}
%     \|f\|_{L^1(\R^n)} \leq  \|f\|_{L^2(\R^n)}^{\frac{\beta}{\beta + n/2}} \left(3 +  \|f x^\beta\|_{L^1(\R^n)}\right).
% \end{align}
    
% \end{proof}

\begin{lemma}[Bounded moments of SDE solutions]
Let \((\hat b,\hat a)\in\bmF\), with \(\hat a=\hat\sigma\hat\sigma^\top\), and let
\(X_{\hat b,\hat\sigma}^{\theta}\) be the solution of the SDE driven by
\((\hat b,\hat\sigma)\) under the control \(u_\theta\), for \(\theta\in\Theta\).
Assume that the hypotheses of Lemma~\ref{lem:FP_ineq_main} hold for
\((\hat b,\hat a)\). If \(\mathbb E[|X_0|^\beta]<\infty\) for some
\(\beta>2\), then there exists \(C>0\) such that, for all
\(\theta\in\Theta\) and all \(t\in[0,T]\),
\[
    \mathbb E\!\left[|X_{\hat b,\hat\sigma}^{\theta}(t)|^\beta\right]\leq C .
\]
\end{lemma}

\begin{proof}
Fix \(\theta\in\Theta\) and write \(X(t)=X_{\hat b,\hat\sigma}^{\theta}(t)\).
By the hypotheses of Lemma~\ref{lem:FP_ineq_main}, the SDE is well posed.
Moreover, Assumption~A2 gives \(u_\theta([0,T])\subset V\) for all
\(\theta\in\Theta\). Since \((\hat b,\hat a)\in\bmF\), the localization and
boundedness assumptions imply that there exists \(B<\infty\), independent of
\(\theta\), such that
\[
    \sup_{\theta\in\Theta}\sup_{t\in[0,T],\,x\in\mathbb R^n}
    |\hat b(t,x,u_\theta(t))| \le B .
\]
Similarly, \(\hat a\) is uniformly bounded on
\([0,T]\times\mathbb R^n\times V\). Choosing a measurable square root
\(\hat\sigma\) with \(|\hat\sigma|_{\mathrm F}^2=\operatorname{tr}(\hat a)\),
there exists \(\Sigma<\infty\), independent of \(\theta\), such that
\[
    \sup_{\theta\in\Theta}\sup_{t\in[0,T],\,x\in\mathbb R^n}
    |\hat\sigma(t,x,u_\theta(t))|_{\mathrm F} \le \Sigma .
\]

Using the integral form of the SDE, Hölder's inequality, and the
Burkholder--Davis--Gundy inequality, there exists \(C_\beta>0\) such that, for
all \(t\in[0,T]\),
\begin{align*}
\mathbb E[|X(t)|^\beta]
&\leq
C_\beta\Bigg(
\mathbb E[|X_0|^\beta]
+
\mathbb E\left[
\left|
\int_0^t \hat b(s,X(s),u_\theta(s))\,ds
\right|^\beta
\right]  \\
&\qquad\qquad
+
\mathbb E\left[
\left|
\int_0^t \hat\sigma(s,X(s),u_\theta(s))\,dW_s
\right|^\beta
\right]
\Bigg) \\
&\leq
C_\beta\left(
\mathbb E[|X_0|^\beta]
+
T^\beta B^\beta
+
T^{\beta/2}\Sigma^\beta
\right).
\end{align*}
The right-hand side is finite and independent of \(t\) and \(\theta\). Absorbing
constants gives the claim.
\end{proof}

\subsection{Proof of Lemma \ref{lem:induced_emb}}\label{proof_emb_sobolev}

We first record a  consequence of the Sobolev--Morrey embedding.
It allows us to turn Sobolev regularity on a product domain into
uniform Sobolev bounds on sections.

\begin{lemma}[Uniform Sobolev sections]
\label{lem:uniform-sobolev-sections}
Let $X\subset \mathbb R^{d_x}$ and $Y\subset \mathbb R^{d_y}$ be bounded
Lipschitz domains. Let $k,s\in\mathbb N$ with $k>d_x/2$. Assume that $f \in H^{k,s}(X \times Y)$, namely
\[
    \partial_x^\alpha\partial_y^\beta f\in L^2(X\times Y),
    \qquad
    |\alpha|\le k,\quad |\beta|\le s .
\]
Then $x\mapsto f(x,\cdot)$ admits a continuous representative with values in
$H^s(Y)$ and there exists a constant $C>0$, depending only on $X,d_x,k$, such
that
\[
    \sup_{x\in X}
    \|f(x,\cdot)\|_{H^s(Y)}
    \le
    C
    \|f\|_{H^{k,s}(X\times Y)} .
\]
\end{lemma}

\begin{proof}
Fix a multi-index $\beta\in\mathbb N^{d_y}$ with $|\beta|\le s$, and set
\[
    h_\beta(x):=\partial_y^\beta f(x,\cdot).
\]
Then $h_\beta\in H^k(X;L^2(Y))$. Indeed, by Fubini, for every
multi-index $\alpha\in\mathbb N^{d_x}$ with $|\alpha|\le k$,
\[
    \partial_x^\alpha h_\beta
    =
    \partial_x^\alpha\partial_y^\beta f
    \in L^2(X;L^2(Y)),
\]
where derivatives are understood in the Sobolev sense. Moreover,
\[
    \|h_\beta\|_{H^k(X;L^2(Y))}^2
    =
    \sum_{|\alpha|\le k}
    \|\partial_x^\alpha\partial_y^\beta f\|_{L^2(X\times Y)}^2 .
\]

Since $X$ is a bounded Lipschitz domain, the Sobolev extension theorem
gives a bounded linear extension operator on scalar Sobolev spaces. The
same operator extends to $L^2(Y)$-valued Sobolev spaces by applying it in
the $x$-variable. Equivalently, expanding in an orthonormal basis of
$L^2(Y)$ and using Parseval's identity gives a bounded linear operator
\[
    E_X:H^k(X;L^2(Y))
    \to
    H^k(\mathbb R^{d_x};L^2(Y))
\]
such that, for every $h\in H^k(X;L^2(Y))$,
\[
    (E_X h)|_X=h
    \quad\text{a.e. on }X,
    \qquad
    \|E_X h\|_{H^k(\mathbb R^{d_x};L^2(Y))}
    \le
    C
    \|h\|_{H^k(X;L^2(Y))}.
\]

Because $k>d_x/2$ and $L^2(Y)$ is a Hilbert space, the Hilbert-valued
Sobolev embedding gives
\[
    H^k(\mathbb R^{d_x};L^2(Y))
    \hookrightarrow
    C^0(\mathbb R^{d_x};L^2(Y)).
\]
Applying this embedding to $E_X h_\beta$, we obtain
\[
\begin{aligned}
    \sup_{x\in X}
    \|\partial_y^\beta f(x,\cdot)\|_{L^2(Y)}
    &=
    \sup_{x\in X}\|h_\beta(x)\|_{L^2(Y)}  \\
    &\le
    \|E_X h_\beta\|_{C^0(\mathbb R^{d_x};L^2(Y))} \\
    &\le
    C
    \|E_X h_\beta\|_{H^k(\mathbb R^{d_x};L^2(Y))} \\
    &\le
    C
    \|h_\beta\|_{H^k(X;L^2(Y))}.
\end{aligned}
\]

Summing over all multi-indices $\beta$ with $|\beta|\le s$, we get
\[
\begin{aligned}
    \sup_{x\in X}
    \|f(x,\cdot)\|_{H^s(Y)}^2
    &=
    \sup_{x\in X}
    \sum_{|\beta|\le s}
    \|\partial_y^\beta f(x,\cdot)\|_{L^2(Y)}^2  \\
    &\le
    C
    \sum_{|\beta|\le s}
    \|h_\beta\|_{H^k(X;L^2(Y))}^2  \\
    &=
    C
    \sum_{|\beta|\le s}
    \sum_{|\alpha|\le k}
    \|\partial_x^\alpha\partial_y^\beta f\|_{L^2(X\times Y)}^2  \\
    &=
    C
    \|f\|_{H^{k,s}(X\times Y)}^2 .
\end{aligned}
\]
In particular, $x\mapsto f(x,\cdot)$ is bounded from $X$ to $H^s(Y)$.
Furthermore, the Sobolev embedding above also gives continuity of each map
$x\mapsto \partial_y^\beta f(x,\cdot)$ in $L^2(Y)$. Since only finitely
many $\beta$'s are involved, this implies that
\[
    x\mapsto f(x,\cdot)
\]
is continuous from $X$ to $H^s(Y)$.
\end{proof}

We next recall a standard supercritical Sobolev composition result. It will be
used to control functions evaluated along parametrized controls.

\begin{lemma}[Composition along parametrized controls]\label{lem:composition-controls}
Let $\Theta\subset\mathbb R^m$, $D\subset\mathbb R^n$, and
$V\subset\mathbb R^d$ be bounded Lipschitz domains, and let $T>0$.
Let $q,\tau\in\mathbb N$ satisfy
\[
    q>\frac{m+n+1}{2},
    \qquad
    \tau>q+\frac{n+d+1}{2}.
\]
Assume that
\[
    u:(\theta,t)\mapsto u_\theta(t)
\]
belongs to $H^q(\Theta\times[0,T];\mathbb R^d)$ and takes values in $V$.
Let
\[
    f\in H^\tau([0,T]\times D\times V;\mathbb R^\ell).
\]
Then the function
\[
    g:\Theta\times[0,T]\times D\to\mathbb R^\ell,
    \qquad
    g(\theta,t,x):=f(t,x,u_\theta(t)),
\]
belongs to
\[
    H^q(\Theta\times[0,T]\times D;\mathbb R^\ell).
\]
Moreover, for every $R>0$, there exists a constant $C_R>0$ such that
\[
    \|g\|_{H^q(\Theta\times[0,T]\times D)}
    \le
    C_R
    \|f\|_{H^\tau([0,T]\times D\times V)}
\]
whenever
\[
    \|u\|_{H^q(\Theta\times[0,T])}\le R.
\]
\end{lemma}

\begin{proof}
By Sobolev embedding on the bounded Lipschitz domain
\((0,T)\times D\times V\), the condition
\[
    \tau>q+\frac{n+d+1}{2}
\]
implies
\[
    H^\tau([0,T]\times D\times V)
    \hookrightarrow
    C^q_b([0,T]\times D\times V),
\]
and therefore
\[
    \|f\|_{C^q_b}
    \le
    C\|f\|_{H^\tau([0,T]\times D\times V)}.
\]

Set
\[
    \Omega:=\Theta\times[0,T]\times D,
    \qquad
    U(\theta,t,x):=u_\theta(t),
    \qquad
    \Phi(\theta,t,x):=(t,x,U(\theta,t,x)).
\]
Since \(U\) is independent of \(x\), Fubini gives
\[
    \|U\|_{H^q(\Omega)}
    \le
    C\|u\|_{H^q(\Theta\times[0,T])}.
\]
Hence, if \(\|u\|_{H^q(\Theta\times[0,T])}\le R\), then
\[
    \|\Phi\|_{H^q(\Omega)}\le C_R.
\]

For \(|\alpha|\le q\), the Faà di Bruno formula expresses
\(\partial^\alpha(f\circ\Phi)\) as a finite sum of terms of the form
\[
    D^r f(\Phi)
    \bigl[
        \partial^{\gamma_1}\Phi,\ldots,
        \partial^{\gamma_r}\Phi
    \bigr],
    \qquad
    |\gamma_1|+\cdots+|\gamma_r|=|\alpha|.
\]
The factors \(D^r f(\Phi)\) are bounded by \(\|f\|_{C^q_b}\). Since
\(q>\dim(\Omega)/2\), the standard Sobolev product estimate controls each
product of derivatives of \(\Phi\) in \(L^2(\Omega)\) by \(C_R\). Therefore
\[
    \|f\circ\Phi\|_{H^q(\Omega)}
    \le
    C_R\|f\|_{C^q_b}.
\]
Using the Sobolev embedding bound for \(f\), we obtain
\[
    \|g\|_{H^q(\Theta\times[0,T]\times D)}
    \le
    C_R
    \|f\|_{H^\tau([0,T]\times D\times V)}.
\]
The vector-valued case follows componentwise.
\end{proof}

We now verify the assumptions used in the Sobolev example. We first recall the
standard Sobolev--RKHS embedding property, and then explain how it applies to
the RKHS induced by the Fokker--Planck matching operator.

\paragraph{Standard Sobolev--RKHS embedding property.}
Let $Z\subset\mathbb R^d$ be a bounded Lipschitz domain, and let $\rho$ be a
probability measure on $Z$ whose density with respect to Lebesgue measure is
bounded above and below by positive constants. Let $k_{q,d}$ be a Sobolev
kernel, equivalently a Matérn-type kernel, whose RKHS is norm-equivalent to
$H^q(Z)$, with $q>d/2$. Denote by $\phi_{q,d}$ its feature map and define
\[
    C_Z
    :=
    \int_Z
    \phi_{q,d}(z)\otimes \phi_{q,d}(z)
    \,d\rho(z).
\]
Then, for every \(r<1-\frac{d}{2q}\), the standard Sobolev--RKHS embedding property gives
\[
    \|C_Z^{-r/2}\phi_{q,d}(z)\|_{H^q(Z)}\le C_r
\]
for \(\rho\)-almost every \(z\in Z\).
% Then the standard Sobolev--RKHS embedding property gives, up to the usual
% endpoint convention,
% \[
%     \|C_Z^{-r/2}\phi_{q,d}(z)\|_{H^q(Z)}
%     \le C,
%     \qquad
%     r=1-\frac{d}{2q},
% \]
% for $\rho$-almost every $z\in Z$. 
This is the usual embedding property used in
kernel ridge regression; see, for instance,
\citet{pillaud2018statistical,fischer2020sobolev,wendland2004scattered,steinwart2009optimal}.

\paragraph{Induced Fokker--Planck RKHS.}
In Fokker--Planck matching, the above result cannot be applied directly to the
coefficient RKHS. Indeed, the relevant feature map $\tilde\phi$ is not the
canonical feature map of a Sobolev kernel on the coefficient space. Instead, it
is the feature map of the image RKHS induced by applying the linear
Fokker--Planck operator to the coefficient RKHS. More precisely, for
$h=(\hat b,\hat a)\in\mathcal F$, define
\[
    Ah(\theta,t,x)
    :=
    L^{u_\theta}_{\hat b,\hat a}p(\theta,t,x).
\]
The induced image space is
\[
    \mathcal M
    :=
    \left\{
    (\theta,t,x)\mapsto Ah(\theta,t,x)
    \;:\;
    h\in\mathcal F
    \right\}.
\]
The purpose of the following lemma is to show that, under Sobolev assumptions
on the coefficients, controls, and density, the conditional sections of
$\mathcal M$ embed continuously into Sobolev spaces. The standard
Sobolev--RKHS embedding property can then be applied conditionally on
$D\subset\mathbb R^{n+1}$ and on $\Theta\subset\mathbb R^m$.

\begin{lemma}[Sobolev verification of the embedding assumptions]
\label{lem:sobolev-verification}
Let $\Theta\subset\mathbb R^m$, $D\subset\mathbb R^n$, and
$V\subset\mathbb R^d$ be bounded Lipschitz domains. Let
$\nu,\tau\in\mathbb N$ satisfy
\[
    \nu>\frac{\max\{m,n+1\}}{2},
    \qquad
    \tau>2\nu+2+\frac{n+d+1}{2}.
\]
Assume that
\[
    (\theta,t)\mapsto u_\theta(t)
    \in H^{2\nu+2}(\Theta\times[0,T];\mathbb R^d),
    \qquad
    u_\theta(t)\in V,
\]
and that
\[
    p\in H^{2\nu+2}(\Theta\times[0,T]\times D).
\]
Assume also that the coefficient hypothesis space is continuously embedded in
\[
    H^\tau([0,T]\times D\times V),
\]
with the same ellipticity and localization constraints as in the estimator,
and assume attainability of the true coefficients. Assume that \(\mathbb P_c\) is the uniform probability measure on \(\Theta\).
Then the source condition holds with $\alpha=0$, and the embedding
assumptions hold for every
\[
    r<1-\frac{n+1}{2\nu},
    \qquad
    s<1-\frac{m}{2\nu}.
\]
\end{lemma}

\begin{proof}
For \(h=(\hat b,\hat a)\in\mathcal F\), define
\[
    Ah(\theta,t,x)
    :=
    L^{u_\theta}_{\hat b,\hat a}p(\theta,t,x).
\]
Equivalently,
\[
\begin{aligned}
    Ah
    &=
    \frac12\sum_{i,j=1}^n
    \partial_{x_i x_j}
    \left(
        \hat a_{ij}(t,x,u_\theta(t))p
    \right)
    -
    \sum_{i=1}^n
    \partial_{x_i}
    \left(
        \hat b_i(t,x,u_\theta(t))p
    \right).
\end{aligned}
\]
The map \(A\) is linear in \(h\), since the density \(p\) is fixed.

By Sobolev embedding on the bounded Lipschitz domain
\((0,T)\times D\times V\), the condition
\[
    \tau>2\nu+2+\frac{n+d+1}{2}
\]
implies
\[
    H^\tau([0,T]\times D\times V)
    \hookrightarrow
    C_b^{2\nu+2}([0,T]\times D\times V).
\]
Hence the coefficient hypothesis space embeds continuously into
\(C_b^{2\nu+2}([0,T]\times D\times V)\).

By the composition lemma, applied with order \(2\nu+2\), the composed
coefficients
\[
    (\theta,t,x)\mapsto \hat b_i(t,x,u_\theta(t)),
    \qquad
    (\theta,t,x)\mapsto \hat a_{ij}(t,x,u_\theta(t))
\]
belong to
\[
    H^{2\nu+2}(\Theta\times[0,T]\times D),
\]
with norms bounded by \(C\|h\|_{\mathcal F}\). The cutoff does not affect this
estimate, since \(\xi\in C_c^\infty(\mathbb R^n)\) is fixed and multiplication
by \(\xi\) is continuous on the Sobolev spaces considered.

Since
\[
    p\in H^{2\nu+2}(\Theta\times[0,T]\times D)
\]
and
\[
    2\nu+2>\frac{m+n+1}{2},
\]
the Sobolev product estimate gives
\[
    \hat a_{ij}(t,x,u_\theta(t))p,
    \qquad
    \hat b_i(t,x,u_\theta(t))p
    \in H^{2\nu+2}(\Theta\times[0,T]\times D),
\]
again with norms bounded by \(C\|h\|_{\mathcal F}\). Applying the two
\(x\)-derivatives in the Fokker--Planck operator yields
\[
    \|Ah\|_{H^{2\nu}(\Theta\times[0,T]\times D)}
    \le
    C\|h\|_{\mathcal F}.
\]

The isotropic bound above implies the mixed regularities required by
Lemma~\ref{lem:uniform-sobolev-sections}. Applying that lemma first with
\[
    X=\Theta,
    \qquad
    Y=[0,T]\times D,
    \qquad
    k=s=\nu,
\]
and then with
\[
    X=[0,T]\times D,
    \qquad
    Y=\Theta,
    \qquad
    k=s=\nu,
\]
is valid because
\[
    \nu>\frac m2,
    \qquad
    \nu>\frac{n+1}{2}.
\]
Therefore,
\[
    \sup_{\theta\in\Theta}
    \|Ah(\theta,\cdot)\|_{H^\nu([0,T]\times D)}
    +
    \sup_{(t,x)\in[0,T]\times D}
    \|Ah(\cdot,t,x)\|_{H^\nu(\Theta)}
    \le
    C\|h\|_{\mathcal F}.
\]

Define the conditional image spaces
\[
    \mathcal M_\theta
    :=
    \{Ah(\theta,\cdot):h\in\mathcal F\},
    \qquad
    \mathcal M_{t,x}
    :=
    \{Ah(\cdot,t,x):h\in\mathcal F\},
\]
equipped with their image RKHS norms. The preceding estimate gives the
uniform continuous embeddings
\[
    \mathcal M_\theta\hookrightarrow H^\nu([0,T]\times D),
    \qquad
    \mathcal M_{t,x}\hookrightarrow H^\nu(\Theta).
\]
The standard Sobolev-RKHS embedding property, applied to the image
RKHSs through the continuous embeddings, then gives
\[
    r <1-\frac{n+1}{2\nu},
    \qquad
    s< 1-\frac{m}{2\nu}.
\]

Finally, by attainability, the true coefficients \((b,a)\) belong to
\(\mathcal F\). Hence the Fokker--Planck equation gives
\[
    \partial_t p
    =
    L^{u_\theta}_{b,a}p
    =
    A(b,a).
\]
Thus \(\partial_t p\) belongs to the induced image RKHS, and the source
condition holds with \(\alpha=0\).
\end{proof}
\section{Implementation Details of Uncontrolled SDE Estimation}\label{sec:implementation_uncontrolled}

This section details the implementation of the uncontrolled SDE estimation method proposed in \citet{bonalli2023non}, available on GitHub (\url{lmotte/sde-learn}) as an open-source Python library. We provide details on the computations involved, including vectorized versions for efficient computation with Python libraries such as NumPy, as well as the computational complexity of each step. In Section \ref{subsec:step1_density_estimation}, we detail probability density estimation. In Section \ref{subsec:step2_fp_matching}, Fokker-Planck matching is presented.

\paragraph{Notations. } We use the following notations.
\begin{equation}
    \1 \triangleq (1, 1, \dots, 1) \in \mathbb{R}^Q, \qquad
    \hat{p}_i \triangleq \frac{\partial \hat{p}}{\partial x_i}, \qquad
    \hat{p}_{ij} \triangleq \frac{\partial^2 \hat{p}(x)}{\partial x_i \partial x_j},
\end{equation}
and similar notations for the partial derivatives of $\hat{g}$ and $\rho$.

\subsection{Step 1: Probability Density estimation}\label{subsec:step1_density_estimation}

\subsubsection{Estimator Closed-Form} 

From \citet{bonalli2023non}, we have
\begin{align}
    \hat{p}(t,x) = k_t(t)^T K_t^{-1} \hat{g}(x),
\end{align}
with $k_t(t) = (k_t(t, t_l))_{l} \in \mathbb{R}^M$, $\hat{g}(x) = Q^{-1}\mathbf{1}^T (\rho(x, X_{kl}))_k$, and $\rho(x, y) = \mu^n (2\pi)^{-n/2} \exp\left(-\frac{\mu^2}{2} \|x - y\|^2\right)$. We consider \(k_t\) defined as the Gaussian kernel with parameter $\nu$: \(k_t(t,t') \triangleq \exp(-\nu (t-t')^2)\).

\subsubsection{Algorithm}

\paragraph{Inputs.} We are provided with a dataset of $Q$ sample paths $(X^{tr}(w_i, t_l))_{i\in \llbracket 1, Q \rrbracket, l\in \llbracket 1, M\rrbracket}$ sampled from an unknown SDE, along with the hyper-parameters $\nu, \mu>0$.

\paragraph{Step 1.1 ($\hat{p}$ fitting).} We compute and store $K_t^{-1} = (k_t(t_k, t_l))_{k,l}^{-1} \in \mathbb{R}^{M \times M}$, $(T^{tr}_l)_l \triangleq (t_l)_l \in \mathbb{R}^M$, and $(X_{kl})_{kl} \triangleq (X^{tr}(w_k, t_l))_{kl} \in \mathbb{R}^{Q \times M \times n}$.

The time and space complexities of this step are $\bmO(M^3)$ and $\bmO(M^2 + QMn)$, respectively.

\paragraph{Step 1.2 ($\hat{p}$ prediction).} For any $T^{te} \in [0, T]^{M_{te}}$ and $X^{te} \in (\mathbb{R}^n)^{N_{te}}$, the evaluations for each $(t,x) \in T^{te} \times X^{te}$ can be computed as
\begin{align}
   (\hat{p}(t_i, x_j))_{i,j} &= K_{t}^{te,tr} \times K_{t}^{-1} \times Q^{-1}\mathds{1}^T \times G^{te} \in \mathbb{R}^{M_{te} \times N_{te}},
\end{align}
where
\begin{align*}
    K_{t}^{te} &\triangleq (k_t(T^{te}_{j}, t_i))_{ji} \in \mathbb{R}^{M_{te} \times M}, \\
    G^{te} &\triangleq (\rho(X_{j}^{te}, X_{il}))_{ilj} \in \mathbb{R}^{Q \times M \times N_{te}},
\end{align*}
and similar formulas hold for $\hat{p}_i$ (or $\hat{p}_{ij}$) by replacing $\rho$ with $\rho_i$ (or $\rho_{ij}$).

In particular, the fitting phase of step 2 requires the computation of the $\hat{P}_{ij}$ and $\hat{d}$. Using a dataset $Z^{fp} = T^{fp} \times X^{fp}$, with $T^{fp} \in \mathbb{R}^{M_{fp}}$ and $X^{fp} \in (\mathbb R^n)^N$, and denoting $D_x \triangleq (X^{fp}_j - X_{il})_{ilj} \in \mathbb{R}^{{N} \times N \times M}$, the $\hat{P}_{ij}$ can be computed recursively as follows
\begin{align*}
    & G^{fp} =  (\rho(X^{fp}_j, X_{il}))_{ilj}, \\
    & G_i^{fp} =  (\rho_i(X^{fp}_j, X_{il}))_{ilj} = - \mu^2 D_x e_i \odot G^{fp}, \\
    & G_{ij}^{fp} =  (\rho_{ij}(X^{fp}_j, X_{il}))_{ilj} = - \mu^2 (D_x e_j \odot G_i^{fp} + \delta_{ij} G^{fp}), \\
    & \hat{P} = \operatorname{vec}(K_t^{fp, tr} \times K_t^{-1} \times Q^{-1} \mathds{1}^T G^{fp}), \\
    & \hat{P}_i = \operatorname{vec}(K_t^{fp, tr} \times K_t^{-1} \times Q^{-1} \mathds{1}^T G_i^{fp}), \\
    & \hat{P}_{ij} = \operatorname{vec}(K_t^{fp, tr} \times K_t^{-1} \times Q^{-1} \mathds{1}^T G_{ij}^{fp}).
\end{align*}

Moreover, denoting $D_t \triangleq (T^{fp}_j - T_{l})_{lj} \in \R^{M_{fp} \times M}$, we have
\begin{align*}
    \hat d =  \operatorname{vec}((- 2\nu D_t \odot K_t^{fp, tr}) \times K_t^{-1} \times Q^{-1} \mathds{1}^T G^{fp}).
\end{align*}

The time and space complexities are $\bmO(QMN\max(M_{fp}, M))$ and $\bmO(QMN)$, respectively.

\paragraph{Outputs.} We return $\hat{P}, (\hat{P}_{i})_{i=1}^n, (\hat{P}_{ij})_{i,j=1}^n$ and $\hat{d}$.
 
\subsection{Step 2: Fokker-Planck Matching}\label{subsec:step2_fp_matching}

\subsubsection{Estimator Closed-Form}\label{subsubsec:fp_closed}

\paragraph{Optimization objective.} We consider an isotropic diffusion model where \(a(t,x) = a_0(t,x) I_{\mathbb{R}^n}\), and we enforce the positivity of \(a_0\) over a subset of the training points \((t_i, x_i)_{i \in I} \subset (t_i, x_i)_{i=1}^{N}\). Specifically, we solve the optimization problem:
\begin{align}\label{eq:opt_pbm}
    \min_{\substack{(b, a_0)\\ \forall\, i\in I,\, a_0(t_i, x_i) \geq 0}} \quad \frac{1}{N}\sum_{i=1}^{N} \left(\frac{\partial \hat p}{\partial t}(t_i,x_i) - (\bmL_{t}^{b, a})^* \hat p(t_i,x_i)\right)^2 + \lambda\|(b, a)\|_{\bmF}^2.
\end{align}

\paragraph{Solving the optimization problem.} Given a p.d. kernel $k$ over $[0, T] \times \R^n$, we consider the model
\begin{align*}
     \forall i \in \llbracket1, n\rrbracket,  \quad & b^i(t, x) = \langle w_b^i,\, \phi(t, x)\rangle_{\bmG} \quad \text{ with } \quad w_b^i \in \bmG,\\
    &a(t, x) =\langle w_{a_0},\, \phi(t, x)\rangle_{\bmG} I_{\R^n}\quad \text{ with } \quad w_{a_0} \in \bmG,
\end{align*}
where $\phi(t,x) \triangleq k((t,x), (.,.))\in \bmG$. Therefore, defining \(w = ((w_b^i)_{i=1}^n | w_{a_0}) \in \bmG^{n+1}\), Eq. \eqref{eq:opt_pbm} is expressed as
\begin{align}
    \min_{w \in \bmG^{n+1}} \max_{\gamma \geq 0}\: \|(\hat C + \lambda I)^{1/2}w - (\hat C + \lambda I)^{-1/2}v\|_{\bmG^{n+ 1}}^2 - 2\sum_{i \in I}  \gamma_i \langle w_{a_0},\, \phi(t_i, x_i)\rangle_{\bmG},
\end{align}
where $w= (w_b | w_{a_0}) \in \bmG^{n+1}$, $\hat C = N^{-1}\sum_{i=1}^N (\tilde \phi \otimes \tilde \phi)(t_i, x_i)$, $v=N^{-1}\sum_{i=1}^N(\frac{\partial \hat p}{\partial t} \tilde \phi)(t_i, x_i)$, and
\begin{align}
    \tilde \phi = \Big((-\tilde \phi_i)_i | 1/2 \sum_{i}\tilde \phi_{ii}\Big) \in \bmG^{n + 1}.
\end{align}
Note the change of feature map $\tilde \phi$ compared to the non-isotropic diffusion model where 
\begin{align}
    \tilde \phi = ((-\tilde \phi_i)_i | (1/2 \tilde \phi_{ij})_{ij}) \in \bmG^{n + n^2}.
\end{align}

This can be equivalently expressed as
\begin{align}
    \min_{w \in \mathcal{G}^{n+1}} \max_{\gamma \geq 0}\: \|w(\hat{C} + \lambda I)^{1/2} - (v + U^*\phi_{\gamma})(\hat{C} + \lambda I)^{-1/2}\|_{\mathcal{G}^{n+1}}^2 - \|\hat{C}_{\lambda}^{-1/2}(v + U^* \phi_{\gamma})\|_{\mathcal{G}^{n+1}}^2,
\end{align}
where \( \phi_{\gamma} = \sum_{i \in I} \gamma_i \phi(t_i, x_i) \), \( U = \sum_i e_i \otimes (0_{\mathcal{G}^n} | e_i)_i \in \mathcal{G} \otimes \mathcal{G}^{n+1} \), with \( (e_i)_i \) being an orthonormal basis of \( \mathcal{G} \), ensuring \( w_{a_0} = Uw \).

Strong duality guarantees solvability for \( w \in \mathcal{G}^{n+1} \) via
\begin{align}
    \hat{w}_{pc} &\triangleq (\hat{C} + \lambda I)^{-1}(v + U^*\phi_\gamma),
\end{align}
and the optimal \( \gamma \in \mathbb{R}^{|I|} \) is obtained by solving the dual problem
\begin{align}
    \max_{\gamma \geq 0}\: QP(\gamma),
\end{align}
with \( QP(\gamma) \triangleq - \phi_\gamma^* U C_{\lambda}^{-1}U^* \phi_\gamma - 2 \phi_\gamma^* U C_{\lambda}^{-1}v \).

\paragraph{Closed-form formula for \(\hat b_{pc}, \hat a_{pc}\) and \(QP(\gamma)\).} We decompose
\begin{align}
    \hat w_{pc} &= \hat w_{std} + \hat w_{+},
\end{align}
where $\hat w_{std} \triangleq (\hat C + \lambda I)^{-1}v$ is the standard unconstrained ridge regression formula, and  $\hat w_{+} \triangleq (\hat C + \lambda I)^{-1}U^*\phi_\gamma$ is an adjustment stemming from the positivity constraint.

Additionally, by denoting \(\tilde K = \sum_i \tilde K_i + \frac{1}{4} \sum_{i,j=1}^n \tilde K_{ii}^{jj}
\), we deduce
\begin{align}
    \hat w_{+} &= N(\tilde \Phi^* \tilde \Phi + N\lambda I)^{-1}  U^*\phi_\gamma\\
    &= \lambda^{-1} (U^*\phi_\gamma  - \tilde \Phi^*(\tilde \Phi \tilde\Phi^* + N\lambda I)^{-1}\tilde \Phi U^*\phi_\gamma)\\
    &= \lambda^{-1} (U^*\phi_\gamma  - \tilde \Phi^* \tilde K_{N\lambda} ^{-1}\tilde \Phi U^*\phi_\gamma),
\end{align}
using the Woodbury identity for any operators $A$ and $B$: $(AB + \lambda I)^{-1} = \lambda^{-1} (I - A(BA + \lambda I)^{-1}B)$. Hence, setting $V_i = \sum_j e_j \otimes (e_j | 0_\bmG)$, we derive
\begin{align}
    \hat b^i_{pc} &= \langle V_i \hat w_{pc},\, \phi(t,x)\rangle_{\bmG}\\
    &= \underbrace{\langle V_i\hat w_{std},\, \phi(t,x)\rangle_{\bmG}}_{\triangleq \hat b^i_{std}} + \underbrace{\langle V_i\hat w_{+},\, \phi(t,x)\rangle_{\bmG}}_{\triangleq \hat b_+^i},
\end{align}
Moreover, using $UV_i^* = 0$, we have
\begin{align}
    \hat b^i_+ &= \langle V_i \hat w_{+},\, \phi(t,x)\rangle_{\bmG}\\
    &= \lambda^{-1} \left(\phi_\gamma^* U V_i^* \phi(t,x) - \phi_\gamma^* U \tilde\Phi^* \tilde K_{N\lambda} ^{-1} \tilde\Phi V_i^* \phi(t,x)\right)\\
    &=  -\lambda^{-1} \phi_\gamma^* U \tilde \Phi^* \tilde K_{N\lambda} ^{-1} \tilde \Phi V_i^* \phi(t,x).
\end{align}

Similarly,
\begin{align}
    \hat a_{0,pc} = \hat a_{0,std} + \hat a_{0,+},
\qquad
\hat a_{pc} = \hat a_{0,pc} I_{\mathbb R^n},
\end{align}
where \( \hat a_{0,std} \) is the standard unconstrained ridge regression estimator, and
\begin{align}
     \hat a_{0,+} &= \lambda^{-1} \left(\phi_\gamma^* \phi(t,x) - \phi_\gamma^* U \tilde{\Phi}^*  \tilde{K}_{N\lambda}^{-1} \tilde{\Phi} U^* \phi(t,x)\right).
\end{align}

Furthermore,
\begin{align}
    QP(\gamma) &= - \phi_\gamma^* U C_{\lambda}^{-1} U^* \phi_\gamma - 2 \phi_\gamma^* U C_{\lambda}^{-1} v \\
    &= - \gamma^T \Phi U C_{\lambda}^{-1} U^* \Phi^* \gamma - 2 N^{-1} \gamma^T \Phi U C_{\lambda}^{-1} \tilde{\Phi}^* \hat{d}.
\end{align}

Therefore, we conclude
\begin{align}
    & \hat b^i_{pc}(t,x) = \hat b^i_{std}(t,x) + \hat b^i_{+}(t,x),\\
    &\hat b^i_{std}(t,x) = \hat d^T \tilde K_{N\lambda}^{-1} r_i^{b}(t,x),\\
    &\hat b^i_+(t,x) = - \lambda^{-1} \gamma^T (R^{a}_{fp, pc})^T \tilde K_{N\lambda}^{-1}r_i^b(t,x),\\
    & \hat a_{0,pc}(t,x) = \hat a_{0,std}(t,x) + \hat a_{0,+}(t,x),\\
    &\hat a_{pc}(t,x)=\hat a_{0,pc}(t,x)I_{\mathbb R^n},\\
    &\hat a_{0,std}(t,x) = \hat d^T \tilde K_{N\lambda}^{-1} r^{a}(t,x),\\
    &\hat a_{0,+}(t,x) = \lambda^{-1} \gamma^T \left(k(t,x) - (R^{a}_{fp, pc})^T  \tilde K_{N\lambda} ^{-1} r^{a}(t,x)\right) I_{R^n},\\
    &QP(\gamma) = -\lambda^{-1}  \gamma^T (K - (R^{a}_{fp, pc})^T \tilde K_{N\lambda}^{-1} (R^{a}_{fp, pc})) \gamma - 2 \gamma^T (R^{a}_{fp, pc})^T \tilde K_{N\lambda}^{-1} \hat d,
\end{align}
using the following notations
\begin{align}
    &z, z' = (t,x), (t',x')\\
    &r_i^{b}(z, z') \triangleq - \langle \tilde \phi_i(z),\, \phi(z')\rangle_{\bmG},\\
    &r_i^{b}(z) \triangleq \left(r_i^b(z_l, z)\right)_l \in \R^{N},\\
    &r^{a}_{ij}(z, z') \triangleq \langle \tilde \phi_{ij}(z),\, \phi(z')\rangle_{\bmG},\\
    &r^{a}(z, z') \triangleq 1/2 \sum_{i} r^{a}_{ii}(z, z'),\\
    &r^{a}(z) \triangleq \left(r^a(z_l, z)\right)_l \in \R^{N},\\
    &R^{a}_{fp, pc} \triangleq 1/2 \sum_{i} \tilde \Phi_{ii} \Phi \in \R^{N \times n_{pc}},\\
    &\hat d \triangleq \left(\frac{\partial \hat p}{\partial t}(z_l)\right)_l \in \R^{N}.
\end{align}

\subsubsection{Gram Matrices Computations}\label{subsubsec:fp_grams}

\vspace{1em}
The term \((\mathcal{L}_{t}^{b, a})^* \hat{p}\) involves second derivatives of \(a \times \hat{p}\), leading to the feature maps \(\tilde{\phi}_{ij}\), which are sums of four terms each. Consequently, the evaluation of the scalar product between two feature map values results in a sum of 16 terms.

\paragraph{$\tilde K$ computation.} Denoting \(\odot\) as the Hadamard product, and \(A \odot BC \triangleq A \odot (BC)\), we have
\begin{align}
    \tilde K \triangleq (\langle \tilde \phi(z_p) ,\, \tilde \phi(z_q)\rangle_{\bmG^{n+1}})_{p,q=1}^{N} = \sum_{i=1}^n \tilde K_i + \frac{1}{4}\sum_{i,j=1}^n \tilde K_{ii}^{jj},
\end{align}
with
\begin{align*}
    \tilde K_i \triangleq (\langle \tilde \phi_i(z_p) ,\, \tilde \phi_i(z_q)\rangle_{\bmG^{n+1}})_{p,q=1}^{N} = \quad K \odot \hat P_i \hat P_i^T
    + K_i \odot \hat P \hat P_i^T
    + K^i \odot \hat P_i \hat P^T
    + K_i^i \odot \hat P \hat P^T,
\end{align*}
and
\begin{equation}
\begin{split}
    \tilde K_{kl}^{ij} \triangleq \left( \langle \tilde \phi_{kl}(z_p),\, \tilde \phi_{il}(z_q)\rangle_{\bmG^{n+1}}\right)_{p,q=1}^N
    =\quad &K \odot  \hat P_{kl} \hat P_{ij}^T \,+\, K_k \odot \hat P_{l} \hat P_{ij}^T \,+\, K_l \odot \hat P_{k} \hat P_{ij}^T \,+\, K_{kl} \odot \hat P \hat P_{ij}^T\\
    + &  K^i \odot \hat P_{kl} \hat P_j^T \,+\, K^i_k \odot \hat P_{l} \hat P_j^T \,+\, K_l^i \odot \hat P_{k} \hat P_{j}^T \,+\, K_{kl}^i \odot \hat P \hat P_{j}^T\\
    + & K^j \odot \hat P_{kl} \hat P_i^T \,+\, K^j_k \odot \hat P_{l} \hat P_i^T \,+\, K_l^j \odot \hat P_{k} \hat P_{i}^T \,+\, K_{kl}^j \odot \hat P \hat P_{i}^T\\
    + & K^{ij} \odot \hat P_{kl} \hat P^T \,+\, K^{ij}_k \odot \hat P_{l} \hat P^T \,+\, K_l^{ij} \odot \hat P_{k} \hat P^T \,+\, K_{kl}^{ij} \odot \hat P \hat P^T,
\end{split}
\end{equation}
where $K_{ij}^{kl} = (k_{ij}^{kl}(z_p, z_q))_{p,q = 1}^{N}$, and $\hat P_{ij} = (\hat p_{ij}(z_p))_{p=1}^{N}$.

\paragraph{Fast computation of the Gram matrices $K_{ij}^{kl}$ with Gaussian kernel.} If
\begin{align}
    k((t,x), (t', x')) \triangleq \exp(-\gamma \|(t,x)-(t',x')\|^2) \quad \text{for} \quad \gamma>0,
\end{align}
then, denoting $D_z \triangleq (Z_k - Z_l)_{k, l=1}^{N}$, the $K^{kl}_{pq}$ can be computed recursively as follows. For $i, j, k, l \in \llbracket 1, n\rrbracket$, 
\begin{align*}
    &K_i = - 2\gamma D_z e_i \odot K,\\
    &K^i = (-K_i),\\
    &K_{ij} = -2\gamma (D_z e_j \odot K_i + \delta_{ij} K),\\
    &K^{ij} = K_{ij},\\
    &K_{i}^j = - K_{ij},\\
    &K_k^{ij} = - 2\gamma (D_z e_j \odot K_{ik} + \delta_{jk} K_i + \delta_{ij} K_k),\\
    &K_{ij}^{k} = - K_{i}^{jk},\\
    &K_{kl}^{ij} = - 2\gamma (- D_z e_j \odot K_{kl}^{i} + \delta_{jl} K_{ik} + \delta_{jk} K_{il} +  \delta_{ij} K_{kl}).
\end{align*}

\subsubsection{Algorithm}

\paragraph{Inputs.} We are provided with a dataset $(z_l)_{l=1}^N = (t_l, x_l)_{l=1}^N$ sampled
i.i.d.\ uniformly from the bounded domain $D$, along with a p.d. kernel $k$ over $[0, T] \times \R^n$, and the hyper-parameter \(\lambda>0\).

\paragraph{Step 2.1 ($(\hat b, \hat a)$ fitting).} From Sections \ref{subsubsec:fp_closed} and \ref{subsubsec:fp_grams}, at fitting time, one needs to compute and store
\begin{align}
    (\tilde K + N\lambda I)^{-1} \in \R^{{N} \times {N}}.
\end{align}

The time and space complexities of this step are $\bmO({N}^3)$ and $\bmO({N}^2)$, respectively.

\paragraph{Step 2.2 ($(\hat b, \hat a)$ predictions).} Denoting $D_z^{te} = (Z_l - Z^{te}_l)_{k,l} \in \R^{{N} \times N_{te}}$, we have
\begin{align*}
    &K_i^{fp, te} = - 2\gamma D_z^{te} e_i \odot K^{fp, te},\\
    &K_{ij}^{fp, te} = - 2\gamma (D_z^{te} e_j \odot K_i^{fp, te} + \delta_{ij} K^{fp, te}).
\end{align*}
Then, predictions can be computed from the formulas provided in Section \ref{subsec:fp_proof}.

The time and space complexities of this step are $\bmO(N N_{te})$ and $\bmO(N N_{te})$, respectively.

\paragraph{Outputs.} We return the predicted SDE coefficients' values for the given test inputs.

\subsubsection{FP Matching with Nyström Approximation}

Time and space complexities of Fokker-Planck matching can be reduced from $\bmO(N^{3})$ and $\bmO(N^{2})$ to $\bmO(m N^{2})$ and $\bmO(Nm)$ by using Nyström approximation with $m$ anchors. More precisely, the Fokker-Planck matching objective with the Nyström approximation 
\begin{align}
    w = \sum_{i=1}^m \alpha_i \tilde \phi(z_i),
\end{align}
writes
\begin{align}
    \min_{\alpha \in \R^m} \alpha^T  \left(\tilde K_{nm}^T \tilde K_{nm} + N\lambda \tilde K_{mm}\right) \alpha - 2 \alpha^T \left(\tilde K_{mn} Y_{tr} + R^{a}_{mn_{pc}}\gamma\right).
\end{align}
It can be solved with
\begin{align}
    \alpha = \left(\tilde K_{nm}^T \tilde K_{nm} + N\lambda \tilde K_{mm}\right)^{-1}(\tilde K_{mn} Y_{tr} + R^{a}_{mn_{pc}}\gamma).
\end{align}
We deduce the following formula for $\hat b_{ny, pc}$, $\hat a_{ny, pc}$.
\begin{align}
    &\hat b_{ny, pc}(z) = \alpha^T R^{b}_m(z),\\
    &\hat a_{ny, pc}(z)= \alpha^T R^{a}_m(z),
\end{align}
and the following formula for the dual problem
\begin{align}
    \min_{\gamma \geq 0} \gamma^T A \gamma + \gamma^T b,
\end{align}
with 
\begin{align}
    &A = (R_{mn_{pc}}^{a})^T \left(\tilde K_{nm}^T \tilde K_{nm} + N\lambda \tilde K_{mm}\right)^{-1} R_{mn_{pc}}^{a},\\
    &b = 2 (R^{a}_{m n_{pc}})^T\left(\tilde K_{nm}^T \tilde K_{nm} + N\lambda \tilde K_{mm}\right)^{-1}\tilde K_{mn}\hat d.
\end{align}

\section{Implementation Details of Controlled SDE Estimation}\label{sec:implementation_controlled}

This section details the implementation of the controlled SDE estimation method proposed in this work, available on GitHub (\url{lmotte/controlled-sde-learn}) as an open-source Python library. We update the formulas and computational complexity of each step provided in Section \ref{sec:implementation_uncontrolled} for uncontrolled SDEs to adapt them for controlled SDEs.

\subsection{Step 1: Probability Density Estimation}

For controlled SDEs, the estimation of the probability density function involves repeating the algorithm from Step 1 of the uncontrolled case \(K\) times. Each iteration corresponds to a specific control setting. At the end of this step, we have computed and stored
\begin{align}
    (\hat p_{k}((t_i, x_j)))_{i,j,k} \in \R^{M_{fp} \times N_{fp} \times K},
\end{align}
and also the partial derivatives evaluations.

The time and space complexities of this step are $\bmO(K M^3)$ and $\bmO(K(M^2 + QMn))$ for fitting, and $\bmO(KQMN\max(M_{fp}, M))$ and $\bmO(KQMN)$ for predictions.

\subsection{Step 2: Fokker-Planck Matching}

We use the same algorithm, adding the control dimensions to the inputs. More precisely, closed-form formulas are updated with $z = (t, x, v) \in [0, T] \times \R^n \times \R^d$ instead of $z=(t,x) \in [0, T] \times \R^n$. We store
\begin{align}
    Z_{tr} = (t_i, x_i, u_{\theta_k}(t_i))_{i=1,\dots, N_{fp}, k =1, \dots, K},
\end{align}
and compute and store
\begin{align}
    (\tilde K + NK \lambda I)^{-1} \in \R^{NK \times NK}.
\end{align}

The time and space complexities of this step are $\bmO(K^3 N^3)$ and $\bmO(K^2 N^2)$ for fitting, and $\bmO(KN N_{te})$ and $\bmO(KN N_{te})$ for predictions.

% Note: in this sample, the section number is hard-coded in. Following
% proper LaTeX conventions, its should properly be coded as a reference:

%In this appendix we prove the following theorem from
%Section~\ref{sec:textree-generalization}:

\vskip 0.2in
\bibliography{references}

\end{document}